\documentclass[11pt]{article}
\usepackage{graphicx} % For images
\usepackage{float}    % For tables and other floats
\usepackage{verbatim} % For comments and other
\usepackage{amsmath}  % For math
\usepackage{amssymb}  % For more math
\usepackage{subfig}   % For subfigures
\usepackage[colorlinks=true,citecolor=blue,linkcolor=blue,urlcolor=blue]{hyperref}
\usepackage{fullpage}
\usepackage{amsthm,mathtools}
\usepackage{enumerate}
\usepackage{paralist}
\usepackage{xspace}
\usepackage{caption}
\usepackage{bbm}
\usepackage{url}
\usepackage{mathrsfs}
\usepackage{algorithm}
\usepackage{algorithmic}
\usepackage{color}
\usepackage{epstopdf}

\newcommand{\vertiii}[1]{{\left\vert\kern-0.25ex\left\vert\kern-0.25ex\left\vert #1 
		\right\vert\kern-0.25ex\right\vert\kern-0.25ex\right\vert}}

\makeatletter

\newtheorem{theorem}{Theorem}
\newtheorem{definition}{Definition}
\newtheorem{lemma}{Lemma}

\newtheorem{proposition}{Proposition}

\newtheorem{assumption}{Assumption}
\newtheorem{fact}{Fact}
\usepackage[margin=1in]{geometry}
\usepackage{times}

\usepackage{thmtools}
\usepackage{thm-restate}

\usepackage{listings}

\@ifundefined{showcaptionsetup}{}{%
	\PassOptionsToPackage{caption=false}{subfig}}
\usepackage{subfig}
\makeatother

\usepackage{listings}

\global\long\def\H2{\mathcal{H}_2}

\global\long\def\E1{\mathcal{E}_1}

\global\long\def\linf{\infty/\infty}
\global\long\def\lone{1/1}
\global\long\def\loff{\mathrm{off}}

\global\long\def\DI{\Delta^{\cap}}

\newcommand{\R}{\ensuremath{\mathbb{R}}}

\newcommand*{\defeq}{\stackrel{\text{def}}{=}}

\begin{document}

	\title{\bf Scalable Inference of Sparsely-changing Markov Random Fields with Strong Statistical Guarantees}
	\author{Salar Fattahi$^1$ and Andr\'es G\'omez$^2$\vspace{5mm}\\
	$^1$Industrial and Operations Engineering, University of Michigan\\
	$^2$Industrial \& Systems Engineering, University of Southern California
%		\thanks{Salar Fattahi is with the Department of Industrial Engineering and Operations Research, University of California, Berkeley. Nikolai Matni is with the Department of Electrical and Systems Engineering, University of Pennsylvania. Somayeh Sojoudi is with the Departments of Electrical Engineering and Computer Sciences and Mechanical Engineering as well as the Tsinghua-Berkeley Shenzhen Institute, University of California, Berkeley. This work was supported by the ONR Award N00014-18-1-2526, NSF Award 1808859 and AFSOR Award FA9550-19-1-0055.}
	}

	\maketitle
	
	\begin{abstract}
In this paper, we study the problem of inferring time-varying Markov random fields (MRF), where the underlying graphical model is both sparse and changes {sparsely} over time. Most of the existing methods for the inference of time-varying MRFs rely on the \textit{regularized maximum likelihood estimation} (MLE), that typically suffer from weak statistical guarantees and high computational time. Instead, we introduce a new class of constrained optimization problems for the inference of sparsely-changing MRFs. The proposed optimization problem is formulated based on the exact $\ell_0$ regularization, and can be solved in near-linear time and memory. Moreover, we show that the proposed estimator enjoys a provably small estimation error. As a special case, we derive sharp statistical guarantees for the inference of sparsely-changing Gaussian MRFs (GMRF) in the high-dimensional regime, showing that such problems can be learned with as few as one sample per time. Our proposed method is extremely efficient in practice: it can accurately estimate sparsely-changing graphical models with more than 500 million variables in less than one hour.
\end{abstract}

\section{Introduction}\label{sec:intro}
Contemporary systems are comprised of massive numbers of interconnected components that interact according to a hierarchy of complex, unknown, and time-varying topologies. For {example}, with billions of neurons and hundreds of thousands of voxels, the human brain is considered as one of the most complex physiological networks~\cite{rubinov2010complex, huang2010learning, liu2013time, narayan2015two, kim2015highly}. 
% As another example, the emergence of self-driving cars has only accentuated the need for the development of real-time and reliable methods for detecting moving objects, whose temporal locations are captured through a dynamically-evolving 3D network~\cite{isard2003pampas, murphy2004using, sigal2004tracking}. 
% As another example, stock correlation networks have recently gained a prominent role in deciphering the dynamic correlation among different stock prices~\cite{kazemilari2015correlation, kenett2010dominating, sun2020spillovers}.
% The vast amounts of parameters to be estimated, caused both by the large number of components and the time-varying nature of the associated data, are currently the major bottlenecks in our ability to successfully solve such inference problems.

The temporal behavior of today's interconnected systems, such as those mentioned above, can be captured via \textit{time-varying Markov random fields (MRF)}. Time-varying MRFs are associated with a temporal sequence of undirected \textit{Markov graphs} {$\mathcal{G}_t(V,E_t)$, where $V$ and $E_t$ are the set of nodes and edges in the graph at time $t$}. The node set $V$ represents the random variables in the model, while the edge set $E_t$ captures the conditional dependency between these variables at time $t$.
%; a feature known as \textit{Markov property}.
A popular approach for the inference of graphical models is based on the so-called \textit{maximum-likelihood estimation} (MLE): to obtain a model based on which the observed data is most probable to occur~\cite{wainwright2008graphical}.

Despite being known as theoretically powerful tools \cite{kalman1960new, shellenbarger1966estimation}, MLE-based methods suffer from several fundamental \textit{drawbacks} which render them impractical in realistic settings. First, they often suffer from notoriously high computational cost in massive problems, where the number of variables to be inferred is in the order of millions, or more. Second, they fail to efficiently incorporate \textit{prior} structural information into their estimation procedure. For instance, it is well-known that most large-scale systems exhibit \textit{sparse} interactions amongst their components, which can be captured via \textit{sparsely-changing MRFs}, and estimated using sparsity-promoting \textit{regularizers} (such as $\ell_0$ penalty). However, due to the inherent computational complexity of the $\ell_0$-regularized MLE, most of the existing methods inevitably resort to \textit{relaxed} or weaker variants of such regularization (such as $\ell_1$ penalty), thereby suffering from inferior statistical guarantees. {The aforementioned drawbacks of MLE estimators are further compounded in time-varying settings, since parameters need to be estimated for \emph{each} time period, resulting in a dramatic increase in the size of the problems.}

In this work we propose an alternative to MLE estimators, which explicitly incorporates the $\ell_0$-penalties but is also tractable and scales to massive instances.

\noindent{\bf Notations.} 
% Upper- and lower-case letters are used to denote matrices and vectors, respectively. The $ij^{\text{th}}$ element of a matrix $M$ is denoted as $M_{ij}$. 
The $i^{th}$ element of a time-series vector $v_t$ is denoted as $v_{t;i}$. For a vector $v$, the notation $v_{i:j}$ is used to denote the subvector of $v$ from index $i$ to $j$. 
% Given a symmetric matrix $M\in\mathbb{R}^{d\times d}$, the notation $M\succ 0$ implies that it is positive definite. 
For a vector $v$, the notations $\|v\|_\infty$, $\|v\|_2$, $\|v\|_0$ denote the $\ell_\infty$ norm, $\ell_2$ norm, and the number of nonzero elements, respectively. Moreover, for a matrix $M$, the notations $\|M\|_2$, $\|M\|_\infty$, $\|M\|_{1/1}$, $\|M\|_{\linf}$ refer to the induced 2-norm, induced $\infty$-norm, $\ell_1/\ell_1$ norm, and $\ell_\infty/\ell_\infty$ norm, respectively. Moreover, we define $\|M\|_{\loff} = \|M\|_{\lone}-\sum_{i=1}^d|M_{ii}|$. For a vector $v$ and matrix $M$, the notations $\mathrm{supp}(v)$ and $\mathrm{supp}(M)$ are defined as the sets of their nonzero elements. Given two sequences $f(n)$ and $g(n)$ indexed by $n$, the notation $f(n)\lesssim g(n)$ implies that there exists a constant $C<\infty$ that satisfies $f(n) \leq Cg(n)$. Finally, the notation $f(n)\asymp g(n)$ implies that $f(n)\lesssim g(n)$ and $g(n)\lesssim f(n)$. 
% Given two scalars $a$ and $b$, the notation $a\vee b$ denotes their maximum.

Due to space restrictions, all proofs are deferred to the supplementary file.

\subsection{Warm-up: Regularized MLE for Sparsely-changing Gaussian MRFs}\label{subsec:TVGL}
To illustrate the fundamental drawbacks of the regularized MLE, first we consider the class of time-varying Gaussian MRFs (GMRF) with sparsely-changing structure. 
% Such models are commonly used in financial markets~\cite{hallac2017network} and video processing~\cite{ullah2018directed}. 
The parameters of time-varying GMRFs can be inferred by estimating a sequence of inverse covariance matrices (also known as precision matrices) $\{\Theta_t\}_{t=0}^T$, where both $\Theta_t$ and $\Theta_t-\Theta_{t-1}$ are sparse. Given these precision matrices, the edge set of the Markov matrix $\mathcal{G}_t$ coincides with the off-diagonal nonzero elements of $\Theta_t$~\cite{weiss2000correctness}. The sparse precision matrices can be estimated via the following regularized MLE, also known as \textit{time-varying Graphical Lasso (GL)}~\cite{hallac2017network, cai2018capturing}:
\begin{subequations}\label{mle_reg_gmrf}
	\begin{align}
	\{\widehat{\Theta}_t\}_{t=0}^T =  \arg\min_{\Theta_t}& \ \ \sum_{t=0}^{T}\left(\langle\Theta_t,\widehat{\Sigma}_t\rangle-\log\det(\Theta_t)\right)\notag\\
	&\hspace{-2cm}+\gamma_1 \sum_{t=0}^{T}\|\Theta_t\|_{\loff}+\gamma_2\sum_{t=1}^{T}\|\Theta_t- \Theta_{t-1}\|_{\lone}\\
	\text{s.t.}&\ \ \Theta_t\succ 0\qquad\quad t = 0,1,\dots, T
	\end{align}
\end{subequations}
where $\widehat\Sigma_t\in\mathbb{R}^{d\times d}$ is the sample covariance matrix. Without loss of generality and to streamline the presentation, we assume that the samples have zero mean. 
%Solving the above optimization is a daunting task due to its highly nonlinear and nonconvex nature. To circumvent this challenge, a common approach is to replace the nonconvex $\ell_0$ penalty with its convex $\ell_1$ relaxation. The resulting relaxed problem is often referred to as \textit{time-varying Graphical Lasso}~\cite{hallac2017network, cai2018capturing}. 
The next example shows that the time-varying GL may lead to poor estimates.

\begin{figure}[ht]\centering
	\includegraphics[width=0.7\paperwidth]{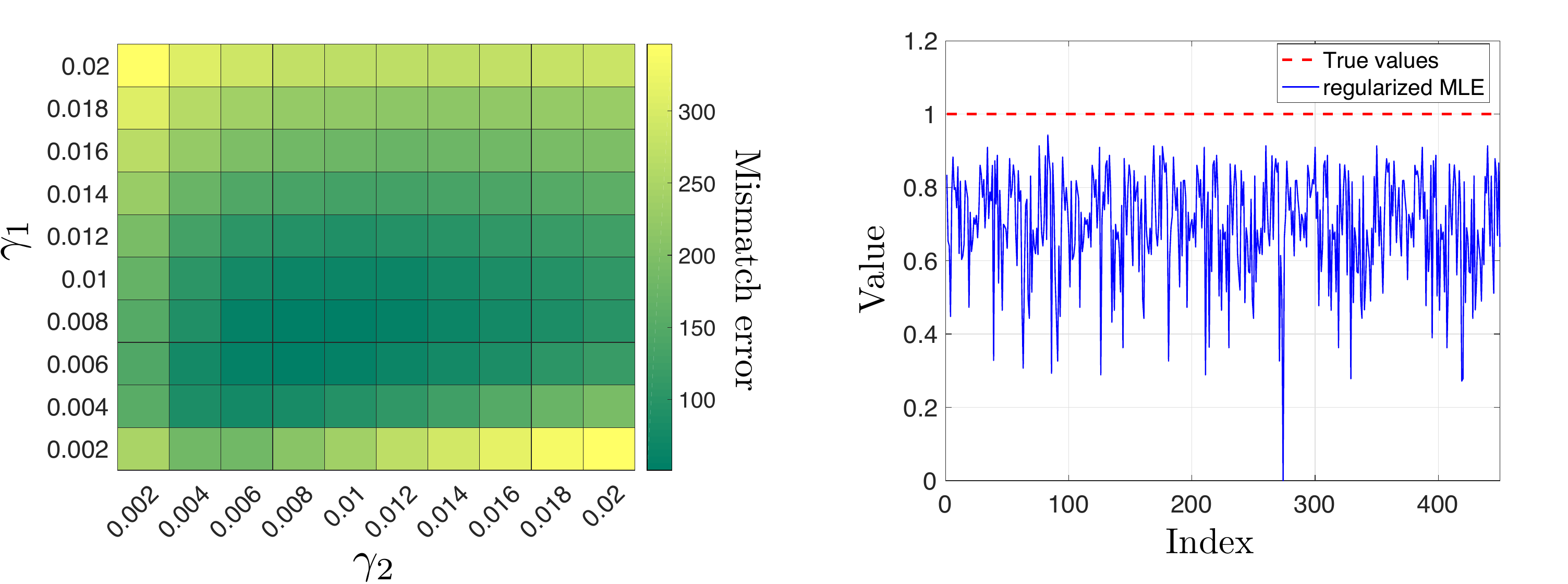}\vspace{-2mm}
	\caption{\footnotesize(left) The heatmap of the mismatch error. (right) The true and estimated nonzero elements of the precision matrix.}\label{fig_mle_reg}
\end{figure}

% \begin{figure*}
% 	\centering
% 	\includegraphics[width=0.55\paperwidth]{fig_mle.pdf}
% 	\vspace{-2mm}
% 	\caption{ (left) The heatmap of the mismatch error. (right) The true and estimated nonzero elements of the precision matrix.}\label{fig_mle_reg}
% \end{figure*}

%\begin{example}
\noindent{\bf Example 1. }{\it Consider a scenario where $\{\Theta_t\}_{t=0}^4 \in\mathbb{R}^{25\times 25}$ are randomly generated symmetric and sparse matrices. At each time $t = 0,\dots,4$, the precision matrix $\Theta_t$ has exactly 100 off-diagonal elements with value one in its upper-triangular part, and the remaining off-diagonal entries are set to zero. Moreover, the diagonal entries $\Theta_{t;ii}$ are chosen as $1+\sum_{j\not=i}\Theta_{t;ij}$. At every time, 10 nonzero off-diagonal elements are changed to zero, and 10 zero elements are set to one. To generate the sample covariance matrix $\widehat{\Sigma}_t$, each element of the true inverse covariance matrix $\Sigma_t^{-1}$ is perturbed with a noise value, uniformly chosen from $[-0.1,0.1]$. Figure~\ref{fig_mle_reg} (left) illustrates a heatmap of the \textit{mismatch error}, i.e., the total number of mismatches in the sparsity patterns of the true and estimated precision matrices and their differences, for different values of the regularization coefficients. It can be seen that after an exhaustive search over the regularization coefficient space, the best achievable mismatch error is in the order of 70.
	% , which is comparable to the total number of nonzero elements in the true inverse covariance matrices. 
	Thus, the estimated parameters reveal little information about the true structure of the time-varying GMRF. %On the hand,
	{Moreover, }Figure~\ref{fig_mle_reg} (right) depicts the concatenation of the nonzero elements in the true precision matrices (dashed red line), and their corresponding values in the estimated matrices (blue curve) at time $t=1$. It can be seen that, even when the sparsity pattern of the elements is correctly recovered, the estimated nonzero entries are ``shrunk'' toward zero, incurring a {substantial} bias.}

%\end{example}
The above example shows the inferior statistical performance of the time-varying GL as an instance of a regularized MLE method for sparsely-changing GMRFs. In addition to its subpar statistical performance, time-varying GL suffers from expensive computational complexity:
% roughly speaking, an off-the-shelf solver for convex optimization, such as \texttt{MOSEK}~\cite{mosek2010mosek}, would take hours on a normal computer to solve instances of the time-varying GL with sizes beyond $d = 200$ and $T = 5$. 
a typical numerical solver for the time-varying GL has a \emph{per-iteration} complexity in the order of $\mathcal{O}(Td^3)$, even if it is tailored to a specific class of problems~\cite{hallac2017network, ravikumar2010high, lelearning}. Solvers with such computational complexity may fall short of practical use in the large-scale settings.
%Optimistically, suppose that an instance of the regularized MLE with $d = 1000$ and $T = 1$ can be solved in 1 minute. Then a 10-fold increase in both $d$ and $T$ suggests that it would take almost 1 week to solve the new problem using the same solver. 
%This calls for a new inference framework that can efficiently cope with the increasing scale of time-varying graphical models, while being statistically correct; a topic that is at the crux of this paper.

\section{Proposed Approach}
% The aim of this paper is to address the aforementioned challenges.
%At a high level, our main idea is to solve a simpler class of surrogate optimization problems \textit{in lieu of} the standard regularized MLE.
The proposed framework is based on exact solutions to a class of \emph{tractable} discrete $\ell_0$-problems, thus circumventing bias and other drawbacks of the standard $\ell_1$-approximations, while guaranteeing the scalability of the proposed method.
As a general framework, we study the optimization problem:
\begin{subequations}\label{generalopt}
	\begin{align}
	{\min}\ \ & (1-\gamma)\sum_{t=0}^T\|\theta_t\|_0+\gamma\sum_{t=1}^T\|\theta_t-\theta_{t-1}\|_0\\
	\text{s.t.}\ \ & \|\theta_t-{\widetilde F^*}(\widehat{\mu}_t)\|_{\linf}\leq \lambda_t\qquad t = 0,1,\dots, T
	\end{align}
\end{subequations}
where the optimal solutions $\{\widehat{\theta}_t\}_{t=0}^T$ with $\widehat\theta_t\in\mathbb{R}^{p}$ are the {estimates of the} unknown \textit{canonical parameters} of the sparsely-changing MRF, $\left\{\widehat{\mu}_t\right\}_{t=0}^T$ are the so-called \textit{empirical moment parameters}, and ${\widetilde F^*}(\cdot)$ is an \textit{approximate backward mapping} of the model; see \S\ref{sec:preliminaries} for formal definitions. In the context of time-varying GMRFs, the unknown canonical parameters correspond to the vectorization of the upper-triangular part of the precision matrices with $p = d(d+1)/2$. Moreover, the mean parameters correspond to the sample covariance matrices obtained directly from the data, and the approximate backward mapping is a \textit{proxy} of the precision matrix. In particular, we use the \textit{inverse of a soft-thresholded weighted sample covariance matrix} as a proxy; see \S\ref{sec:GMRF}. %In the next section, the formal definitions of these parameters will be discussed in details. 
% \noteag{Any reason why we define the vectorized precision matrices, instead of the matrices itself? Should we say $p=d\times d$, or $p=d(d+1)/2$?} \noteag{To make this section as self contained as possible, provide intuition, and facilitate the understanding from a reader not very familiar to the topic, I think it would make to mention, as an example, what the backward mapping is in the context of GMRF.}
%Later, we will explain how to efficiently obtain such approximate backward mappings for different classes of time-varying MRFs. 

Our first result establishes a deterministic guarantee on the estimation error of the optimal solution to~\eqref{generalopt}.
% In particular, it provides a deterministic guarantee on the estimation error and the sparsity pattern (also known as \textit{sparsistency}) of the optimal solution to~\eqref{generalopt}.
\begin{theorem}[Estimation error and sparsistency]\label{thm_deterministic}
	Suppose that $0<\gamma<1$. For every $t=0,\dots,T$, define $\mathcal{S}_t$ as the set of indices corresponding to the nonzero elements of the true canonical parameter $\theta_t^*$. Similarly, for every $t=1,\dots,T$, define $\mathcal{D}_t$ as the set of indices corresponding to the nonzero elements of $\theta_t^*-\theta_{t-1}^*$. Assume that\newline
	$\bullet$ $\left\|\theta_t^*-{\widetilde F^*}(\widehat{\mu}_t)\right\|_{\infty}<\lambda_t$,  $\forall 0\leq t\leq T$,\newline
	$\bullet$ $2\lambda_t\leq\min_{i\in\mathcal{S}_t}|\theta_{t;i}^*|$, $\forall 0\leq t\leq T$,\newline
	$\bullet$ $2\lambda_t+2\lambda_{t-1}\leq\min_{i\in\mathcal{D}_t}|\theta_{t;i}^*-\theta_{t-1;i}^*|$, $\forall 0\leq t\leq T$.
	
	Then, the following statements hold for every $0\leq t\leq T$:
	\begin{itemize}
		\item[-] \textit{(Sparsistency)} We have $\mathrm{supp}\left(\widehat{\theta}_t\right) = \mathrm{supp}\left({\theta}^*_t\right)$ and $\mathrm{supp}\left(\widehat{\theta}_t-\widehat{\theta}_{t-1}\right) = \mathrm{supp}\left({\theta}^*_t-{\theta}^*_{t-1}\right)$.
		% 		\begin{align*}
		% 		\mathrm{supp}\left(\widehat{\theta}_t\right) &= \mathrm{supp}\left({\theta}^*_t\right) &\!\forall 0\leq t\leq T\\
		% 		\mathrm{supp}\left(\widehat{\theta}_t-\widehat{\theta}_{t-1}\right) &= \mathrm{supp}\left({\theta}^*_t-{\theta}^*_{t-1}\right) &\!\forall 0\leq t\leq T
		% 		\end{align*}
		\item[-] (Estimation error) We have
		\begin{align*}
		\|\widehat{\theta}_t-\theta_t^*\|_{\infty}\leq 2\lambda_t,\qquad \|\widehat{\theta}_t-\theta_t^*\|_2\leq 2\sqrt{|\mathcal{S}_t|}\lambda_t.
		\end{align*}
	\end{itemize}
\end{theorem}

% \noteag{Shouldn't we say somewhere that the proof is in the online companion?}

Theorem~\ref{thm_deterministic} presents a set of conditions under which the proposed estimation method achieves sparsistency and small element-wise estimation error. The first condition entails that the true canonical parameter $\theta_t^*$ is a feasible solution to~\eqref{generalopt}. %, by controlling the approximation error of the backward mapping.
The second and third conditions imply that there is a non-negligible gap between the zero and nonzero elements of the true parameters and their temporal changes. Such assumptions are crucial for excluding false negatives from the support of the estimated parameters. 

%It is worthwhile to note that 
Theorem~\ref{thm_deterministic} holds for the general class of sparsely-changing MRFs, provided that an accurate approximation of the backward mapping is available. Such backward mappings have been widely studied for different classes of MRFs, such as GMRFs~\cite{yang2014elementary, fattahi2019graphical} and Discrete MRFs~\cite{wainwright2003tree}. We focus on the class of sparsely-changing GMRFs, and show how Theorem~\ref{thm_deterministic} can be used to provide end-to-end sample complexity bounds on the inference of sparsely-changing GMRFs in the high-dimensional settings, see \S\ref{sec:GMRF}.

%Our next theorem characterizes the computational complexity of solving~\eqref{generalopt} to global optimality.

\begin{theorem}[Computational complexity]\label{thm_runtime}
	The optimization problem~\eqref{generalopt} can be solved to optimality in {at most} $\mathcal{O}(pT^2)$ time and memory {on a single thread}.
\end{theorem}
Theorem~\ref{thm_runtime} shows that the optimization problem~\eqref{generalopt} can be solved efficiently in practice despite its non-convex nature. Our solution method for~\eqref{generalopt} relies on the \textit{element-wise decomposability} of~\eqref{generalopt}. In particular, we decompose~\eqref{generalopt} into smaller subproblems over different coordinates of $\{{\theta}_t\}_{t=0}^T$. Then, we show that the optimal solution to each subproblem can be obtained by solving a shortest path problem on an auxiliary \textit{weighted directed acyclic graph (DAG)}. The details of our solution method, as well as some improvements on its runtime are presented in \S\ref{sec:algorithm}. {
	% We point out that in practice, the proposed algorithm may be much faster than the worst case complexity stated in Theorem~\ref{thm_runtime}, even when executed on in single thread. 
	Moreover, our algorithm is easily parallelizable, leading to better runtimes in practice.} \vspace{2mm}
%As will be delineated in Section~\ref{label}, the complexity of solving~\eqref{generalopt} can be further reduced to $\mathcal{O}(\sum_{i=1}^{p}TZ_i)$, where $0\leq Z_i\leq T/2$ is the number of zero-feasible sequences in the subproblem $i$.

\begin{figure}
	\centering
	\includegraphics[width=0.7\paperwidth]{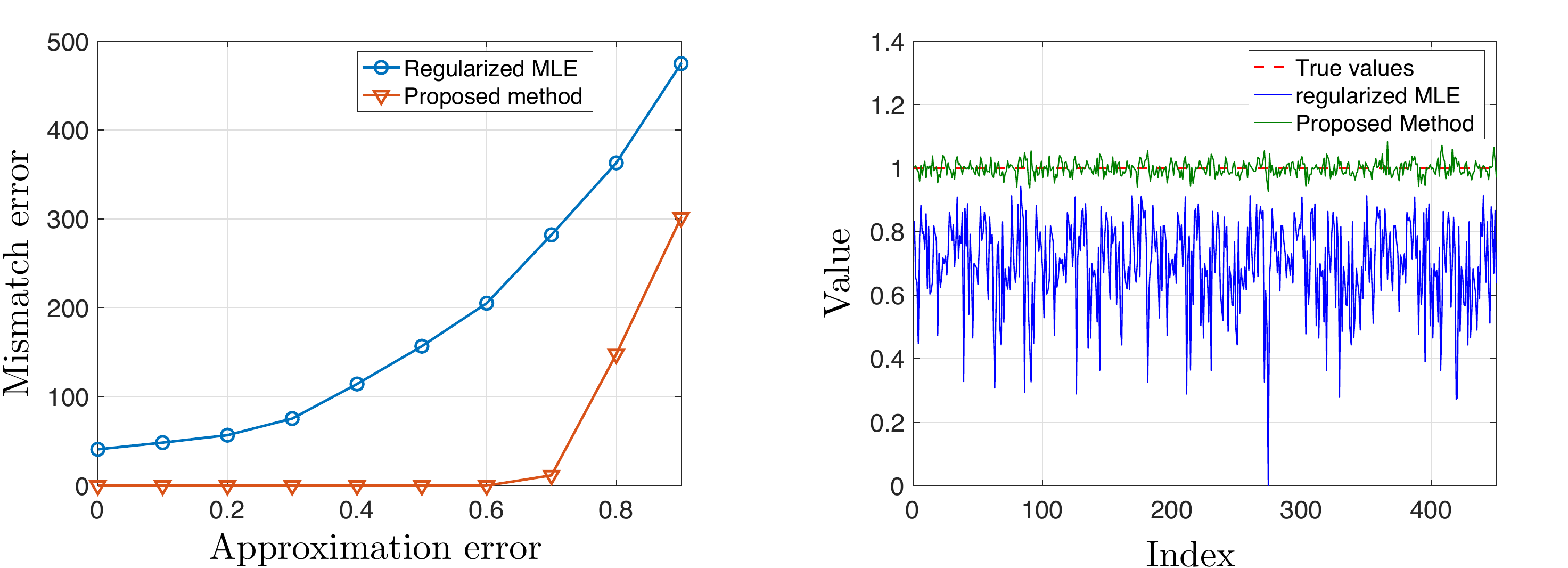}
	\vspace{-2mm}
	\caption{ (left) The heatmap of the mismatch error. (right) The true and estimated nonzero elements of the precision matrix.}\label{fig_mle_l0}
\end{figure}

\noindent{\bf Example 1 (continued). } {\it Figure~\ref{fig_mle_l0} depicts the performance of the proposed method, compared to that of the regularized MLE with $\gamma_1 = 0.06$ and $\gamma_2 = 0.08$ (corresponding to the smallest mismatch error) for the instances generated in Example 1. The regularization parameter $\gamma$ is set to $1/3$.
	% \noteag{While technically there is nothing incorrect here, I think it might be easier for a reader if we use different notation for the regularization params of glasso vs our method, e.g., replace $\gamma_1$ and $\gamma_2$ with $\gamma_1$ and $\gamma_2$. Also, I dislike using $\alpha$ since in stats $\alpha$ is often used for a decision variable, but that could be me being particular}. 
	Each element of the true inverse covariance matrix $\Sigma_t^{-1}$ is perturbed with a noise value, uniformly chosen from $[-\alpha,\alpha]$ for a varying scalar $\alpha$, and $\lambda_t$ is set to $\alpha$. Figure~\ref{fig_mle_l0} (left) demonstrates that the proposed method enjoys a significantly smaller mismatch error, for different levels of approximation in the backward mapping. In particular, the proposed estimator achieves a zero mismatch error, even for a fairly large approximation error ($\|\theta_t^*-\widetilde F^*(\widehat{\mu}_t)\|_\infty\leq 0.6$). On the other hand, Figure~\ref{fig_mle_l0} (right) shows that the synthetic bias caused by the regularized MLE is alleviated via the proposed method. 
	% 	Our algorithm can solve instances with 500 million unknown parameters ($p \approx 5\times 10^7$, and $T = 10$) in less than one hour, which are far beyond the capability of the available solvers for the time-varying Graphical Lasso. 
}
%\noteag{I wonder if we can put here a mismatch error as a function of the regularization params, as we did for glasso? Perhaps we can show that our method is robust in terms of the regularization params.}$\hfill\square$ }

\section{Related Works}
\label{sec:literature}

% A considerable amount of work has been carried out to study the inference of MRF with different side or prior information. 
We now summarize the works most relevant to our results.

\noindent {\bf Inference of time-varying MRF.} 
% The problem of inferring MRFs can be traced back to Kalman filters, where the goal is to predict a random signal evolving over time by ``filtering out'' the observational noise~\cite{kalman1960new, brown1992introduction}. More recent results have studied the non-asymptotic inference of time-varying MRF with side information, such as sparsity and smoothness. 
In addition to the time-varying graphical Lasso introduced in \S\ref{subsec:TVGL}, a recent line of works have studied the inference of smoothly changing GMRFs~\cite{zhou2010time, greenewald2017time}, where a kernel averaging technique combined with Graphical Lasso is used to estimate the smoothly-changing precision matrices. However, these methods do not leverage the prior information about the sparsity of the parameter differences. With the goal of addressing this deficiency, several works have studied the inference of sparsely-changing MRF (also known as sparse \textit{differential networks})~\cite{wang2018fast, zhao2014direct, liu2017learning}. However, the main drawback of these methods is that they only estimate the parameter differences, and their theoretical guarantees are restricted to problems with two time steps ($T=0,1$).
% Moreover, similar to the aforementioned time-varying GL, they heavily rely on relaxed versions of sparsity-promoting regularizations (such as $\ell_1$ penalty), and hence, they too suffer from substantial bias and inferior statistical performance.

\noindent {\bf Sparsity-promoting optimization.} Optimization problems with $\ell_0$ terms are often deemed to be intractable, and approximations are solved instead. Perhaps the most popular approach is the \emph{fused lasso} \cite{rinaldo2009properties,tibshirani2005sparsity,rudin1992nonlinear,vogel1996iterative}, which calls for replacing terms $\|\theta_t\|_0$ and $\|\theta_t-\theta_{t-1}\|_0$ with their $\ell_1$-approximations. Nonetheless, such approximations result in subpar statistical performance when compared with exact $\ell_0$ methods \cite{jewell2017exact,miller2002subset}.

Exact or near-optimal methods for optimization problems of the form 
\begin{align}\label{pairwise}
\min_{ \theta\in [\ell,u]^p}\;& \|\theta\|_0+\sum_{i=1}^p\sum_{j=i+1}^p g_{ij}(\theta_i-\theta_j),
\end{align}
for given one-dimensional functions $g_{ij}:\R\to \R$, have also been studied in the literature.  If functions $g_{ij}$ are convex, then problem \eqref{pairwise} admits pseudo-polynomial time algorithm \cite{ahuja2004cut,bach2019submodular}. Moreover, convex relaxations that deliver near-optimal solutions for \eqref{pairwise} were proposed for the special case of convex quadratic $g$ functions \cite{atamturk2018sparse}; if, additionally, we have $\ell=0$ and $u=\infty$, then problem \eqref{pairwise} is in fact solvable in strongly polynomial time \cite{atamturk2018strong}. On the other hand, problem \eqref{pairwise} is much more challenging for non-convex $g$: if $g(x)=\mathbbm{1}\left\{x\neq 0\right\}$, as is the case in \eqref{generalopt}, then problem \eqref{pairwise} is NP-hard even if the term $\|\theta\|_0$ is dropped from the objective \cite{hochbaum2001efficient}. Nonetheless, as we show in this paper, problem \eqref{pairwise} can be solved quite efficiently in the context of time-varying MRFs, where $g_{ij}(x)=0$ whenever $j>i+1$. 

\section{On Time-varying MRFs and Exponential Families}\label{sec:preliminaries}

% \noteag{I really dislike the term ``preliminaries" this late in the paper. How about killing it? ``On Time-varying MRFs and Exponential Families"?}

% \noteag{Can we send this section to an appendix? It seems to be that we are simply defining terms and notions that may be widely known among statisticians.} \notesf{I don't think this is widely known. We should certainly define what we mean by ``backward mapping'' and ``canonical parameters'' before using them.}

A large class of time-varying MRFs can be expressed as sequences of exponential distributions, defined as:
\begin{align}\label{exp}
\mathbb{P}(X_t;\theta_t) = \exp\left\{\langle \theta_t, \phi(X_t)\rangle-A(\theta_t)\right\}\quad  \forall 0\leq t\leq T,
\end{align}
where $\theta_t\in\Omega\in\mathbb{R}^p$ is the \textit{canonical parameter}
% \footnote{Without loss of generality, we assume that all canonical parameters belong to the same feasible region $\Omega$ that does not change over time.} 
of the exponential distribution at time $t$, the function $\phi : \mathbb{R}^n\to\mathbb{R}^p$ is the \textit{sufficient statistics}, and $A: \mathbb{R}^p\to\mathbb{R}$ is the \textit{log-partition} function, which is used to normalize the distribution. Special classes of time-varying MRFs that can be represented as instances of~\eqref{exp} include time-varying Gaussian MRFs (GMRFs) and Discrete MRFs (DMRFs), corresponding to multivariate Gaussian and discrete random processes, respectively. 
% \noteag{Here we have a small intro to GMRFs and DMRFs... however, we use those in Section 2 without the proper intro.} \notesf{I think it is fine, since we are not getting into the technical details in the intro.}
Due to the equivalence between time-varying MRFs and exponential families, Markov graphs can be systematically obtained from the canonical parameters $\{\theta_t\}_{t=0}^T$~\cite{wainwright2008graphical}. For instance, the canonical parameters in time-varying
GMRFs at time $t$ correspond to the tuple of time-varying precision matrix $\Theta_t$ and mean vector $\eta_t$. 
Moreover, the edge set of the Markov matrix $\mathcal{G}_t$ coincides  with the off-diagonal nonzero elements of $\Theta_t$~\cite{weiss2000correctness}. 
%As another example, the canonical parameters in time-varying DMRFs are the nodewise and edgewise marginal probabilities, whose block-sparsity reveal the structure of the Markov graphs~\cite{wainwright2008graphical}. 

An alternative parameterization of exponential families is via \textit{mean} or \textit{moment parameters}, i.e., the moments of the sufficient statistics $\mu_t(\theta_t) = \mathbb{E}_{\theta_t}[\phi(X_t)]\in\mathcal{M}$. Given the canonical parameters $\{\theta_t\}_{t=0}^T$, the mean parameters $\{\mu_t(\theta_t)\}_{t=0}^T$ can be obtained via the \textit{forward mapping} $F:\Omega\to\mathcal{M}$, where $\mu_t(\theta_t) = F(\theta_t) = \nabla A(\theta_t)$. The conjugate (or Fenchel) duality can be used to define the \emph{backward mapping} $F^*:\mathcal{M}\to\Omega$ with $\theta_t(\mu_t) = F^*(\mu_t) = \nabla A^*(\mu_t)$, where $A^*$ is the conjugate dual of the log-partition function. In practice, the true mean parameters are rarely available, and should be replaced by their \textit{empirical} versions $\{\widehat{\mu}_t\}_{t=0}^T$, where $\widehat{\mu}_t = \frac{1}{N_t}\sum_{i=1}^{N_t}\phi\left(X^{(i)}_t\right)$, and $\left\{X_t^{(i)}\right\}_{i=1}^{N_t}$ is the sequence of available data samples at time~$t$. 

\section{Time-varying GMRFs} 

Theorem~\ref{thm_deterministic} presents a set of deterministic conditions under which the estimates from~\eqref{generalopt} enjoy zero mismatch error and small estimation error. However, the formulation of~\eqref{generalopt} is contingent upon the availability of an accurate backward mapping, and a choice of $\lambda_t$ that satisfies the conditions of Theorem~\ref{thm_deterministic}. In this section, we show how to efficiently design sample-efficient approximate backward mappings, and select $\lambda_t$ accordingly for the class of sparsely-changing GMRFs. Moreover, we use the deterministic conditions of Theorem~\ref{thm_deterministic} to arrive at a non-asymptotic probabilistic guarantee for the inference of time-varying GMFRs under different prior knowledge on their temporal behavior, such as sparsity and smoothness.
% First, we present the proof of Theorem~\ref{thm_deterministic}. 
% Then, we show how Theorem~\ref{thm_deterministic} can be used to provide non-asymptotic statistical guarantees for the inference of sparsely-changing GMRFs in high-dimensional settings.

\subsection{Sparsely-changing GMRFs}\label{subsec:sparse_GMRF}\label{sec:GMRF}

% \noteag{This subsection title is the same as the section.}

Consider a multivariate zero-mean Gaussian process $\{X_t\}_{t=0}^T$ with distribution
{\small
	\begin{align}\label{gmrf}
	\mathbb{P}(X_t) \!=\! \exp\left\{-\frac{1}{2}\langle\Theta_t,X_tX_t^\top\rangle\!+\!\langle\eta_t,X_t\rangle\!-\!A(\mu_t,\Theta_t)\right\} 
	\end{align}
}
for $t=0,1,\dots,T$ where, without loss of generality, we assumed that the mean is zero. Note that~\eqref{gmrf} is a special case of the exponential family~\eqref{exp}. The canonical parameter, i.e., the precision matrix $\Theta_t$, belongs to the domain $\Omega = \{M\in\mathbb{R}^{d\times d}:M\succ 0\}$ with the effective dimension of $p=d(d+1)/2$. Suppose that at any given time $t$, a sequence of data samples $\{X_t^{(i)}\}_{i=1}^{N_t}$ is collected from~\eqref{gmrf}. Therefore, the inference of time-varying GMRFs reduces to estimating the time-varying precision matrix $\Theta_t$ from the data samples $\{X^{(i)}_t\}_{i=1}^{N_t}$. The forward and backward mappings take the following closed-form expressions: $F(\Theta_t) = \Theta_t^{-1}$ and $F^*(\Sigma_t) = \Sigma_t^{-1}$, where $\Sigma_t$ is the true covariance matrix at time $t$. A common choice of approximate backward mapping is $\widetilde F^*(\widehat\Sigma_t) = {\widehat\Sigma_t}^{-1}$, where $\widehat \Sigma_t = \frac{1}{N_t}\sum_{i=1}^{N_t}X^{(i)}_t{X^{(i)}_t}^\top$ is the sample covariance matrix. However, in the high-dimensional settings where $p\gg N_t$, this approximate backward mapping is not well-defined, since the sample covariance matrix is highly rank-deficient and not invertible.  

To address this issue, \cite{yang2014elementary} propose a \textit{proxy} backward mapping for high-dimensional settings: consider the soft-thresholding operator $\texttt{ST}_{\nu}(M):\mathbb{R}^{p\times p}\to \mathbb{R}^{p\times p}$, where $\texttt{ST}_{\nu}(M)]_{ij} = M_{ij}-\mathrm{sign}(M_{ij})\min\{|M_{ij}|,\nu\}$ if $i\not=j$, and $\texttt{ST}_{\nu}(M)]_{ij} = M_{ij}$ if $i=j$.
% \noteag{Don't we need $M_{ij}-\mathrm{sign}(M_{ij})\min\{|M_{ij}|,\nu\}$ below?}
% \begin{align}
% [\texttt{ST}_{\nu}(M)]_{ij} = \begin{cases} M_{ij}-\mathrm{sign}(M_{ij})\min\{|M_{ij}|,\nu\}\ \ \ &\text{if}\ i\not=j\\
% M_{ij} \ \ \ &\text{if}\ i=j.
% \end{cases}
% \end{align}
Based on this definition, we choose the approximate backward mapping as $\widetilde F^*(\widehat \Sigma_t) = [\texttt{ST}_{\nu}(\widehat \Sigma_t)]^{-1}$. \cite{yang2014elementary} show that $\widetilde F^*(\widehat \Sigma_t)$ is well-defined, even in the high dimension, with an appropriate choice of the threshold $\nu$. 

We now make two assumptions. 
% Moreover,~\cite{yang2014elementary} shows that this approximate backward mapping benefits from strong statistical guarantees, provided that the following conditions are satisfied:
\begin{assumption}[Bounded norm]\label{assum1}
	There exist constant numbers $\kappa_1<\infty$, $\kappa_2>0$, and $\kappa_3<\infty$ such that
	\begin{align*}
	\|\Theta_t\|_\infty\leq \kappa_1,\; \inf_{w:\|w\|_\infty=1}\|\Sigma_tw\|_\infty\geq \kappa_2,\; \|\Sigma_t\|_{\linf}\leq \kappa_3,
	\end{align*}
	for every $t = 0,\dots, T$.
\end{assumption}
Assumption~\ref{assum1} implies that the true covariance matrices and their inverses have bounded norms. 
\begin{assumption}[Weak sparsity]\label{assum2} \begin{sloppypar}Given any $t = 0,\dots, T$, the covariance matrix $\Sigma_t$ satisfies $\max_{i}\sum_{j=1}^d|[\Sigma_t]_{ij}|^q\leq s(q,d)$, for some function $s:\mathbb{R}\times \mathbb{Z}_+\to\mathbb{R}$ and $0\leq q<1$. 
		% \noteag{ for some function $s:\R\times \Z_+\to \R$?}.
	\end{sloppypar} 
\end{assumption}
Assuming that $s(0,d)\leq s\ll d$, Assumption~\ref{assum2} reduces to the covariance matrix being sparse. Moreover, in many cases, a sparse inverse covariance matrix leads to weakly sparse covariance matrices. For instance, if $\Theta_t$ has a banded structure with small bandwidth, then it is known that the elements of $\Sigma_t = \Theta_t^{-1}$ enjoy exponential decay away from the main diagonal elements~\cite{demko1984decay, kershaw1970inequalities}. Under such circumstances, one can verify that $s(q,d)\leq\frac{C}{1-\rho^q}$ for some constant $C>0$ and $\rho<1$. More generally, a similar statement holds for a class of inverse covariance matrices whose support graphs have large average path length~\cite{benzi2007decay, benzi2015decay}; a large class of inverse covariance matrices with row- and column-sparse structures satisfy this condition. 

To streamline the presentation, define $\Theta_t^{\min} = \min_{(i,j)\in\mathcal{S}_t}|[\Theta_t]_{ij}|$ and $\Delta\Theta_t^{\min} = \min_{(i,j)\in\mathcal{D}_t}|[\Theta_t-\Theta_{t-1}]_{ij}|$, where $\mathcal{S}_t$ and $\mathcal{D}_t$ are defined as in Theorem~\ref{thm_deterministic}. We assume that $\Theta_t^{\min}$ and $\Delta\Theta_t^{\min}$ are independent of $d$ and $T$.

\begin{theorem}\label{cor_GMRF}
	Consider a sparsely-changing GMRF and let $\zeta = \max\{\log_d(T+1), 1\}$. Given an arbitrary $\tau>\zeta+2$, let $N_t\gtrsim{s(q,d)}^{\frac{2}{1-q}}\tau\log d$.
	% \begin{align}
	%     % N_t\gtrsim\left(\frac{\kappa_1\kappa_3}{\kappa_2}\right)^2{s(q,d)}^{\frac{2}{1-q}}\zeta\log(d\vee N_t)
	%     N_t\gtrsim{s(q,d)}^{\frac{2}{1-q}}\zeta\log d
	% \end{align}
	Then, with the approximate backward mapping $\widetilde F^*(\widehat{\Sigma}_t) = [\texttt{ST}_{\nu_t}(\widehat{\Sigma}_t)]^{-1}$, and parameters $\nu_t \asymp \sqrt{\frac{\tau\log d}{N_t}}$ and $\lambda_t \asymp \sqrt{\frac{\tau\log d}{N_t}}$,
	% \begin{align}
	%     % \nu_t \asymp \kappa_3\sqrt{\frac{\zeta\log\left(d\vee N_t\right)}{N_t}}, \qquad\qquad \lambda_t \asymp \frac{\kappa_1\kappa_3}{\kappa_2}\sqrt{\frac{\zeta\log\left(d\vee N_t\right)}{N_t}},
	%     \nu_t \asymp \sqrt{\frac{\zeta\log d}{N_t}}, \qquad\qquad \lambda_t \asymp \sqrt{\frac{\zeta\log d}{N_t}},
	% \end{align}
	the estimates $\{\widehat{\Theta}_t\}_{t=0}^T$ obtained from~\eqref{generalopt} satisfy the following statements for all $t = 0,1,\dots,T$, with probability of at least $1-4d^{-\tau+\zeta+2}$:
	\begin{itemize}
		\item[-] \textit{(Sparsistency)} We have $\mathrm{supp}\left(\widehat{\Theta}_t\right) = \mathrm{supp}\left({\Theta}^*_t\right)$ and $\mathrm{supp}\left(\widehat{\Theta}_t-\widehat{\Theta}_{t-1}\right) = \mathrm{supp}\left({\Theta}^*_t-{\Theta}^*_{t-1}\right)$.
		% 		\begin{align*}
		% 		\mathrm{supp}\left(\widehat{\Theta}_t\right) &= \mathrm{supp}\left({\Theta}^*_t\right)\\
		% 		\mathrm{supp}\left(\widehat{\Theta}_t-\widehat{\Theta}_{t-1}\right) &= \mathrm{supp}\left({\Theta}^*_t-{\Theta}^*_{t-1}\right).
		% 		\end{align*}
		\item[-] (Estimation error) We have 
		{\small
			\begin{align*}
			\hspace{-7mm}\|\widehat{\Theta}_t\!-\!\Theta_t^*\|_{\linf}&\!\lesssim\! \sqrt{\frac{\tau\log d}{N_t}},\;\|\widehat{\Theta}_t\!-\!\Theta_t^*\|_F\lesssim\! \sqrt{\frac{\tau|\mathcal{S}_t|\log d}{N_t}}.
			\end{align*} }
	\end{itemize}
\end{theorem}
% \begin{proof}
% 	The proof is provided in Appendix~\ref{app_cor_GMRF}.
% \end{proof}

Theorem~\ref{cor_GMRF} is a direct consequence of Theorem~\ref{thm_deterministic} and provides, to the best of our knowledge, the first non-asymptotic guarantee on the inference of sparsely-changing GMRFs with an arbitrary length of time horizon $T$. In particular, it shows that the proposed optimization~\eqref{generalopt} guarantees small estimation error and zero mismatch error for sparsely-changing GMRFs, provided that $N_t$ scales logarithmically with the dimension of the precision matrices. In the static setting ($T=0$), the derived bound recovers the existing results on the sample complexity of learning static GMRFs~\cite{ravikumar2011high, lam2009sparsistency, rothman2008sparse}. 

\subsection{Sparsely-and-smoothly-changing GMRFs} In many applications, such as financial markets and motion detection in video frames, the associated graphical model should be learned ``on-the-go'', as the data arrives with a continuously changing graphical model. Under such circumstances, one may have access to few (or even one) samples at each time. 
% This naturally gives rise to the following question:
% {Can we infer sparsely-changing GMRFs with as few as one sample per time?}
% At the first glance, one may speculate a negative answer to this question. However, we show that one can still accurately estimate the sparsely-changing GMRFs with only one sample per time, provided that the changes in the underlying graphical model are smooth. 
% \noteag{Since we are already short on space, I think this paragraph asking research questions can be safely dropped.}

Suppose that the precision matrices change smoothly over time. Such smooth changes can be modeled via a continuous function $\Theta(x): [0,1]\to \mathbb{R}^{d\times d}$ with uniformly bounded element-wise second derivatives $[\Theta(x)_{ij}]'' = \frac{d^2\Theta(t)_{ij}}{dx^2}$, such that $\Theta_t^* = \Theta(t/T)$ \cite{greenewald2017time, zhou2010time}. If $\Theta(t) \succeq aI$ for every $t\in[0,1]$ and some $a>0$, then the covariance matrix $\Sigma(t) = \Theta(t)^{-1}$ is well-defined and smooth. Then, the problem of inferring the time-varying GMRF reduces to estimating a sequence of precision matrices $\{\Theta(0),\Theta(1/T),\dots,\Theta(1)\}$ given the samples $X_t\sim\mathcal{N}(0,\Sigma(t/T))$.
% , assuming that: (i) the element-wise second derivatives $[\Theta(x)_{ij}]''$ are uniformly bounded over the domain $[0,1]$, and (ii) $\Theta(t/T)-\Theta((t+1)/T)$ is sparse for every $t=1,\dots, T$. 
To alleviate the scarcity of samples, \cite{greenewald2017time, zhou2010time} leverage the smoothness of the precision matrices, by taking the weighted average of the samples over time, where the weights are obtained from a nonparametric kernel. In particular, consider the weighted sample covariance matrix~$\widehat \Sigma^{w}_t$: 
\small
\begin{align}\label{eq_kernal}
\widehat \Sigma^{w}_t = \sum_{s=0}^tw(s,t)\Sigma_s,\text{ where } w(s,t) = \frac{1}{Th}K\left(\frac{s-t}{Th}\right)
\end{align}\normalsize
and $K(\cdot)$ is a symmetric nonnegative kernel that satisfies a set of mild conditions which hold for most standard kernels, including (truncated) Gaussian kernel. These conditions are delineated in the appendix.
% \begin{assumption}[~\cite{greenewald2017time}]
% 	The kernel $K(x)$ satisfies the following conditions:
% 	\begin{itemize}
% 		\item[-] $\int_{-1}^{1} K(x)dx=1$,
% 		\item[-] $\int_{-1}^{1} x^2K(x)dx\leq\infty$,
% 		\item[-] $K(x)$ is uniformly bounded on its support,
% 		\item[-] $\sup_{-1\leq x\leq 1}K''(x/h) = \mathcal{O}(h^{-4})$.
% 	\end{itemize}
% \end{assumption}
% The above assumptions are satisfied for many standard kernels, including (truncated) Gaussian kernel. Roughly speaking, the chosen kernel assigns larger weights to the samples with smaller temporal distance to $t$, thereby leveraging the smoothness of $\Sigma(t)$. Moreover, the parameter $h$ is the \textit{bandwidth} of the kernel, controlling the decay rate of the weights. 
The following assumptions are the counterparts of Assumptions~\ref{assum1} and~\ref{assum2} for sparsely-and-smoothly-changing GMRF.

\begin{assumption}[Bounded norm]\label{assum12}
	There exist constant numbers $\kappa_1<\infty$, $\kappa_2>0$, and $\kappa_3<\infty$ such that
	\begin{align*}
	&\|\Theta(t/T)\|_\infty\leq \kappa_1,\; \inf_{w:\|w\|_\infty=1}\|\Sigma(t/T)w\|_\infty\geq \kappa_2,\;\|\Sigma(t/T)\|_{\linf}\leq \kappa_3,
	\end{align*}
	for every $t \in[0,T]$.
\end{assumption}
% Assumption~\ref{assum1} implies that the true covariance matrices and their inverses have bounded norms. 
\begin{assumption}[Weak sparsity]\label{assum22} The covariance matrix $\Sigma(t/T)$ satisfies $\max_{i}\sum_{j=1}^d|[\Sigma(t/T)]_{ij}|^q\leq s(q,d)$, for some $0\leq q<1$, and every $t \in[0,T]$.
\end{assumption}
Under Assumptions~\ref{assum12}-\ref{assum22}, we present the analog of Theorem~\ref{cor_GMRF} for sparsely-and-smoothly-changing GMRFs. 
% \noteag{Corollary~\ref{cor_GMRF} is not presented as a corollary, but is actually Theorem~\ref{cor_GMRF}}
\begin{theorem}\label{cor_GMRF_ker}
	Consider a sparsely-and-smoothly-changing GMRF with one sample per time, let $\zeta = \max\{\log_d(T+1), 1\}$, and suppose that the sample covariance matrices are constructed according to~\eqref{eq_kernal} with $h\asymp T^{-1/3}$. Given an arbitrary $\tau>\zeta+2$, let $T\gtrsim{s(q,d)}^{\frac{3}{1-q}}(\tau\log d)^{3/2}$.
	% \begin{align}
	%     % N_t\gtrsim\left(\frac{\kappa_1\kappa_3}{\kappa_2}\right)^2{s(q,d)}^{\frac{2}{1-q}}\zeta\log(d\vee N_t)
	%     T\gtrsim{s(q,d)}^{\frac{3}{1-q}}(\zeta\log d)^{3/2}
	% \end{align}
	Then, with the approximate backward mapping $\widetilde F^*(\widehat{\Sigma}_t) = \texttt{ST}_{\nu_t}(\widehat{\Sigma}^w_t)$ and parameters $\nu_t \asymp {\frac{\sqrt{\tau\log d}}{T^{1/3}}}$ and $\lambda_t \asymp {\frac{\sqrt{\tau\log d}}{T^{1/3}}}$,
	% \begin{align}
	%     % \nu_t \asymp \kappa_3\sqrt{\frac{\zeta\log\left(d\vee N_t\right)}{N_t}}, \qquad\qquad \lambda_t \asymp \frac{\kappa_1\kappa_3}{\kappa_2}\sqrt{\frac{\zeta\log\left(d\vee N_t\right)}{N_t}},
	%     \nu_t \asymp {\frac{\sqrt{\zeta\log d}}{T^{1/3}}}, \qquad\qquad \lambda_t \asymp {\frac{\sqrt{\zeta\log d}}{T^{1/3}}},
	% \end{align}
	the estimates $\{\widehat{\Theta}_t\}_{t=0}^T$ obtained from~\eqref{generalopt} satisfy the following statements for all $T=0,1,\dots, T$ with probability of at least $1-d^{-\tau+\zeta+2}$:
	\begin{itemize}
		\item[-] \textit{(Sparsistency)} We have $\mathrm{supp}\left(\widehat{\Theta}_t\right) = \mathrm{supp}\left({\Theta}^*_t\right)$ and $\mathrm{supp}\left(\widehat{\Theta}_t-\widehat{\Theta}_{t-1}\right) = \mathrm{supp}\left({\Theta}^*_t-{\Theta}^*_{t-1}\right)$.
		\item[-] (Estimation error) We have
		{\small
			\begin{align*}
			\hspace{-7mm}\|\widehat{\Theta}_t\!-\!\Theta_t^*\|_{\linf}&\!\lesssim\! \sqrt{\frac{\tau\log d}{T^{2/3}}},\;\|\widehat{\Theta}_t\!-\!\Theta_t^*\|_F\lesssim\! \sqrt{\frac{\tau|\mathcal{S}_t|\log d}{T^{2/3}}}.
			\end{align*} 
		}
	\end{itemize}
\end{theorem}
% \begin{proof}
% 	The proof is provided in the Appendix~\ref{app_cor_GMRF_ker}.
% \end{proof}
Theorem~\ref{cor_GMRF_ker} shows how the smoothness assumption on the true covariance matrix can be used to construct the backward mappings using the samples collected during the \textit{entire} time horizon, thereby significantly reducing the sample complexity of learning time-varying GMRFs. In particular, leveraging the smoothness of the covariance matrix can reduce the minimum required number of samples from $\mathcal{O}(T\log d)$ to $\mathcal{O}((\log d)^{1.5})$. On the other hand, Theorem~\ref{cor_GMRF} does not impose any lower bound on $T$, and its estimation error decays faster in terms of the sample size.

% Moreover, the sufficient statistics is characterized by $\{X_t, X_tX_t^\top\}$ with mean parameters $(\nu_t,\Sigma_t) = \left(\mathbb{E}_{\theta_t}[X_t], \mathbb{E}_{\theta_t}[X_tX_t^\top]\right)$.

\section{Solution Method}\label{sec:algorithm}
In this section, we describe the proposed algorithm for solving~\eqref{generalopt}. For the simplicity of notation, we define the lower bound and upper bound vectors $l_t$ and $u_t$ as $l_{t;k} =  [\widetilde F^*(\widehat{\mu}_t)]_{k}-\lambda_t$ and $u_{t;k} =  [\widetilde F^*(\widehat{\mu}_t)]_{k}+\lambda_t$, for every $k=1,\dots, p$. 
% Then, the optimization~\eqref{generalopt} can be written as
% \begin{subequations}\label{generalopt2}
% 	\begin{align}
% 	\left\{\widehat{\theta}_t\right\}_{t=0}^T\!\!\in\underset{\{\theta_t\}_{t=0}^T}{\arg\min}\ \ & (1\!-\!\alpha)\sum_{t=0}^T\|\theta_t\|_0\!+\!\alpha\sum_{t=1}^T\|\theta_t\!-\!\theta_{t-1}\|_0\\
% 	\text{s.t.}\ \ & l_t\leq \theta_t \leq u_t\quad \forall 0\leq t\leq T.\label{generalopt2_bounds}
% 	\end{align}
% \end{subequations}
The following fact plays a key role in our analysis.
\begin{fact}\label{fact}
	An optimal solution $\left\{\widehat{\theta}_{t}\right\}_{t=0}^T$ of \eqref{generalopt} satisfies for every $k=1,\dots,p$,
	\begin{subequations}\label{opt_i}
		\begin{align}
		\left\{\widehat{\theta}_{t;k}\right\}_{t=0}^T\in&\underset{\{\theta_{t;k}\}_{t=0}^T}{\arg\min}\ \  (1-\alpha)\sum_{t=0}^T\mathbbm{1}\{\theta_{t;k}\not=0\}+\alpha\sum_{t=1}^T\mathbbm{1}\{\theta_{t;k}-\theta_{t-1;k}\not=0\}\\
		\mathrm{s.t.}\ \ & l_{t;k}\leq \theta_{t;k} \leq u_{t;k}\quad \forall 0\leq t\leq T
		\end{align}
	\end{subequations}
\end{fact}
Fact~\ref{fact} implies that~\eqref{generalopt} decomposes into the smaller subproblems~\eqref{opt_i}. Therefore, our main focus is devoted to solving each subproblem independently. To further simplify the notation, we drop the subscript $k$ from~\eqref{opt_i}, whenever it is chosen arbitrarily.  Let $\texttt{OPT}_{i\to j}(\alpha)$ denote, for a given $k$, the truncated problem from index $i$ to $j$ with the regularization coefficient $\alpha$, defined as
\begin{subequations}\label{opt_ij}
	\begin{align}
	{f_{i\to j}^*=}\underset{\{\theta_{t}\}_{t=i}^j}{\min}\ \ & (1-\alpha)\sum_{t=i}^j\mathbbm{1}\{\theta_{t}\not=0\}+\alpha\sum_{t=i+1}^j\mathbbm{1}\{\theta_{t}-\theta_{t-1}\not=0\}\\
	\mathrm{subject\ to}\ \ & l_{t}\leq \theta_{t} \leq u_{t}\quad \forall i\leq t\leq j. \label{opt_ij_bounds}
	\end{align}
\end{subequations}
Let the objective function {for a candidate solution $\theta$} be denoted as $f_{i\to j}(\theta)${; by convention, we let  $f_{i\to j}(\theta)=0$ whenever $j<i$}. Moreover, the optimal objective value and the set of optimal solutions to $\texttt{OPT}_{i\to j}(\alpha)$ are respectively denoted as $f_{i\to j}^*$ and $\mathcal{X}_{i\to j}^*$. Similarly, $\widehat\theta_{i\to j}\in\mathcal{X}_{i\to j}^*$ is used to denote an optimal solution to $\texttt{OPT}_{i\to j}(\alpha)$. TWe omit the subscript $i\to j$ whenever $i=0$ and $j=T$. The $t^{\text{th}}$ feasible interval is defined as $\Delta_t=[l_t,u_t]$. Accordingly, the notation $\DI_{t\to s}$ refers to $\DI_{t\to s}\defeq\Delta_t\cap\Delta_{t+1}\cap\dots\cap\Delta_s$. 

\subsection{Special case: $\pmb{\alpha=1}$}
As the first step, we consider the special case $\alpha=1$, and provide an efficient algorithm (Algorithm~\ref{alg_greedy}) for solving $\texttt{OPT}_{0\to T}(1)$, where the sparsity is only promoted on the parameter differences (and not on the individual parameters). As will be shown later, Algorithm~\ref{alg_greedy} will be used as a subroutine in our proposed algorithm for the general case $0<\alpha<1$. At a high level, Algorithm~\ref{alg_greedy} recursively performs the following operations: at any given time $\tau$, the algorithm looks into the future to find a nonempty interval that is feasible for the longest possible time $\delta$. Then, it sets the subvector $\theta_{\tau:\delta}$ to an arbitrarily chosen element from this nonempty interval.
\begin{algorithm}
	\caption{$\texttt{Greedy}(l,u,\tau,T)$}
	\label{alg_greedy}
	\begin{algorithmic}[1]
% 		\STATE{{\bf Input:} $l$, $u$, $\tau$, and $T$}
		\STATE{{\bf Output:} Solution $\theta_{\tau\to T}^{\texttt{Greedy}}$, the objective value $f^{\texttt{Greedy}}_{\tau\to T}$ to $\texttt{OPT}_{\tau\to T}(1)$, and the index set $\Gamma$ of maximal nonempty intervals}
		\vspace{2mm}
		\STATE{Find largest $\delta$ such that $\DI_{\tau\to\delta}\not=0$;}\label{alg_greedy:delta}
		\STATE{Set $\theta_{\tau:\delta}^{\texttt{Greedy}} =\eta$ for some $\eta\in\DI_{\tau\to\delta}$;}
		% 		\IF{$\delta = T-1$}
		% 		\STATE{Set $\theta_T^{\texttt{Greedy}} = \eta$ for some $\eta\in\Delta_T$} \noteag{I don't think we need this}
		\IF{$\delta\leq T-1$}
		\STATE{Execute \texttt{Greedy}$(l,u,\delta+1,T)$;}
		\ENDIF
		\STATE{Set $f^{\texttt{Greedy}}_{\tau\to T} = \sum_{t=\tau+1}^T\mathbbm{1}\{\theta^{\texttt{Greedy}}_{t}-\theta^{\texttt{Greedy}}_{t-1}\not=0\}$;}
		\STATE{Set $\Gamma \leftarrow \Gamma\cup \{\delta\}$;}
		\STATE \textbf{Return }{$\{\theta^{\texttt{Greedy}}_t\}_{t=\tau}^T$, $f^{\texttt{Greedy}}_{\tau\to T}$, and $\Gamma$;}
	\end{algorithmic}
\end{algorithm}

\begin{proposition}\label{prop_beta1}
	$\texttt{Greedy}(l,u,0,T)$ returns an {optimal solution $\{\theta^{\texttt{Greedy}}_t\}_{t=0}^T$} to $\texttt{OPT}_{0\to T}(1)$. {Moreover, the truncated solution $\{\theta^{\texttt{Greedy}}_t\}_{t=0}^j$ is optimal for $\texttt{OPT}_{0\to j}(1)$.}
\end{proposition}
% \begin{proof}
% 	The proof is provided in Appendix~\ref{app_prop_beta1}.
% \end{proof}

\subsection{General case: $\pmb{0<\alpha<1}$}
Now, we present our main algorithm for the general case $0<\alpha<1$. To this goal, we first present the following definition.
\begin{definition}
	The set $\mathcal{Z}_{i\to j} = \{i,i+1,\dots,j\}$ is called a {\bf zero-feasible sequence} if $l_k\leq 0\leq u_k$ for every $k\in\mathcal{Z}_{i\to j}$. Moreover, the zero-feasible sequence $\mathcal{Z}_{i\to j}$ is called {\bf maximal} if it is not strictly contained within another zero-feasible sequence.
\end{definition}

Let $\mathcal{Z}_{i_1\to j_1}, \mathcal{Z}_{i_2\to j_2}, \dots, \mathcal{Z}_{i_Z\to j_Z}$ be the set of all maximal zero-feasible sequences such that $0\leq i_1\leq j_1<i_2\leq j_2<\dots<i_Z\leq j_Z\leq T$, where $Z$ is the number of maximal zero-feasible sequences. If $Z=0$, i.e., there is no zero-feasible sequence, then it is easy to see that $\sum_{t=0}^T\mathbbm{1}\{\theta_t\not=0\} = T+1$ for every feasible solution, and hence, $\texttt{Greedy}(l,u,0,T)$ leads to an optimal solution to~\eqref{opt_ij}. Another special case is when $i_1 = 0$ and $j_1=T$, i.e., zero is always feasible. In this case, the optimal solution is $\widehat\theta = 0$. Therefore, without loss of generality, suppose that $Z>0$ and either $i_1\not=0$ or $j_1\not=T$. {As will be shown later, the zero-feasible sequences play an important role in characterizing the optimal solution of \eqref{opt_ij}.
}

% {Since Lemma~\ref{prop:allZero} holds for any zero-feasible sequence, it holds in particular for all maximal zero-feasible sequences. }
Our goal {now} is to obtain an optimal solution to~\eqref{opt_ij} by solving a shortest path problem over a \textit{weighted directed acyclic graph} (DAG) whose nodes correspond to the maximal zero-feasible sequences.
In particular, consider a weighted DAG $\mathcal{G}$ with the vertex set $V = \{0,1,\dots,Z, Z+1\}$, where the vertices $k$ and $l$ are connected via a directed arc $(k,l)$ if $k<l$. Moreover, for every arc $(k,l)$, the weight $W(k,l) = 0$ if $(k,l) = (0,1),i_1=0$ or $(k,l) = (Z,Z+1),j_Z=T$, and
\begin{align}\label{eq_A}
W(k,l) =& 
(1-\alpha)(i_l-j_k-1) + \alpha f^{\texttt{Greedy}}_{j_k+1\to i_l-1}+\alpha\mathbbm{1}\{k\not=0\}+\alpha\mathbbm{1}\{l\not=T+1\}
\end{align}
otherwise, where we define  $j_{0} = -1$ and $i_{T+1} = T+1$.
\begin{algorithm}
	\caption{Algorithm for solving~\eqref{opt_ij}}
	\label{alg_path}
	\begin{algorithmic}[1]
% 		\STATE{{\bf Input:} $l$, $u$, and $\alpha$}
		\STATE{{\bf Output:} Optimal solution $\widehat\theta$ and the objective value $f^*$ to $\texttt{OPT}_{0\to T}(\alpha)$}
		\vspace{2mm}
		\STATE{Find the maximal zero-feasible sequences;
% 		$\bar{\mathcal{Z}} = \{\mathcal{Z}_{i_1\to j_1}, \mathcal{Z}_{i_2\to j_2}, \dots, \mathcal{Z}_{i_Z\to j_Z}\}$;
		}
		\STATE{Construct the DAG $\mathcal{G}$ with weights defined as~\eqref{eq_A};
% 		the node set $V=\{0,1,\dots,Z+1\}$, arc set $A = \{(k,l): k<l\}$, and the weights defined as~\eqref{eq_A};
		}\label{alg_path_construct}
		\STATE{Find the shortest path $p = (v_1,v_2,\dots,v_r)$ between the vertices $0$ and $Z+1$ in $\mathcal{G}$;}\label{alg_path_solve}
		\STATE{Set $\widehat{\theta}_{j_{v_i}+1:i_{v_{i+1}}-1} = \theta^{\texttt{Greedy}}_{j_{v_i}+1\to i_{v_{i+1}}-1}$ and $\widehat{\theta}_{i_{v_l}:j_{v_{l+1}}}  = 0$ for every $l = 1,2,\dots, r$;}
		% 		\IF{there is no zero-feasible sequence}
		% 		\STATE{$(\theta^{\texttt{Greedy}},f^{\texttt{Greedy}}) = \texttt{Greedy}(\texttt{OPT}_{0\to T}(1))$,}
		% 		\STATE{Set $\theta^* = \theta^{\texttt{Greedy}}$ and $f^* = (1-\alpha)(T+1)+\alpha f^{\texttt{Greedy}}$,} \noteag{I don't think we need this case}
		% 		\ELSIF{$i_1 = 0$ \textbf{and} $j_1=T$}
		% 		\STATE{Set $\theta^*=0$ and $f^* = 0$,}\noteag{I don't think we need this case either}
		% 		\ELSE 
		% % 		\STATE{$A = \texttt{ConstructWeight}(l,u,\alpha)$,}
		% % 		\STATE{$(\mathcal{P},f^*) = \texttt{Dijkstra}(A)$,}
		% 		\STATE{Construct adjacency matrix $A = \texttt{ConstructWeight}(l,u,\alpha)$,}\label{alg_path_construct}
		% 		\STATE{Find the shortest path between vertices $0$ and $Z+1$ in $\mathcal{G}(V,A)$}\label{alg_path_solve}
		% 		\ENDIF
		\STATE \textbf{Return }{$\{\widehat\theta\}_{t=0}^T$ and $f^*$}
	\end{algorithmic}
\end{algorithm}
% \notesf{I think it would be better if we write a more detailed subroutine (or pseudocode) for steps 4, 5, and 6 in the appendix.}

% \noteag{I think line~5 does not make sense, since it is impossible to figure out how to do so from the contents of the main body. I think it is important to state in the main body how to recover the solution.}

{
	\begin{theorem}\label{thm_spath}
		The shortest path from $0$ to $Z+1$ on $\mathcal{G}$ has value $f^*$. 
	\end{theorem}
 	The above theorem implies that the optimal solution to~\eqref{generalopt} can be obtained via Algorithm~\ref{alg_path}. 

	\begin{theorem}\label{thm_runtime_sp}
		Problem \eqref{opt_ij} can be solved in $\mathcal{O}(ZT)$ time and memory. 
	\end{theorem}
% 	\begin{proof}
% 		{Algorithm~\ref{alg_path} involves three main components: construct graph $\mathcal{G}$ (line \ref{alg_path_construct}), solve a shortest problem on the constructed graph (line \ref{alg_path_solve}), and recover the optimal solution from the obtained shortest path.} Since $\mathcal{G}$ is acyclic, the shortest path problem can be solved in time linear in the number of arcs, which is $\mathcal{O}(Z^2)$, via a simple labeling algorithm; see, e.g., Chapter 4.4. in \cite{ahuja1988network}. Constructing graph $\mathcal{G}$ requires computing the costs of all arcs. A na\"ive implementation, where Algorithm~\ref{alg_greedy} is called for every arc, would require $\mathcal(O)(Z^2T)$ time and memory. However, from the second statement in Proposition~\ref{prop_beta1}, we note that a single call to \texttt{Greedy}$(l,u,i,T)$ allows us to compute $f_{i\to j}^{\texttt{Greedy}}$ for all $i\leq j\leq T$. Therefore, Algorithm~\ref{alg_greedy} needs to be invoked only $\mathcal{O}(Z)$ times, and each call require $\mathcal{O}(T)$ leading to a total complexity of $\mathcal{O}(ZT)$. {Moreover, given the shortest path, the optimal solution can be constructed by concatenating the solutions obtained from the calls of \texttt{Greedy}.} Finally, since $Z\leq T+2$, we find that the overall complexity is dominated by that of constructing the graph. This completes the proof. 
% 	\end{proof}
	
	Since $Z=\mathcal{O}(T)$ and solving \eqref{generalopt} requires solving $\mathcal{O}(p)$ instances of \eqref{opt_ij}, we find the total complexity stated in Theorem~\ref{thm_runtime}. 
% 	Note however that if the graph contains few maximal zero-feasible sequences, {then Algorithm~\ref{alg_path} could be substantially faster than what the worst case scenario indicated in Theorem~\ref{thm_runtime}. 
	Note however that if $Z = \mathcal{O}(1)$ for every instance of~\eqref{generalopt}, then the overall complexity reduces to $\mathcal{O}(pT)$, which is linear in the total number of variables.
	In the next section, we will show that the practical runtime of the proposed algorithm is near-linear with respect to the number of variables.
}

\section{Numerical Analysis}\label{sec:simulations}
In this section, we provide detailed information about the performance of the proposed estimator in different case studies. In the first case study, our goal is to compare the statistical performance of our proposed method with two other state-of-the-art methods, namely time-varying Graphical Lasso~\cite{hallac2017network}, and a modified version of the elementary $\ell_1$ estimator~\cite{wang2018fast, yang2014elementary}. We will show that the proposed estimator outperforms the other two estimators, in terms of both sparsity recovery and estimation error. In the second case study, we showcase the statistical and computational performance of the proposed method on massive-scale datasets. In particular, we will show that our proposed estimation method can solve instances of the problem with more than 500 million variables in less than one hour, with almost perfect sparsity recovery. Moreover, we demonstrate the improvements in the runtime of our algorithm with parallelization. Finally, we conduct a case study on the correlation network inference in stock markets. In particular, we show that the inferred time-varying graphical model can correctly identify the stock market spikes based on the historical data.

All simulations are run on a desktop computer with an Intel Core i9 3.50 GHz CPU and 128GB RAM. The reported results are for an implementation in MATLAB R2020b.

\begin{figure*}[ht]\centering
	\subfloat[]{%
		\includegraphics[width=5.5cm]{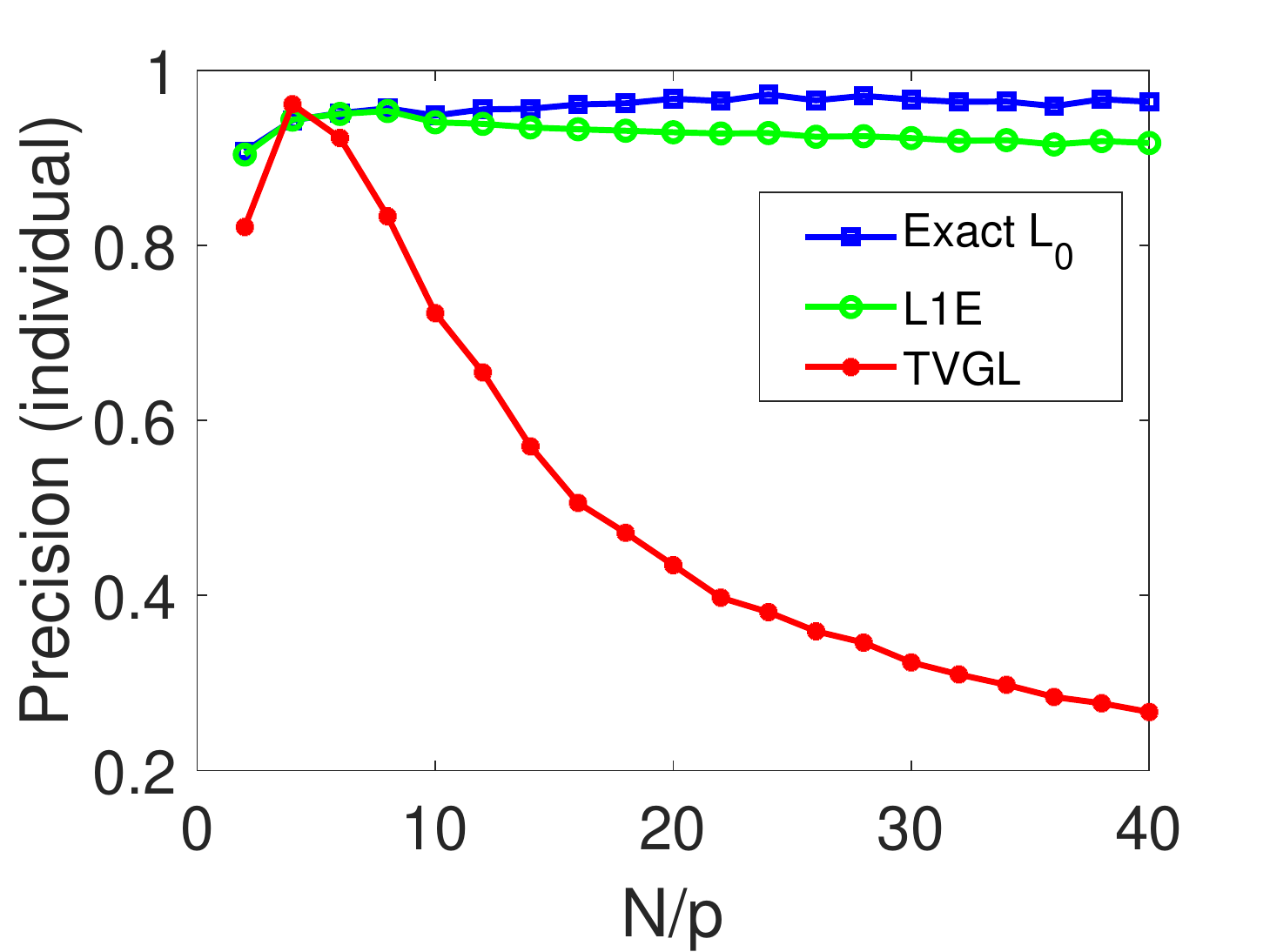}\label{fig:T_precision}%
	}
	\subfloat[]{%
		\includegraphics[width=5.5cm]{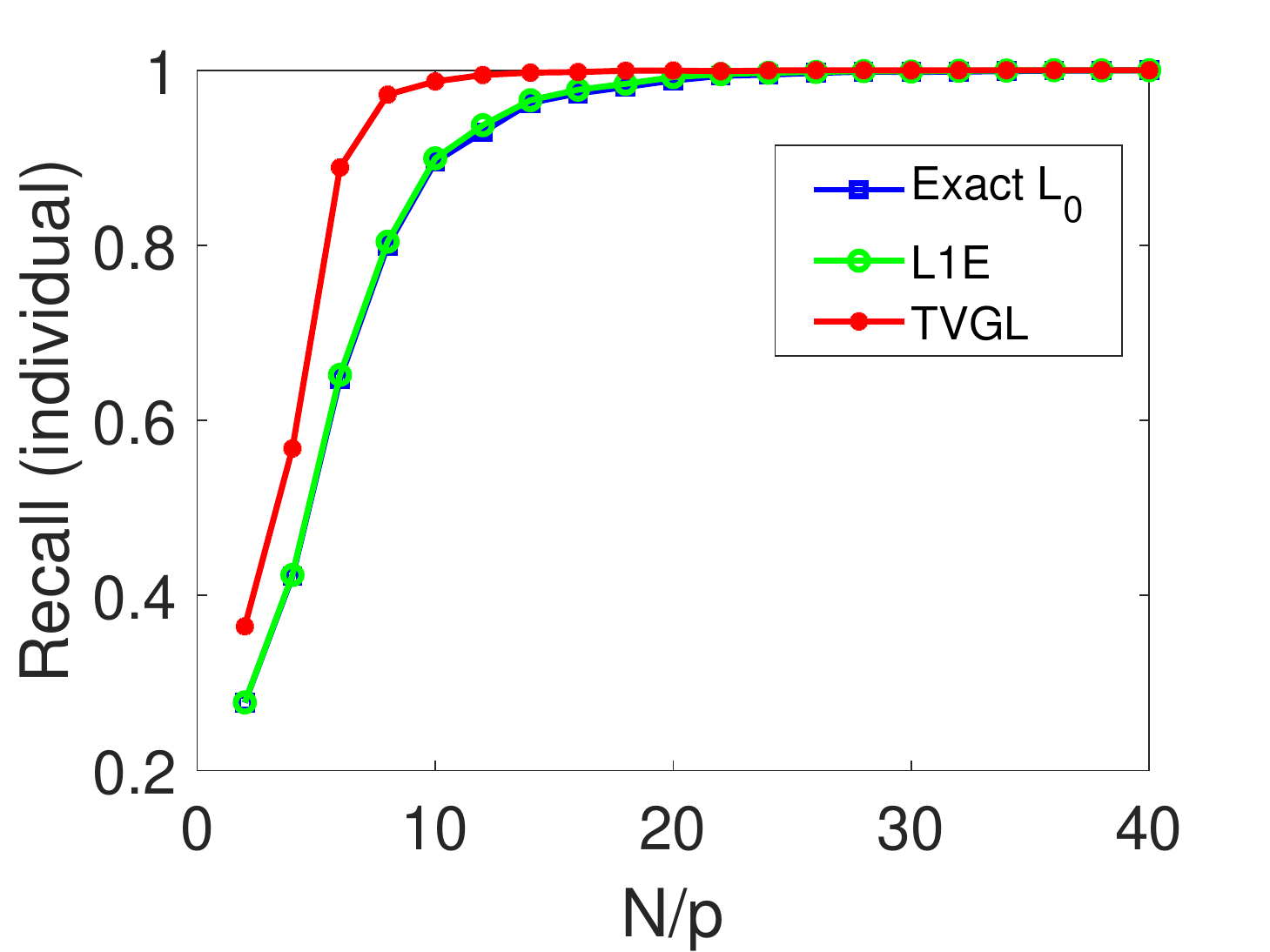} \label{fig:T_precision_diff}%
	}
	\subfloat[]{%
		\includegraphics[width=5.5cm]{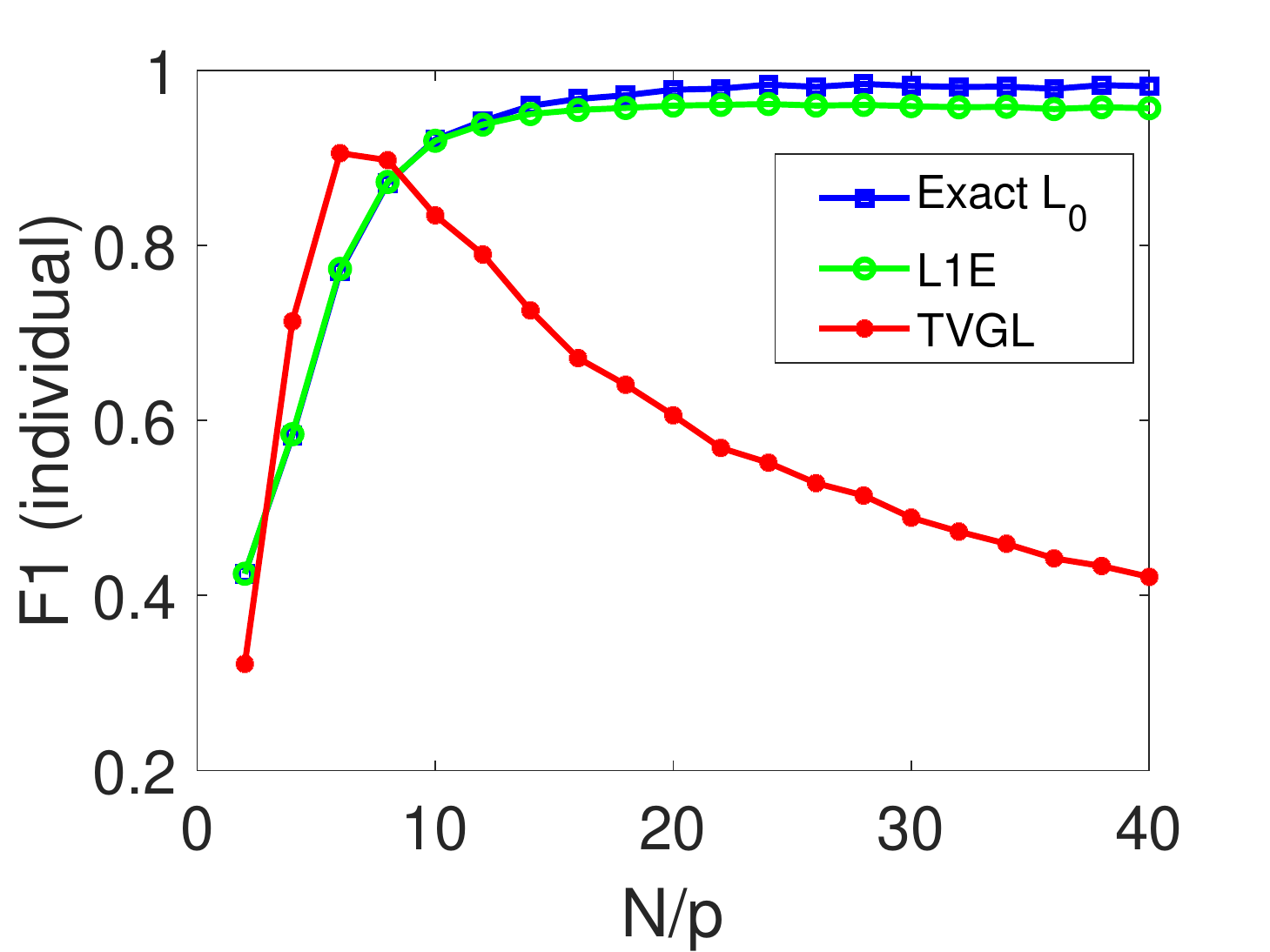} \label{fig:T_recall}%
	}
	
	\subfloat[]{%
		\includegraphics[width=5.5cm]{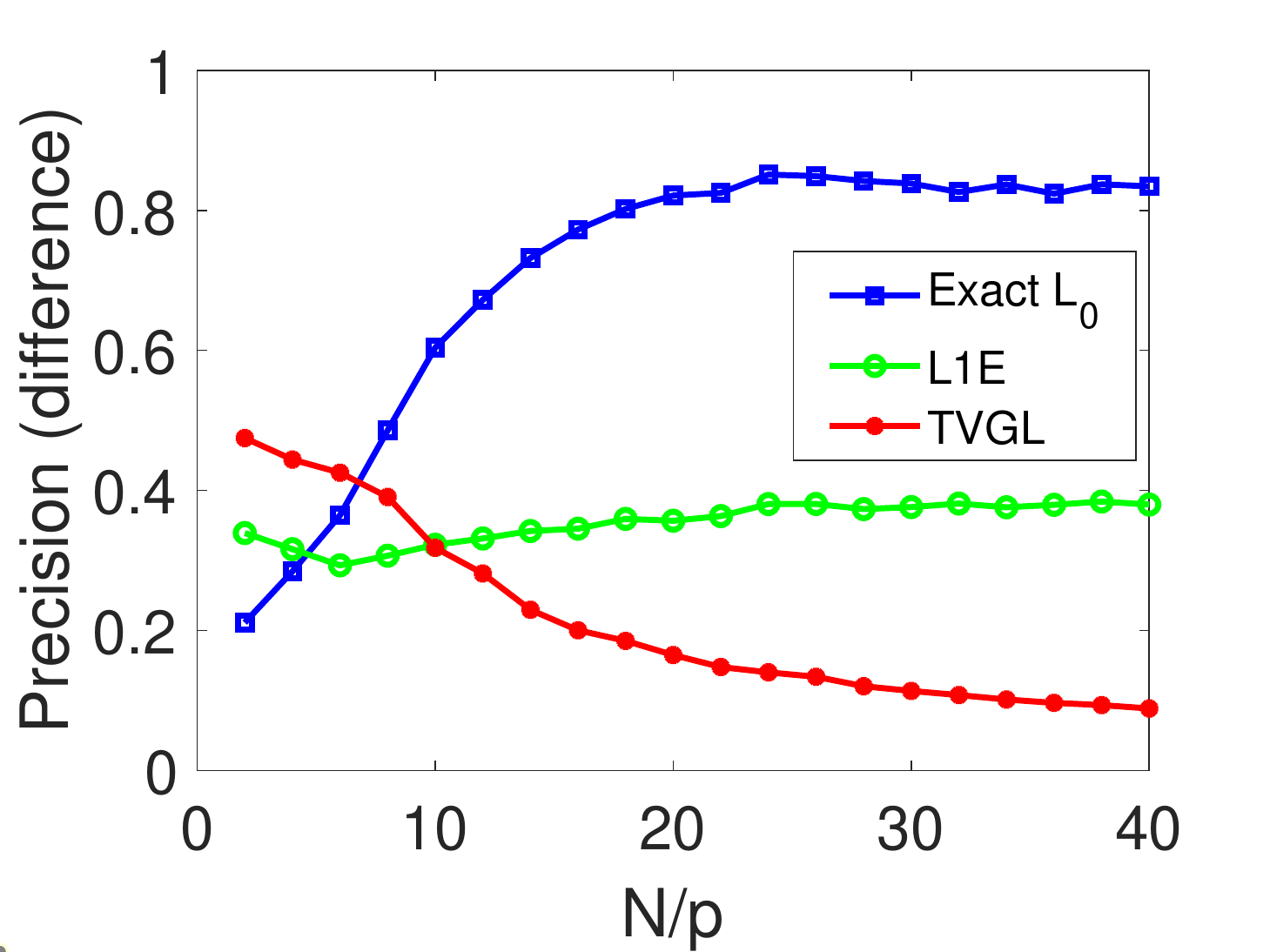}\label{fig:T_recall_diff}%
	}
	\subfloat[]{%
		\includegraphics[width=5.5cm]{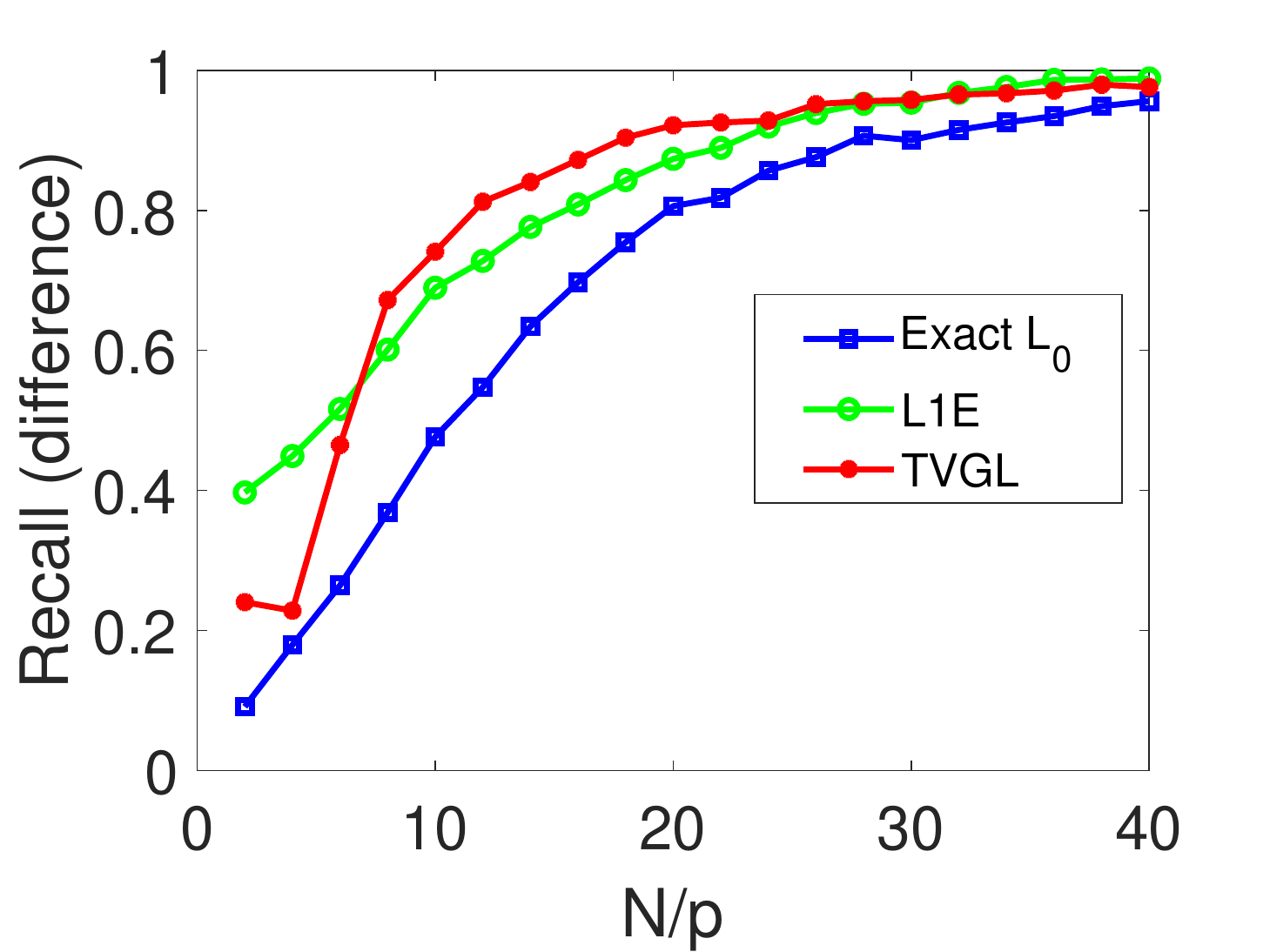} \label{fig:T_F1}%
	}
	\subfloat[]{%
		\includegraphics[width=5.5cm]{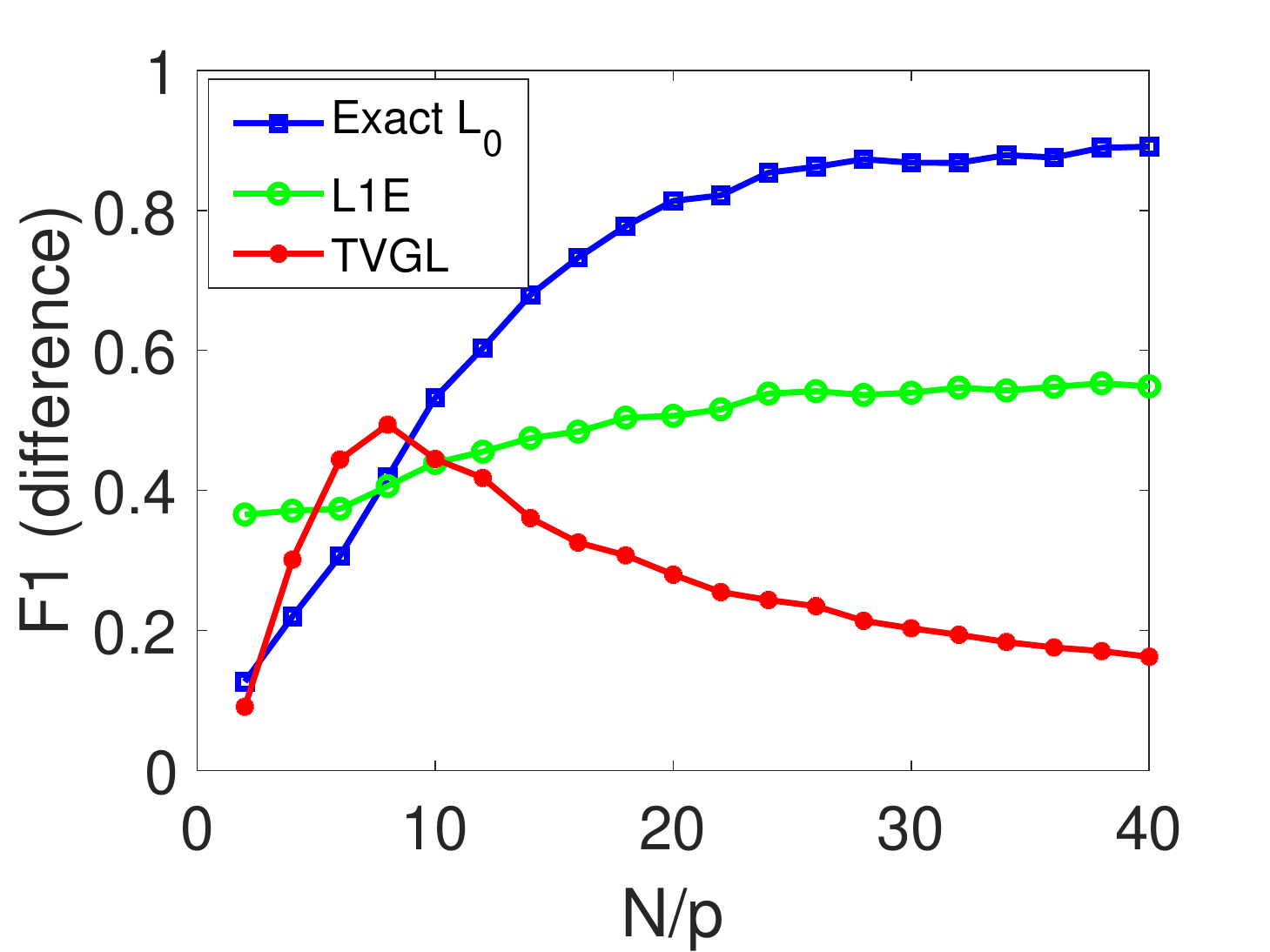} \label{fig:T_F1_diff}%
	}
	\caption{\texttt{Precision}, \texttt{Recall}, and \texttt{F1-score} for the estimated precision matrices and their differences using the proposed method (denoted as \texttt{Exact L$_0$}), \texttt{L1E}, and \texttt{TVGL} (averaged over 10 independent trials).}\label{fig_errors1_small}
\end{figure*}

\begin{figure*}[ht]\centering
	\subfloat[]{%
		\includegraphics[width=6.5cm]{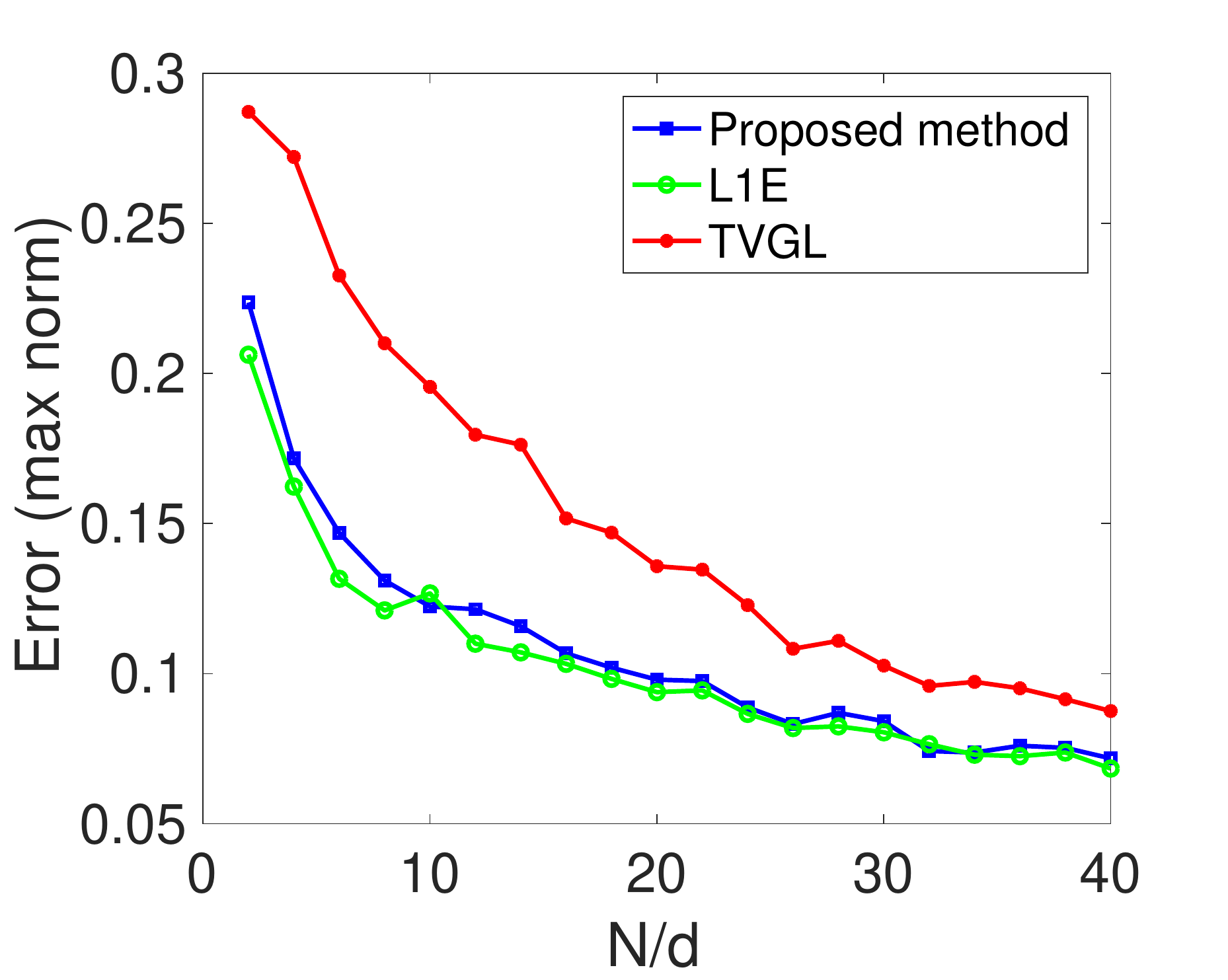}\label{fig:Maydemand}%
	}\hspace{1cm}
	\subfloat[]{%
		\includegraphics[width=6.5cm]{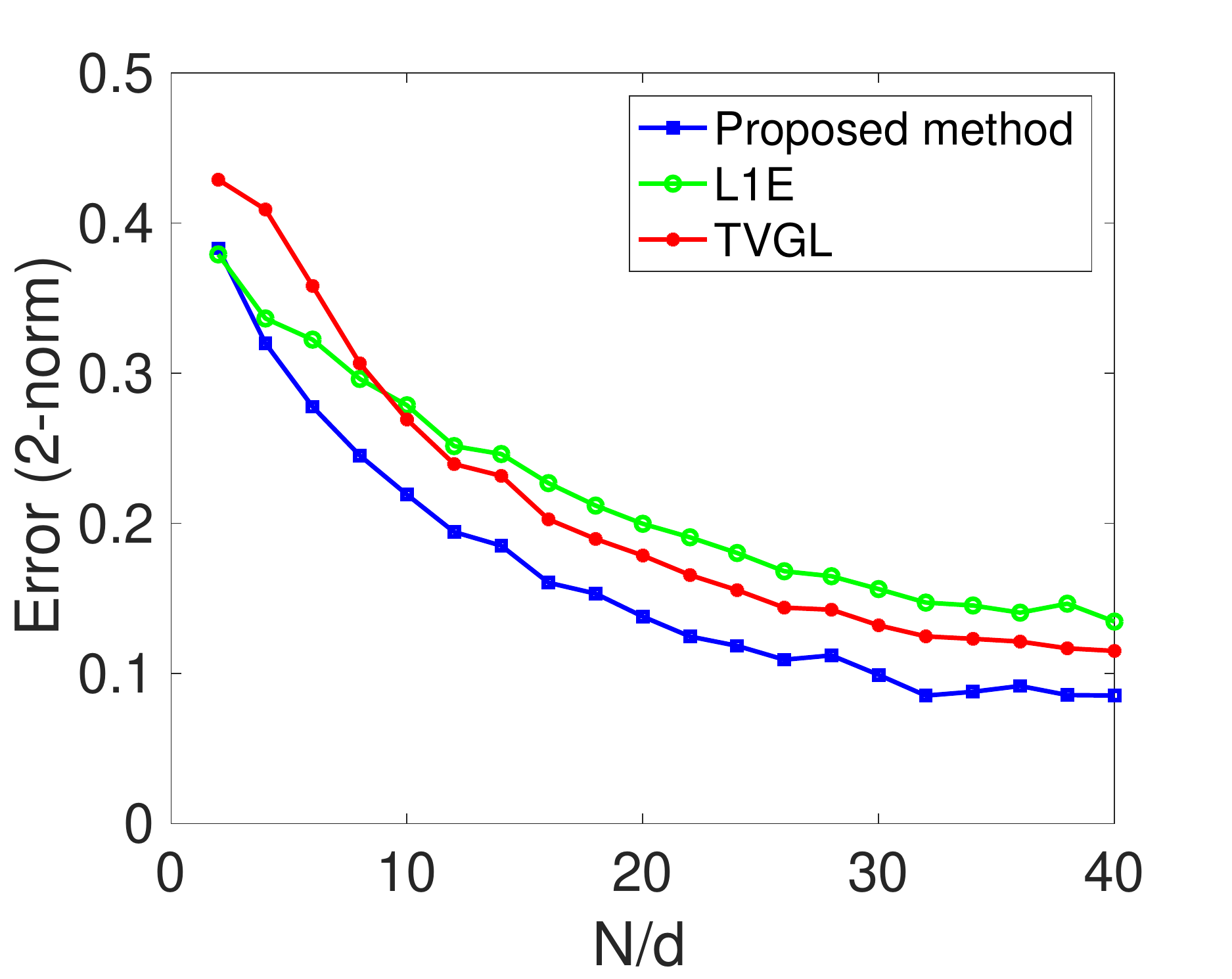} \label{fig:9bus}%
	}
	\caption{The normalized $\ell_{\infty}$-norm and induced 2-norm of the estimation error for the estimated precision matrices and their differences using the proposed method, \texttt{L1E}, and \texttt{TVGL} (averaged over 10 independent trials).}\label{fig_errors2_small}
\end{figure*}

\subsection{Case Study on Small Datasets}\label{subsec:sim_small}
In this case study, we evaluate the statistical performance of the proposed estimator, compared to two other methods, namely time-varying Graphical Lasso (\texttt{TVGL})~\cite{hallac2017network, cai2018capturing}, and a modified version of the elementary $\ell_1$ estimator (\texttt{L1E}) introduced in~\cite{wang2018fast, yang2014elementary}. As mentioned in the introduction, \texttt{TVGL} is a well-known regularized MLE approach for estimating the sparsely-changing GMRFs. On the other hand, different variants of \texttt{L1E} have been used to estimate static MRFs~\cite{yang2014elementary}, and differential networks with sparsity imposed only on the parameter differences~\cite{wang2018fast}. Consider an $\ell_1$ relaxation of the proposed estimator~\eqref{generalopt}, where the $\ell_0$ penalties in the objective function are replaced with $\ell_1$ penalties. The resulted estimator reduces to that of~\cite{yang2014elementary} for $T=0$, and~\cite{wang2018fast} for $T=1$ and $\gamma = 1$.

We consider randomly generated instances of sparsely-changing GMRFs, where the true inverse covariance matrix is constructed as follows: at time $t=0$, we set $\Theta_0 = I_{d\times d}+\sum_{(i,j)\in \mathcal{S}}A^{(i,j)}$, where $d=50$ and $A^{(i,j)}$ is a sparse positive semidefinite matrix with exactly two nonzero off-diagonal elements. In particular, we randomly select $100$ edges in the graph (corresponding to $200$ off-diagonal entries in $\Theta_0$) and collect their indices in $\mathcal{S}$. For every $(i,j)\in\mathcal{S}$, we set $A^{(i,j)}_{ij} = A^{(i,j)}_{ji} = -0.4$ and $A^{(i,j)}_{ii} = A^{(i,j)}_{jj} = 0.4$. Clearly, $A^{(i,j)}\succeq 0$, and hence, $\Theta_0\succ 0$. Moreover, at every time $t=1,\dots,9$, exactly 20 nonzero off-diagonal entries are added to $\Theta_t$ according to the aforementioned rules, and 20 nonzero nonzero off-diagonal entries are deleted by reversing the above procedure. Our goal is to estimate the true sparsely-changing precision matrices $\{\Theta_t\}_{t=0}^9$ based on a varying number of samples $N_t$. We evaluate the accuracy of the different methods in terms of \texttt{Recall}, \texttt{Precision}, and \texttt{F1-score} values, defined as 
\begin{align}
\texttt{Recall} = \frac{\texttt{TP}}{\texttt{TP}+\texttt{FP}},\quad \texttt{Precision} = \frac{\texttt{TP}}{\texttt{TP}+\texttt{FN}}, \quad \texttt{F1-score} = \frac{2\times \texttt{Recall}\times \texttt{Precision}}{\texttt{Recall}+\texttt{Precision}},
\end{align}
where $\texttt{TP}$, $\texttt{FP}$, and $\texttt{FN}$ respectively denote the number of true positives, false positives, and false negatives in the estimated sequence of precision matrices. In all of our experiments, we set $\alpha = 0.7$. Moreover, according to Theorem~\ref{cor_GMRF}, we set $\lambda_t = C_1\frac{\log d}{N_t}$ and $\nu_t = C_2\frac{\log d}{N_t}$ for the proposed method and \texttt{L1E}, where $C_1$ and $C_2$ are constants that are fine-tuned for an instance of the problem, and are kept unchanged throughout the simulations. In particular, we perform an exhaustive search over the constants $C_1$ and $C_2$, and select the ones that achieve the smallest estimation error for a random instance of the problem. Similarly, we set the regularization coefficients $\gamma_1 = C_3\frac{\log d}{N_t}$ and $\gamma_2 = C_4\frac{\log d}{N_t}$ for \texttt{TVGL}~\eqref{mle_reg_gmrf}, where the constants $C_3$ and $C_4$ are selected in a similar fashion.

\begin{figure}[ht]\centering
	\subfloat[]{%
		\includegraphics[width=6.5cm]{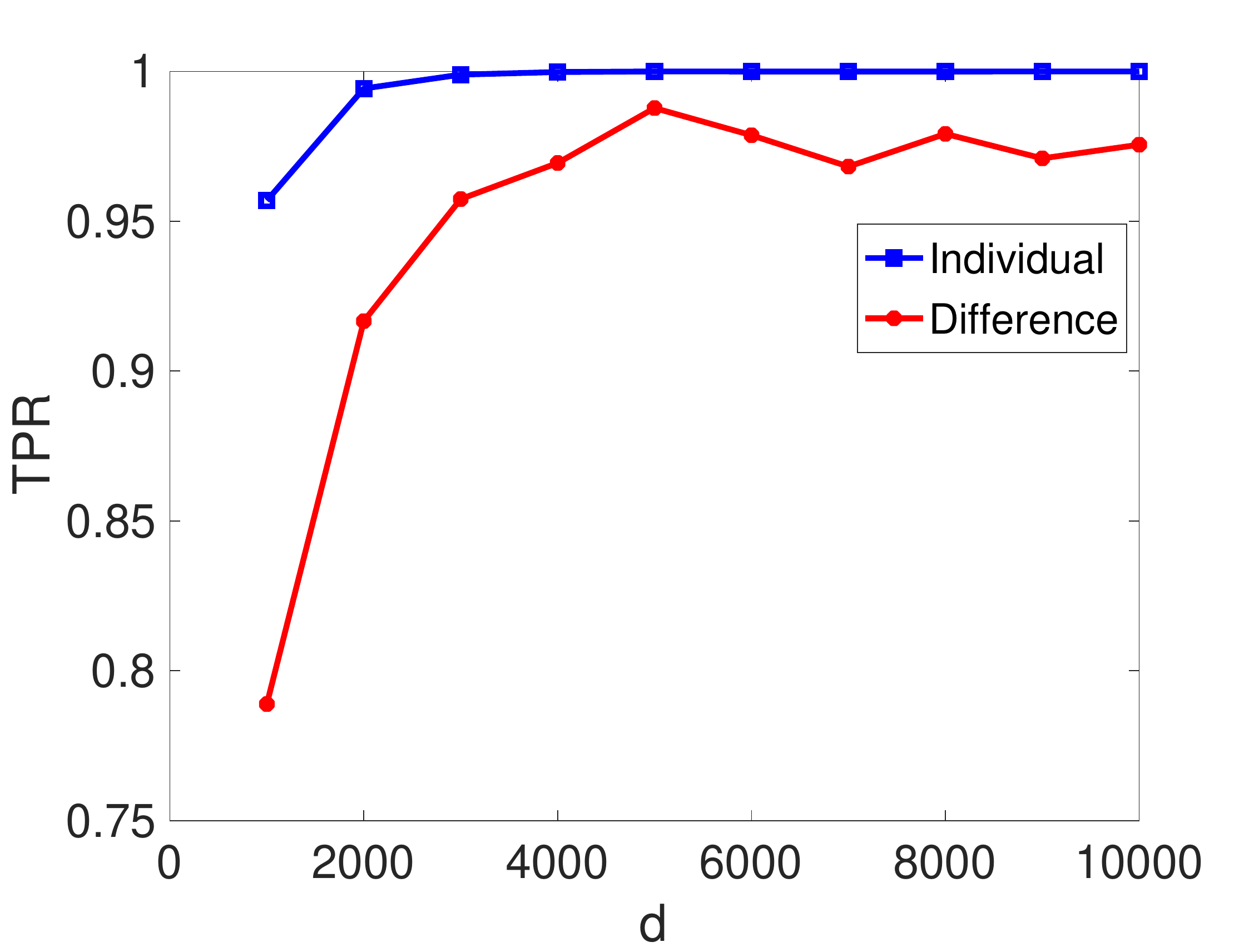}\label{fig:TPR_p}%
	}\hspace{1cm}
	\subfloat[]{%
		\includegraphics[width=6.5cm]{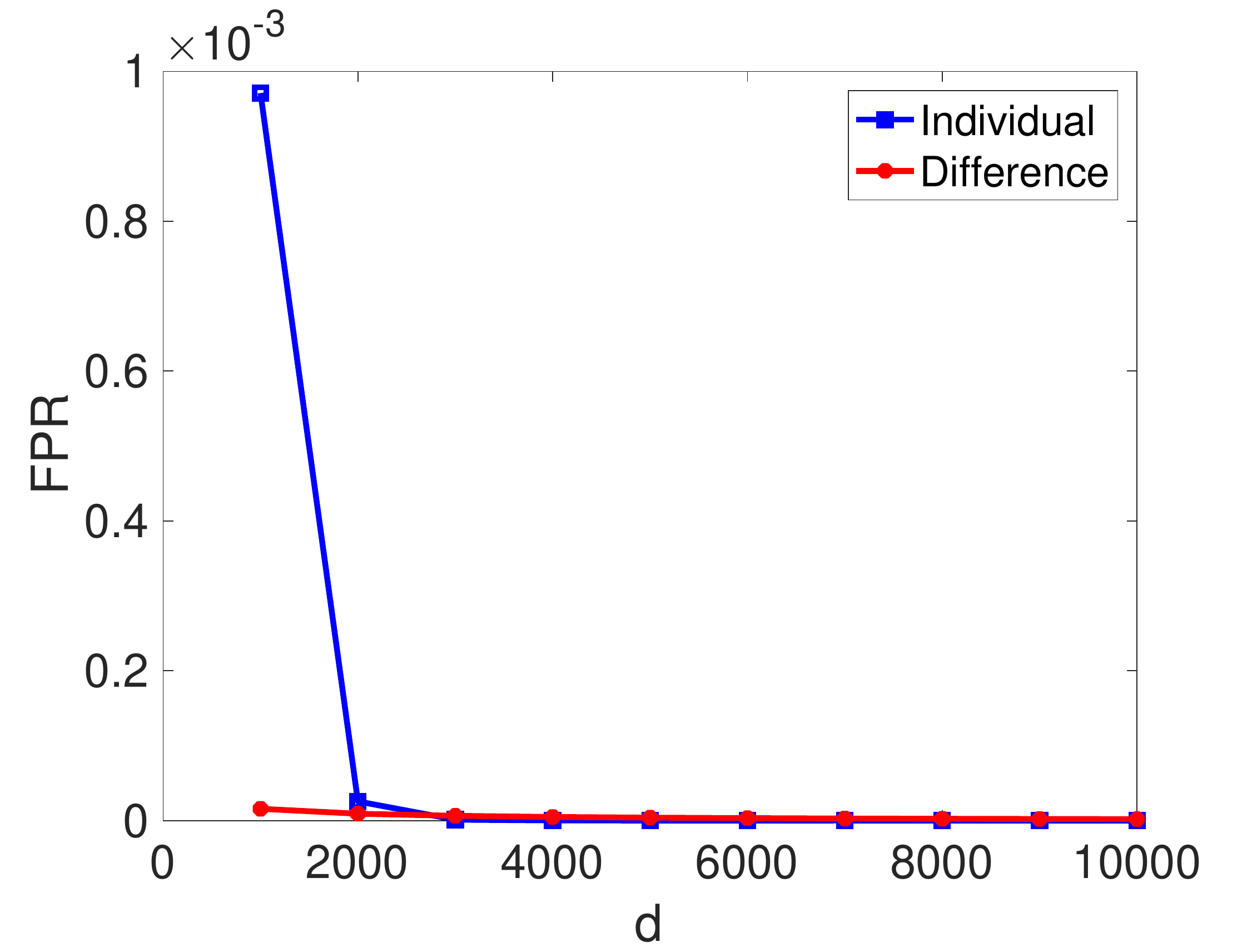} \label{fig:FPR_p}%
	}
	
	\subfloat[]{%
		\includegraphics[width=6.5cm]{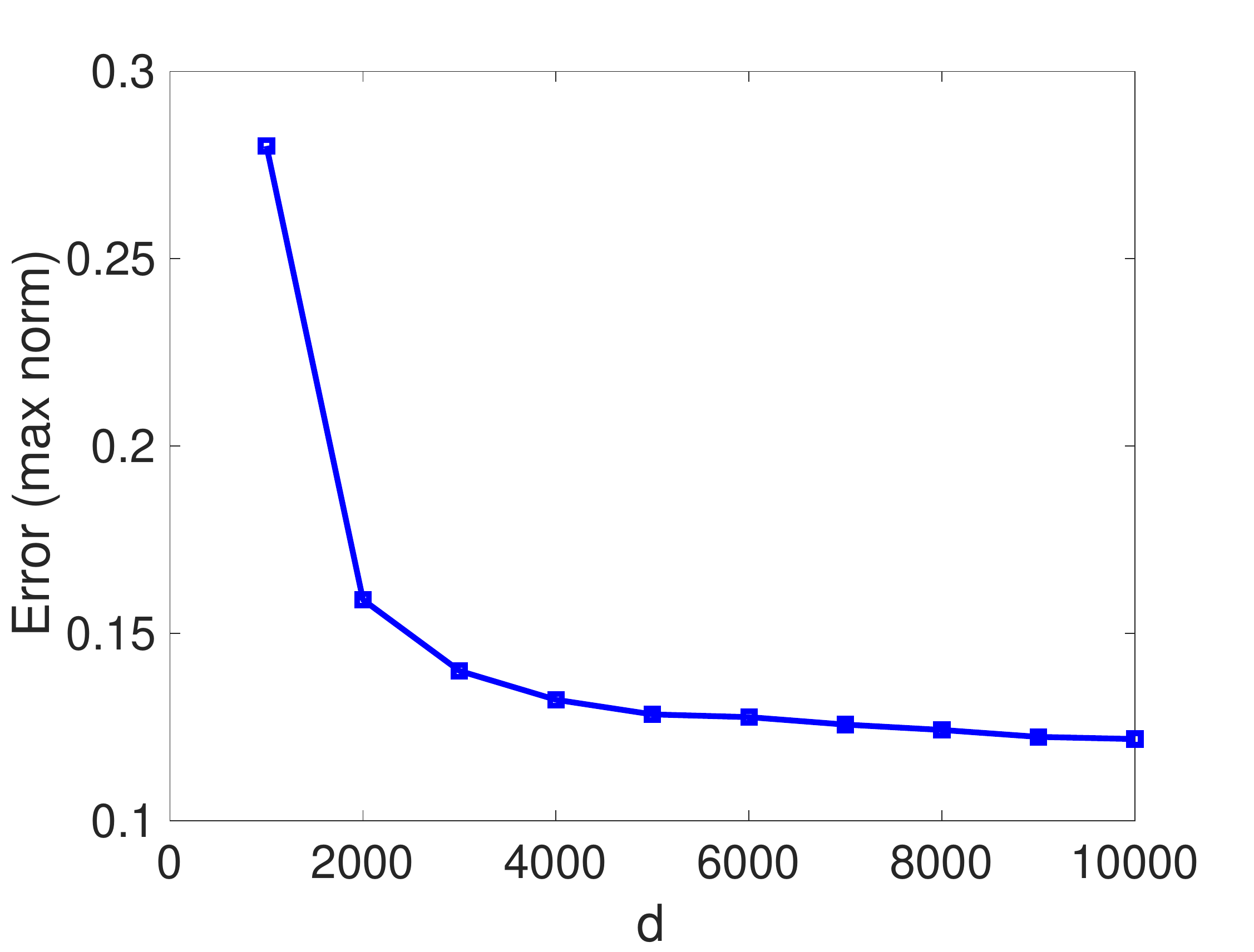}\label{fig:error_l1_p}%
	}\hspace{1cm}
	\subfloat[]{%
		\includegraphics[width=6.5cm]{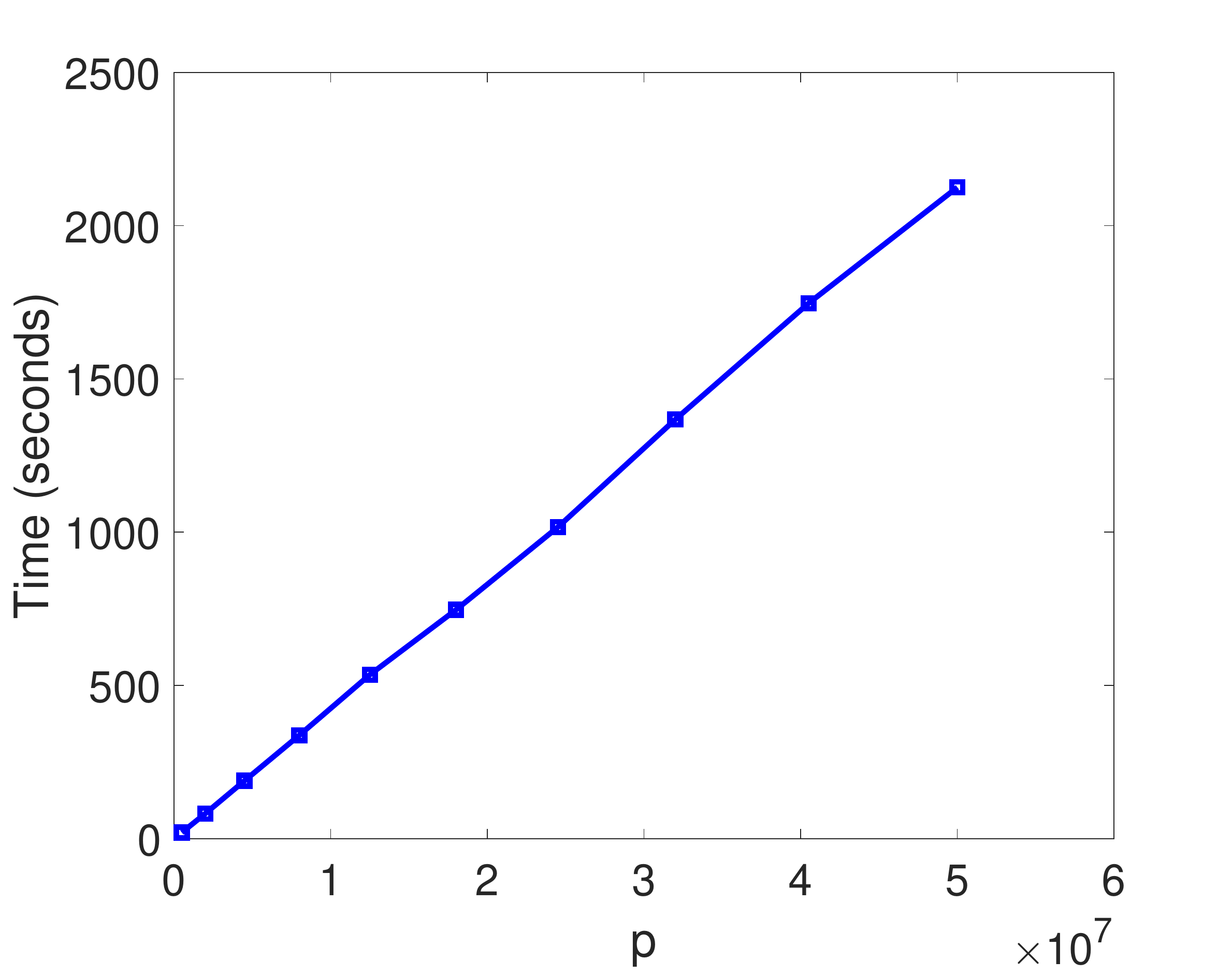} \label{fig:time_p}%
	}
	\caption{\texttt{TPR}, \texttt{FPR}, $\ell_1$-norm estimation error, and the runtime of the proposed method for fixed $T$ and different values of $d$. The number of samples $N_t$ is set to $d/2$ for every $t$. The runtime is shown with respect to $p = d(d+1)/2$.}\label{fig:p}
\end{figure}

Figure~\ref{fig_errors1_small} illustrates the accuracy of the estimated precision matrices for different number of samples. It can be seen that the proposed estimator outperforms \texttt{L1E} and \texttt{TVGL} in terms of \texttt{Precision} value, but has a slightly worse \texttt{Recall} value. In particular, both \texttt{L1E} and \texttt{TVGL} tend to \textit{overestimate} the number of nonzero elements in the precision matrices. This overestimation naturally reduces the number of false negatives (leading to better \texttt{Precision} values), while significantly increasing the number of false positives (leading to worse \texttt{Recall} values). Moreover, \texttt{F1-score} shows the overall performance of the estimates in terms of the sparsity recovery. It can be seen that the proposed estimator outperforms the other two methods. In particular, both \texttt{L1E} and \texttt{TVGL} perform poorly on the sparsity recovery of the parameter differences. Finally, Figure~\ref{fig_errors2_small} depicts the normalized $\ell_{\infty}$-norm and induced 2-norm estimation errors. It can be seen that \texttt{TVGL} incurs a relatively large $\ell_{\infty}$-norm error due to the shrinking effect of its regularization.

% \begin{figure*}
% 	\centering
% 	\includegraphics[width=1\columnwidth]{fig_errors_small.jpg}
% 	\vspace{-2mm}
% 	\caption{\texttt{Precision}, \texttt{Recall}, and \texttt{F1-score} for the estimated precision matrices and their differences using the proposed method, \texttt{L1E}, and \texttt{TVGL} (averaged over 10 independent trials).}\label{fig_errors1_small}
% \end{figure*}

% \begin{figure*}
% 	\centering
% 	\includegraphics[width=0.9\columnwidth]{Picture1.jpg}
% 	\vspace{-2mm}
% 	\caption{The normalized $\ell_{\infty}$-norm and induced 2-norm of the estimation error for the estimated precision matrices and their differences using the proposed method, \texttt{L1E}, and \texttt{TVGL} (averaged over 10 independent trials).}\label{fig_errors2_small}
% \end{figure*}

\begin{figure}[ht]\centering
	\subfloat[]{%
		\includegraphics[width=6.5cm]{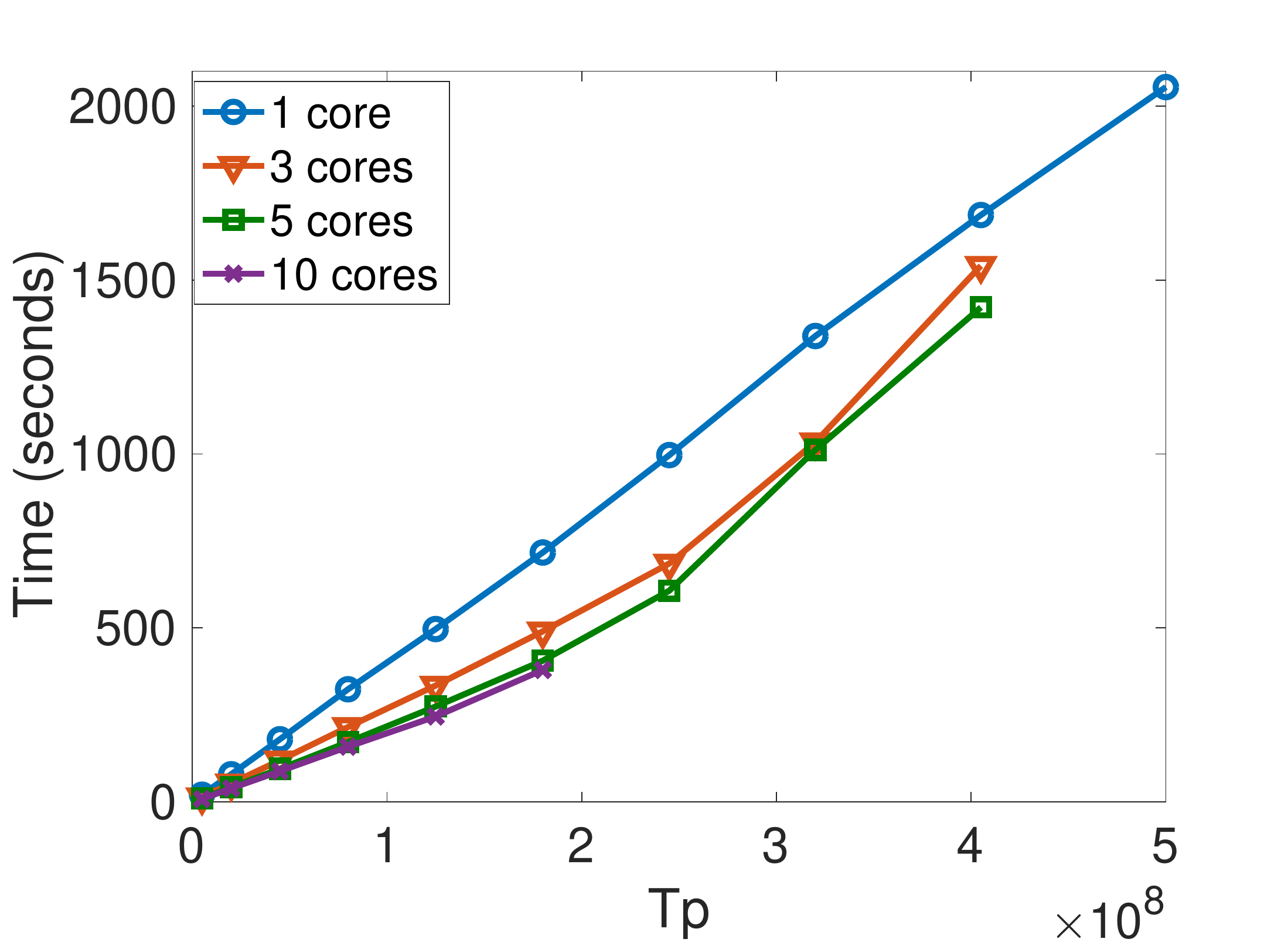}\label{fig:time_multicore}%
	}\hspace{1cm}
	\subfloat[]{%
		\includegraphics[width=6.5cm]{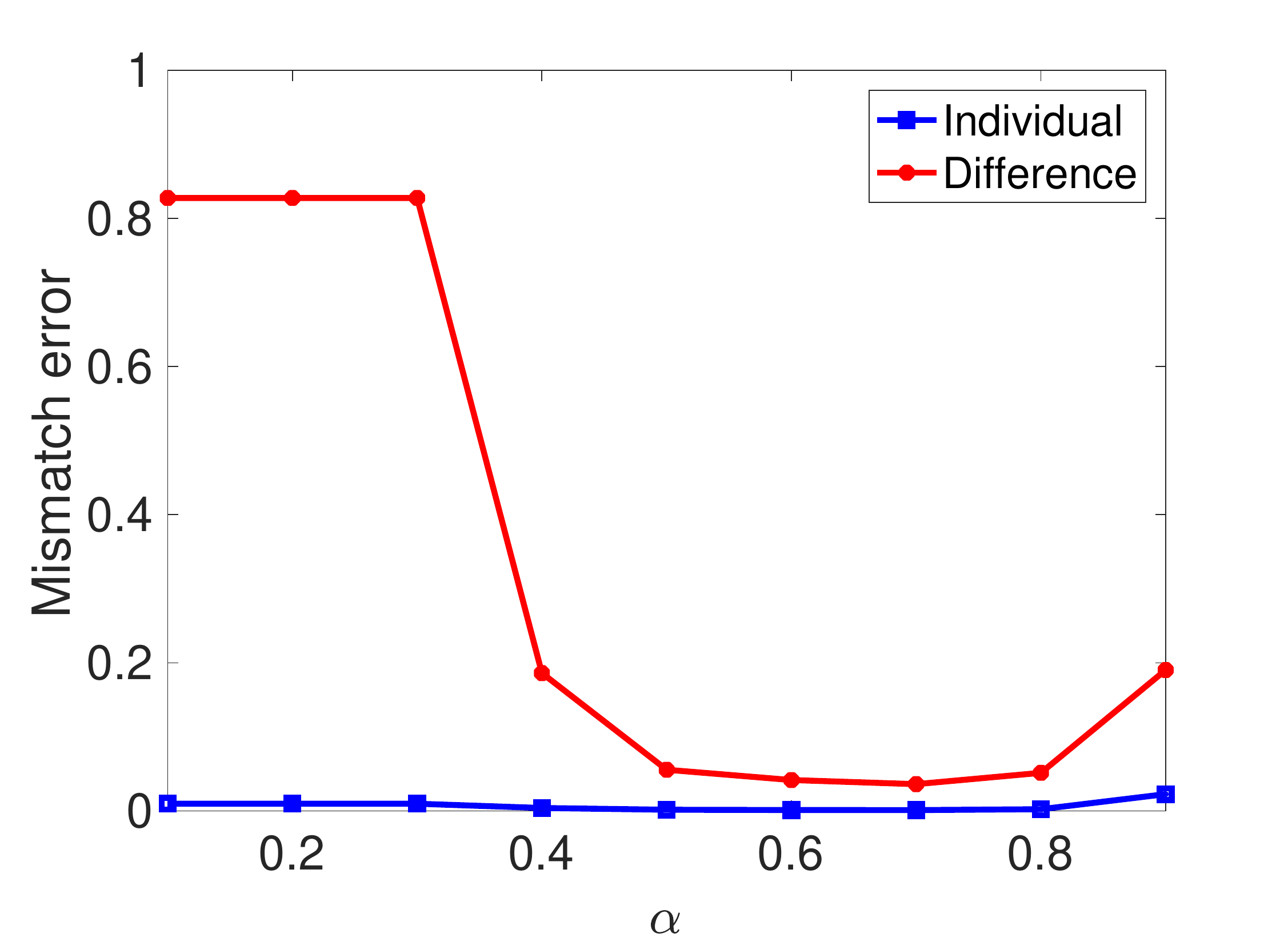}\label{fig:beta}%
	}
	\caption{(a) The runtime of the parallelized algorithm with respect to the number of variables $Tp$, for different number of cores. (b) The normalized mismatch error with respect to the regularization coefficient $\alpha$, for the choices of parameters $d = 4000$, $T=10$, and $N_t = 2000$ for every $t$.}\label{fig:T}
\end{figure}

\begin{figure}[ht]\centering
	\subfloat[]{%
		\includegraphics[width=5.5cm]{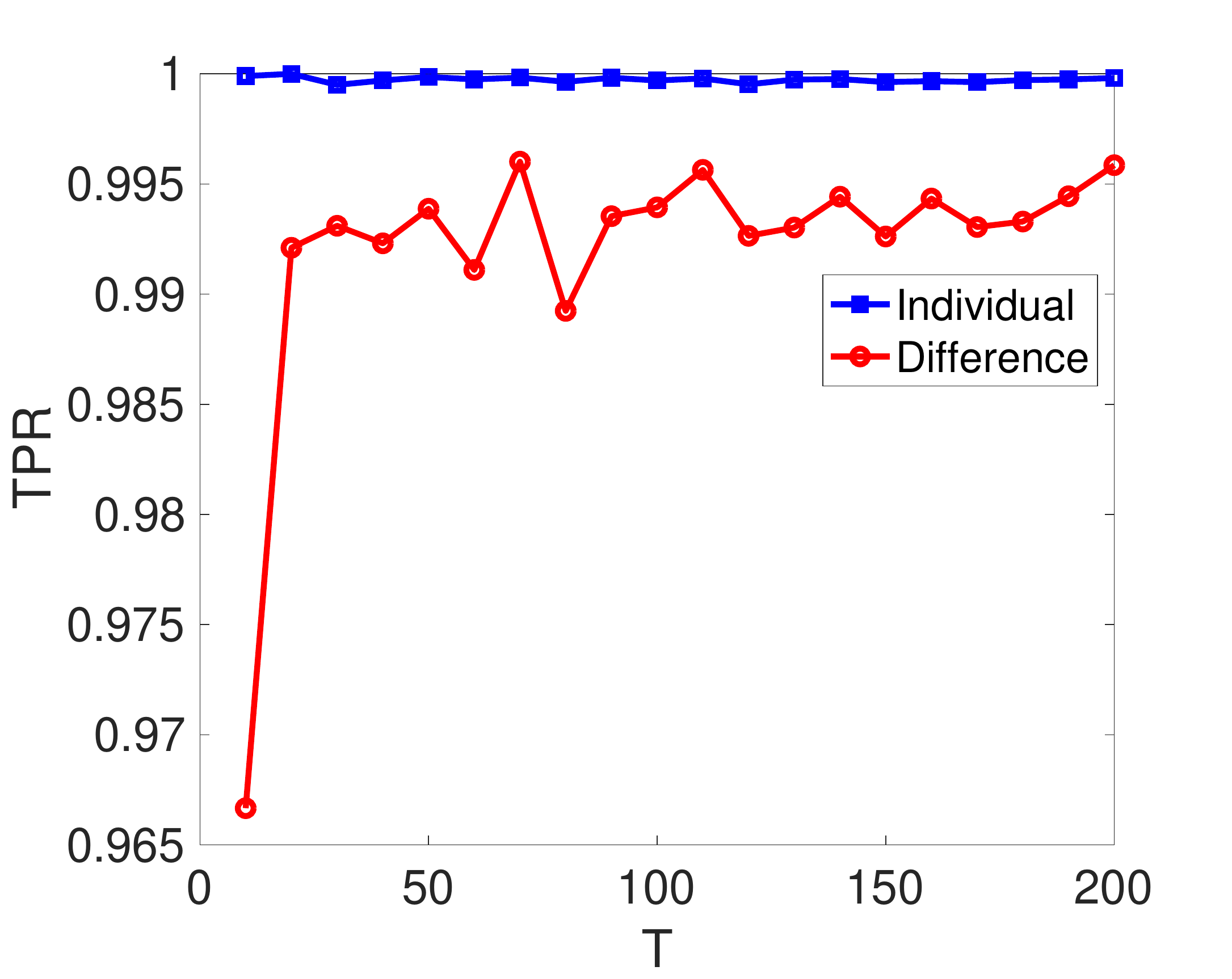}\label{fig:T_TPR}%
	}
	\subfloat[]{%
		\includegraphics[width=5.5cm]{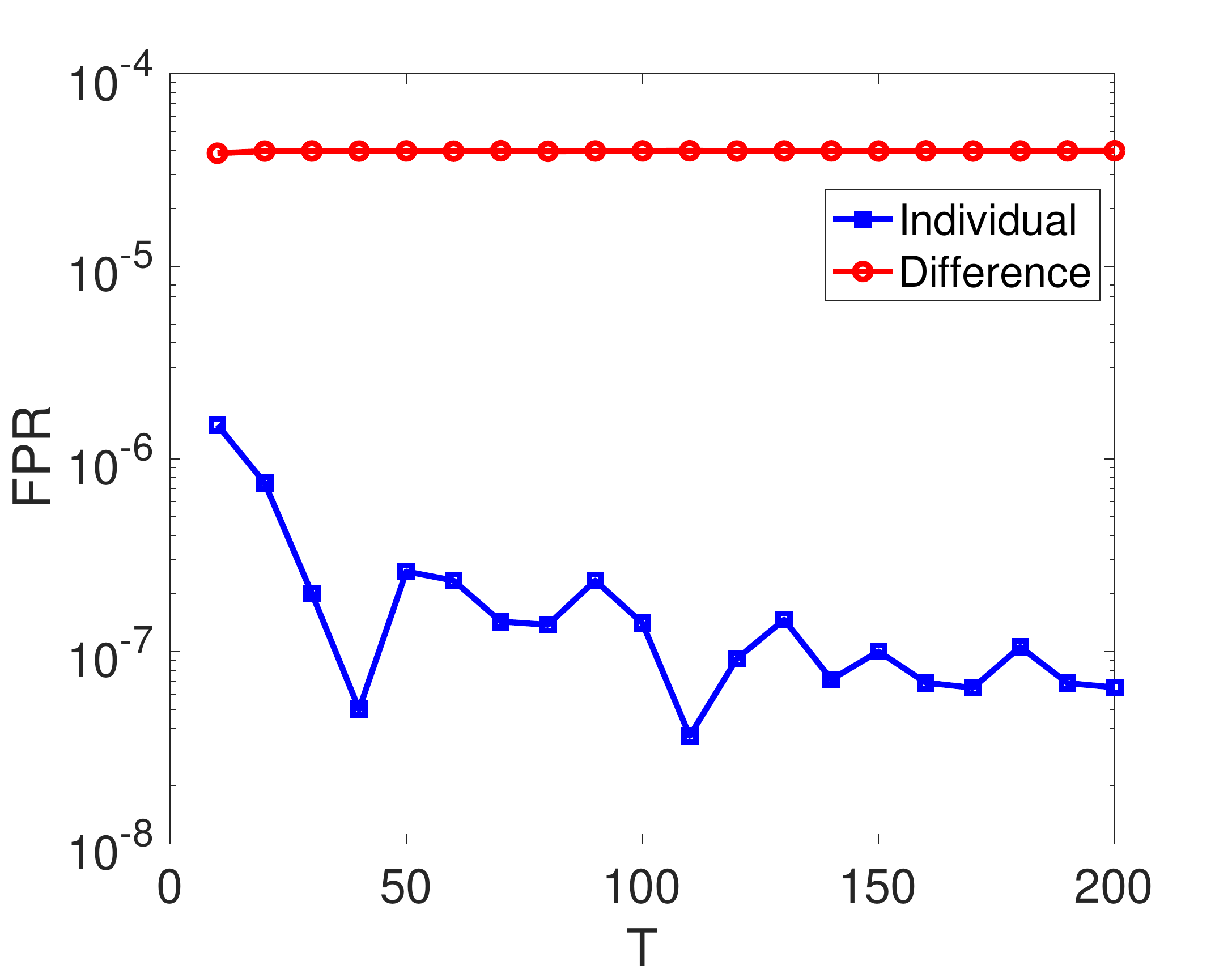} \label{fig:T_FPR}%
	}
	\subfloat[]{%
		\includegraphics[width=5.5cm]{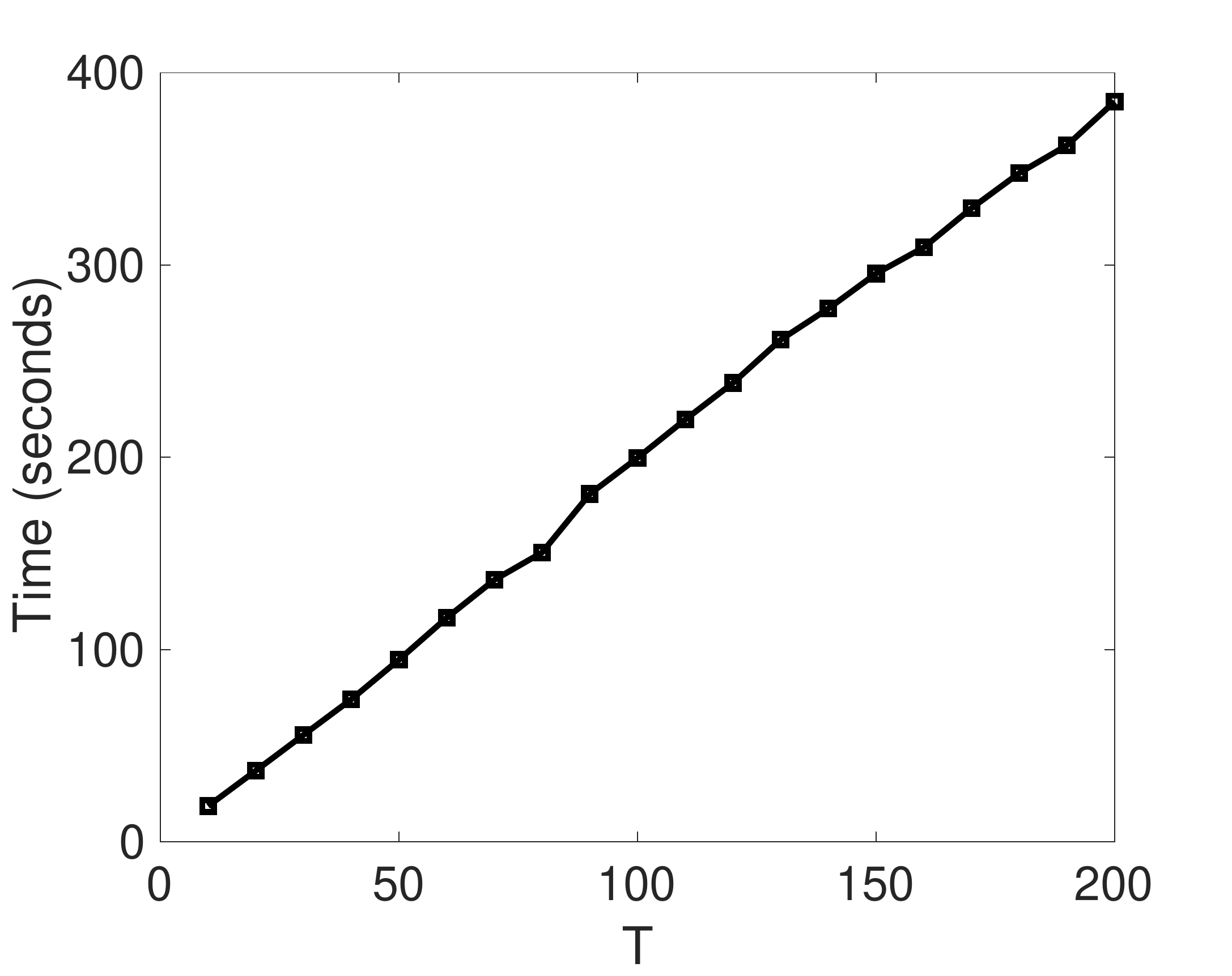} \label{fig:T_time}%
	}
	\caption{\texttt{TPR}, \texttt{FPR}, $\ell_1$-norm estimation error, and the runtime of the proposed method for fixed $d$ and different values of $T$. The number of samples $N_t$ is set to $2d$ for every $t$.}\label{fig:T}
\end{figure}

\subsection{Case Study on Large Datasets}
In this case study, we analyze the performance of the proposed estimator on large datasets, with different values of $d$ and $T$. In particular, we will analyze the runtime of the proposed algorithm and its statistical performance in high dimensional settings, where $N_t<d$ for every $t = 1,2,\dots, T$. Moreover, we will report the improvements in the runtime with parallelization, and analyze the robustness of the estimator for different choices of the regularization coefficient $\alpha$.

Consider the class of synthetically generated sparsely-changing GMRFs with random precision matrices, as explained in Subsection~\ref{subsec:sim_small}. In the first experiment, we fix $T=10$ and change the values of $d$. The number of nonzero elements in the individual precision matrices and their differences are set to $3d$ and $0.04d$, respectively. We evaluate the performance of the proposed method in the high dimensional settings, where $N_t=d/2$ for every $t=0,\dots,T$. The parameters $\lambda_t$ and $\nu_t$ are fine-tuned similar to the previous case study and $\alpha = 0.7$ in all instances. Moreover, define \texttt{TPR} and \texttt{FPR} for the individual parameters and their differences as the \texttt{TP} and \texttt{FP} values, normalized by the total number of nonzero and zero elements in the true precision matrices and their differences, respectively. Clearly, both \texttt{TPR} and \texttt{FPR} are between 0 and 1, with $\texttt{TPR}=1$ and $\texttt{FPR}=0$ corresponding to the perfect recovery of the sparsity patterns. Figure~\ref{fig:p} depicts \texttt{TPR}, \texttt{FPR}, and the $\ell_1$-norm error of the estimated parameters, as well as the runtime of our algorithm for different values of $d$. It can be seen that both \texttt{TPR} and \texttt{FPR} values improve with the dimension for the estimated parameters and their differences. Moreover, the runtime of our algorithm scales almost linearly with $p = d(d+1)/2$, which is in line with the result of Theorem~\ref{thm_runtime}. Using our algorithm, we reliably infer instances of sparsely-changing GMRFs with more than 500 million variables in less than one hour.

As mentioned before, our proposed optimization framework is amenable to parallelization due to its elementwise decomposable nature. Figure~\ref{fig:time_multicore} illustrates the runtime of our parallelized algorithm with respect to the total number of variables (fixed $T$ and varying $p$), for different number of cores. Using 5 cores, the runtime of our algorithm is improved by $40\%$ on average. On the other hand, using 10 cores deteriorates the performance due to the shared memory limitations. Finally, we evaluate the accuracy of the estimated parameters for different choices of the regularization coefficient $\alpha$. In particular, we fix $d = 4000$, $T=10$, and $N_t = 2000$ for every $t$, and depict the normalized mismatch error in the sparsity pattern of the estimated parameters and their differences for $\alpha\in\{0.1,0.2,\dots,1\}$. Based on this figure, it can be concluded that overall performance of the proposed method is not too sensitive to specific choice of the regularization parameter $\alpha$. In particular, it can be seen that the normalized mismatch error remains approximately the same for $\alpha\in[0.5, 0.8]$.

In the next experiment, we set $d=1000$ and $N_t = 2d$, and evaluate the performance of the proposed method for different values of $T \in\{10,20,30,\dots, 200\}$. Figure~\ref{fig:T_TPR} shows \texttt{TPR} for the estimated precision matrices and their differences. It can be seen that \texttt{TPR} for the estimated precision matrices is close to 1 for all values of $T$. Moreover, the \texttt{TPR} for the differences of the estimated precision matrices is at least $0.966$. On the other hand, Figure~\ref{fig:T_FPR} shows that the \texttt{FPR} for the estimated precision matrices is close to zero. Finally, Figure~\ref{fig:T_time} shows that the runtime of the proposed algorithm scales almost linearly with $T$. Together with Figure~\ref{fig:time_p}, this implies that the empirical complexity of the algorithm is linear in both $p$ and $T$.

\begin{figure*}
	\centering
	\includegraphics[width=1\columnwidth]{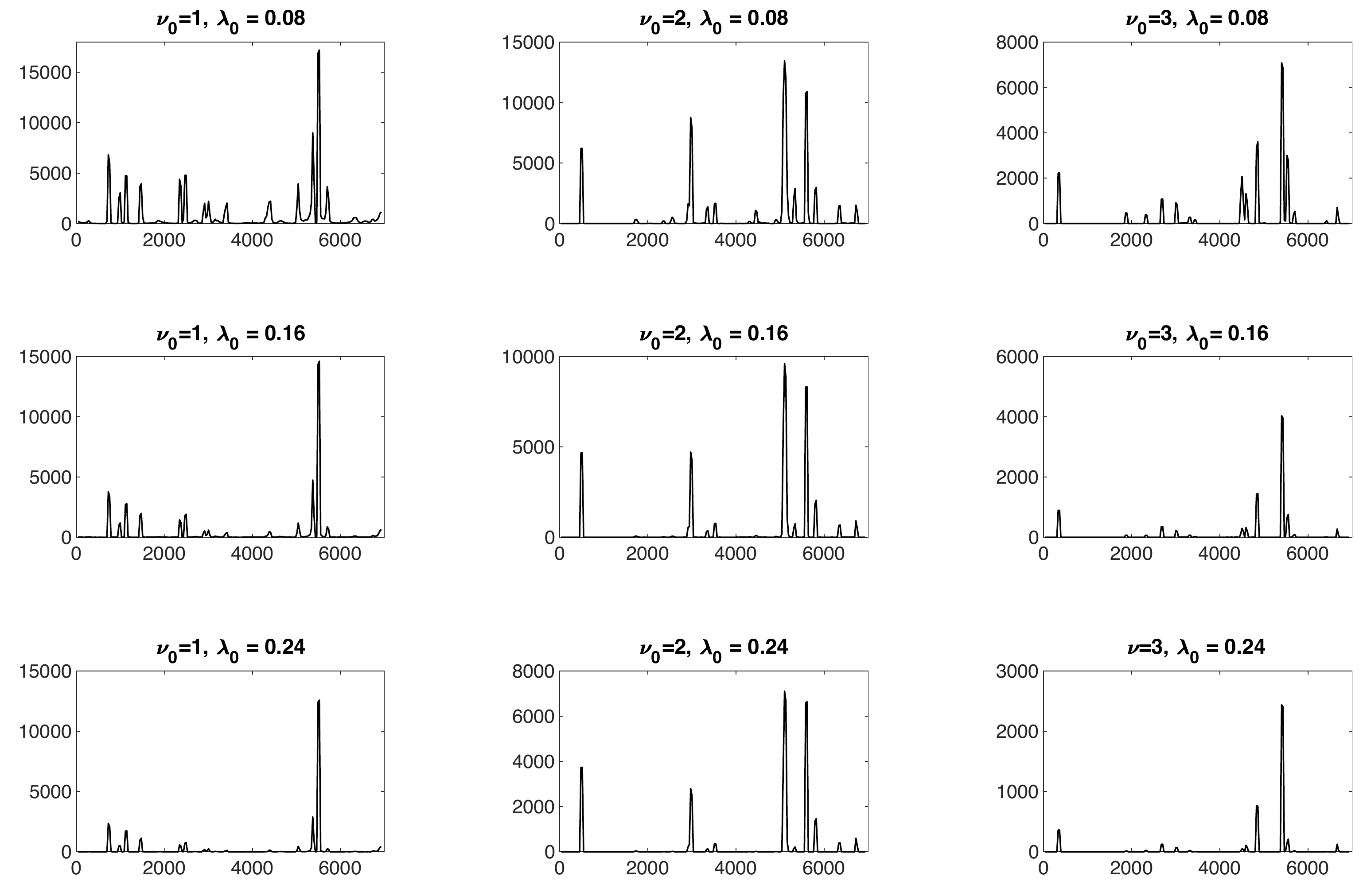}
	\vspace{-2mm}
	\caption{ The number of changes in the estimated stock correlation network, for different choices of $\nu_0$ and $\lambda_0$. The $x$-axis represent the day indexes.}\label{fig_stock}
\end{figure*}

\begin{figure*}[ht]\centering
	\subfloat[]{%
		\includegraphics[width=7.0cm]{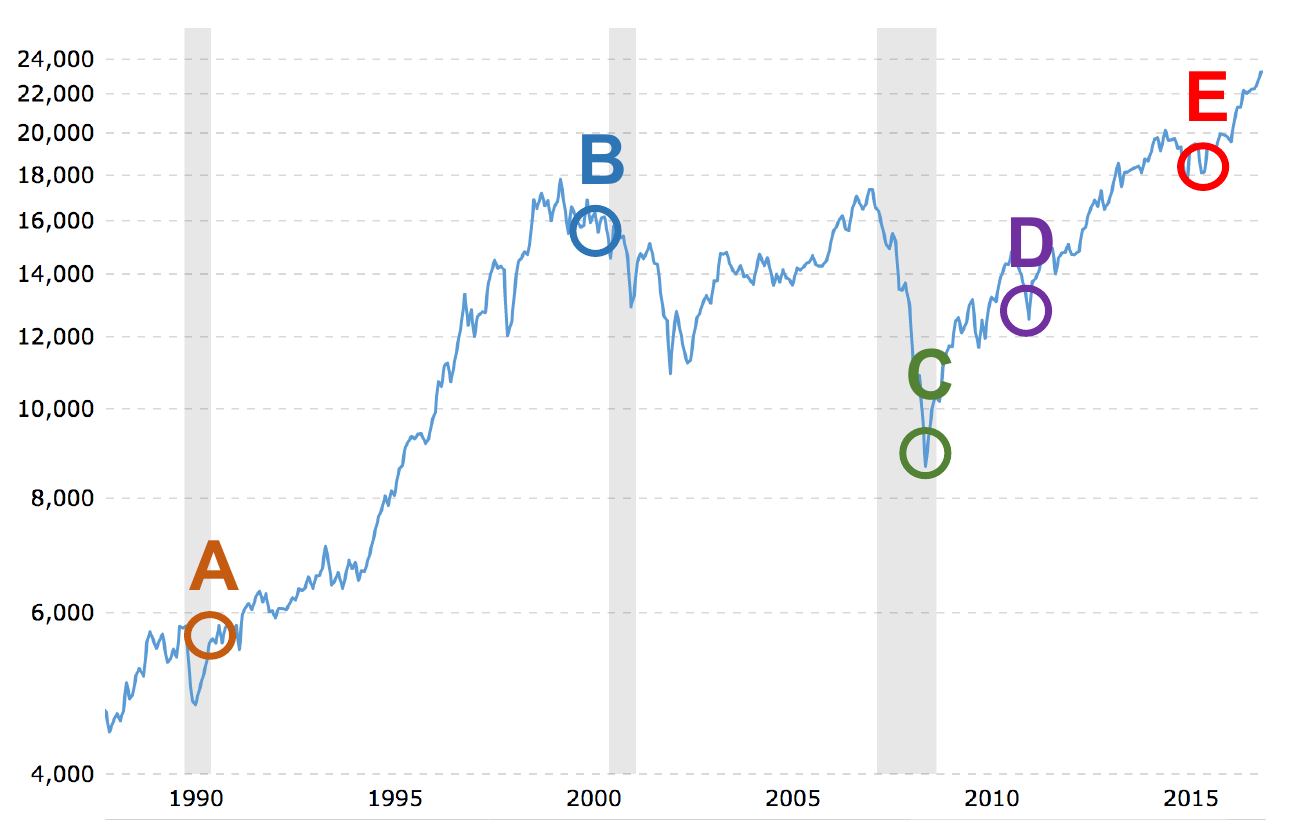}\label{fig:nasdaq}%
	}\hspace{1cm}
	\subfloat[]{%
		\includegraphics[width=7.0cm]{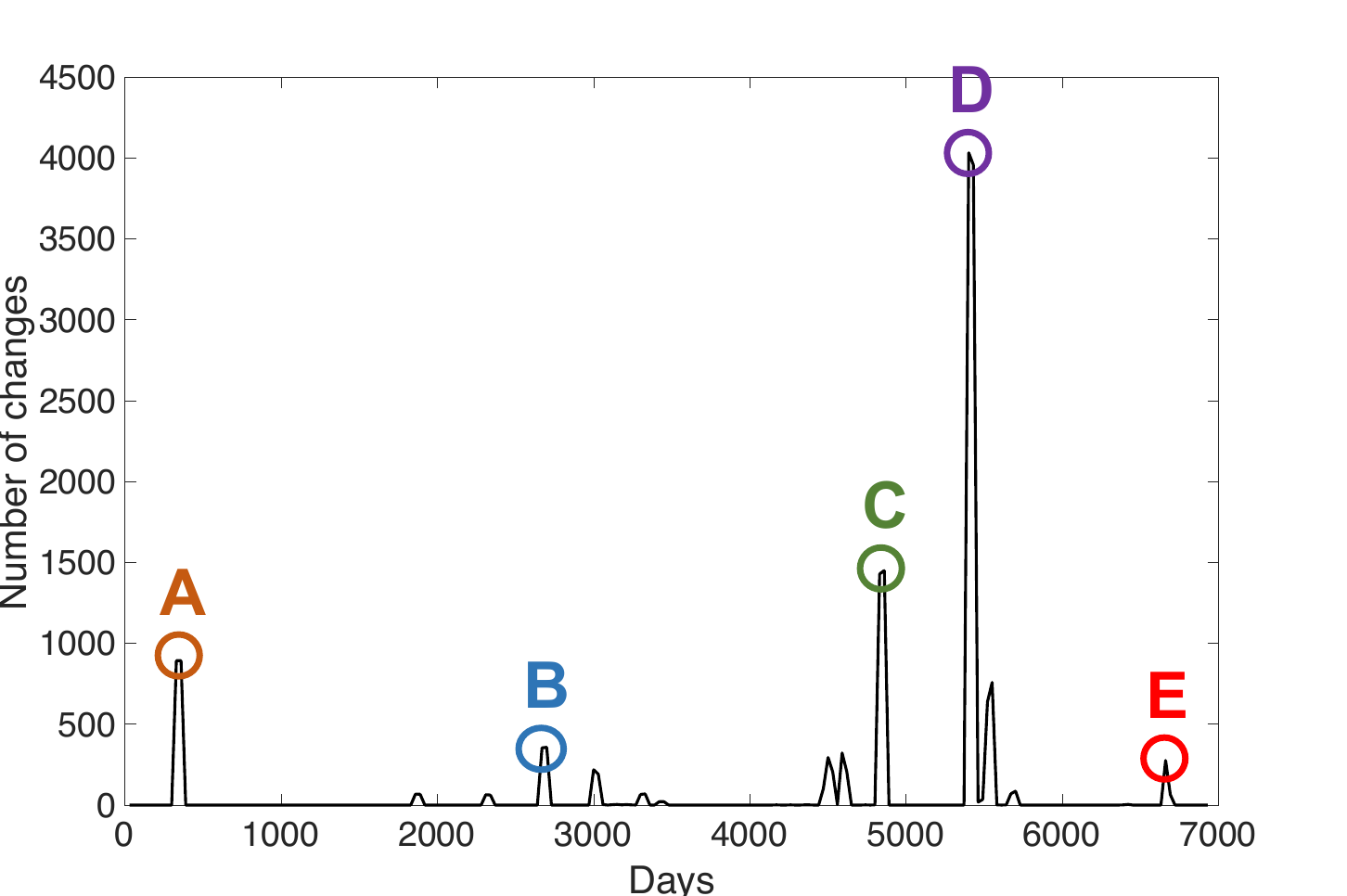}\label{fig:sim_stock}%
	}
	\caption{(a) NASDAQ historical chart from 1988 to 2017~\cite{nasdaq_chart}. (b) The number of changes in the estimated correlation network for $\nu_0=3$ and $\lambda_0=0.16$.}\label{fig:stock_oneshot}
\end{figure*}

\subsection{Case Study on Stock Market}
Finally, we illustrate the performance of our algorithm for the inference of stock correlation network. 
% \notesf{Explain the problem and the data. E.g., what are the applications of this stock correlation network?}
We consider the daily stock prices for 214 securities from $1990/01/04$ to $2017/08/10$, with the total number of 6990 days ($d=214$ and $T=6990$). Due to the continuously changing nature of the stock correlation network, we will use the kernel averaging approach that was introduced in Subsection~\ref{subsec:sparse_GMRF} to estimate the underlying time-varying network. In particular, we consider a Gaussian kernel with bandwidth $h = 0.3T^{-1/3}$ to obtain the sequence of weighted sample covariance matrices. Using the constructed sample covariance matrices, we estimate the sparsely-changing precision matrix $\Theta(t/T)$ at discrete times $t \in\{30,60,90,\dots, 6990\}$. Moreover, we set $\alpha=0.9$, $\lambda_t = \lambda_0\sqrt{\frac{\log(d)}{Th}}$, and $\nu_t = \nu_0\sqrt{\frac{\log(d)}{Th}}$, for some constants $\lambda_0$ and $\nu_0$ to be defined later. Note that these choices of the parameters are consistent with the assumptions of Theorem~\ref{cor_GMRF_ker}.

Figure~\ref{fig_stock} shows the number of changes in the sparsity pattern of the estimated correlation network, for different choices of the parameters $\nu_0$ and $\lambda_0$. A drastic change in the correlation network signals a \textit{spike} in the stock market, which may reflect the market's response to unexpected global events. It can be seen that, for small values of $\nu_0$ and $\lambda_0$, the estimated network can detect both small and large spikes. As the values of $\nu_0$ and $\lambda_0$ increase, the small spikes gradually dimish, and the estimated network only ``picks up'' major changes in the network. Nonetheless, there is a recurring pattern of spikes in these plots that is almost insensitive to different values of $\nu_0$ and $\lambda_0$. A closer look at this recurring pattern sheds light on the behavior of the market. Figure~\ref{fig:stock_oneshot} shows the number of changes in the estimated network, for the choices of $\nu_0 = 3$ and $\lambda_0 = 0.16$, together with the historical chart of National Association of Securities Dealers Automated Quotations (NASDAQ)~\cite{nasdaq}. It can be seen that the major spikes in the estimated network can be attributed to the historical stock market \textit{crashes}. For instance, the spikes \texttt{A}, \texttt{B}, and \texttt{C} respectively correspond to the {``early 1990s recession''}, {``dot-com bubble''}, and {``global financial crisis''}; see~\cite{aliber2017manias} for more details. Interestingly, the estimated network can also detect other historical (but less severe) downturns in 2011 (point \texttt{D}) and 2016 (point \texttt{E}). 

\section{Conclusion}
In this paper, we study the inference of sparsely-changing Markov random fields (MRF), where the goal is to estimate a sequence of time-varying Markov graphs from a limited number of samples, while promoting sparsity of the individual Markov graphs and their differences.We introduce a new class of constrained optimization problems which, unlike existing alternatives, is based on the exact $\ell_0$ regularization. We provide an efficient algorithm that can solve the proposed optimization problem to optimality in polynomial time. The proposed estimator benefits from strong statistical guarantees. As a special case, we show that, using our method, the sparsely-changing Gaussian MRFs can be reliably estimated with as few as one sample per time. Finally, we illustrate the performance of the proposed method in different case studies on synthetic data and financial markets: problems with more than 500 million unknown parameters can be solved in less than one hour.

\bibliography{reference}

\begin{thebibliography}{10}

\bibitem{nasdaq}
About nasdaq.
\newblock \url{https://www.nasdaq.com/about}, 2020.

\bibitem{nasdaq_chart}
Nasdaq composite - 45 year historical chart.
\newblock \url{https://www.macrotrends.net/1320/nasdaq-historical-chart}, 2020.

\bibitem{ahuja2004cut}
R.~K. Ahuja, D.~S. Hochbaum, and J.~B. Orlin.
\newblock A cut-based algorithm for the nonlinear dual of the minimum cost
  network flow problem.
\newblock {\em Algorithmica}, 39(3):189--208, 2004.

\bibitem{ahuja1988network}
R.~K. Ahuja, T.~L. Magnanti, and J.~B. Orlin.
\newblock Network flows.
\newblock 1988.

\bibitem{aliber2017manias}
R.~Z. Aliber and C.~P. Kindleberger.
\newblock {\em Manias, panics, and crashes: A history of financial crises}.
\newblock Springer, 2017.

\bibitem{atamturk2018strong}
A.~Atamt{\"u}rk and A.~G{\'o}mez.
\newblock Strong formulations for quadratic optimization with m-matrices and
  indicator variables.
\newblock {\em Mathematical Programming}, 170(1):141--176, 2018.

\bibitem{atamturk2018sparse}
A.~Atamt\"urk, A.~G\'omez, and S.~Han.
\newblock Sparse and smooth signal estimation: Convexification of
  $\ell_0$-formulations.
\newblock {\em arXiv preprint arXiv:1811.02655}, 2018.

\bibitem{bach2019submodular}
F.~Bach.
\newblock Submodular functions: from discrete to continuous domains.
\newblock {\em Mathematical Programming}, 175(1-2):419--459, 2019.

\bibitem{benzi2007decay}
M.~Benzi and N.~Razouk.
\newblock Decay bounds and o (n) algorithms for approximating functions of
  sparse matrices.
\newblock {\em Electron. Trans. Numer. Anal}, 28:16--39, 2007.

\bibitem{benzi2015decay}
M.~Benzi and V.~Simoncini.
\newblock Decay bounds for functions of hermitian matrices with banded or
  kronecker structure.
\newblock {\em SIAM Journal on Matrix Analysis and Applications},
  36(3):1263--1282, 2015.

\bibitem{cai2018capturing}
B.~Cai, G.~Zhang, A.~Zhang, J.~M. Stephen, T.~W. Wilson, V.~D. Calhoun, and
  Y.-P. Wang.
\newblock Capturing dynamic connectivity from resting state fmri using
  time-varying graphical lasso.
\newblock {\em IEEE Transactions on Biomedical Engineering}, 66(7):1852--1862,
  2018.

\bibitem{demko1984decay}
S.~Demko, W.~F. Moss, and P.~W. Smith.
\newblock Decay rates for inverses of band matrices.
\newblock {\em Mathematics of computation}, 43(168):491--499, 1984.

\bibitem{fattahi2019graphical}
S.~Fattahi and S.~Sojoudi.
\newblock Graphical lasso and thresholding: Equivalence and closed-form
  solutions.
\newblock {\em The Journal of Machine Learning Research}, 20(1):364--407, 2019.

\bibitem{greenewald2017time}
K.~Greenewald, S.~Park, S.~Zhou, and A.~Giessing.
\newblock Time-dependent spatially varying graphical models, with application
  to brain fmri data analysis.
\newblock In {\em Advances in Neural Information Processing Systems}, pages
  5832--5840, 2017.

\bibitem{hallac2017network}
D.~Hallac, Y.~Park, S.~Boyd, and J.~Leskovec.
\newblock Network inference via the time-varying graphical lasso.
\newblock In {\em Proceedings of the 23rd ACM SIGKDD International Conference
  on Knowledge Discovery and Data Mining}, pages 205--213, 2017.

\bibitem{hochbaum2001efficient}
D.~S. Hochbaum.
\newblock An efficient algorithm for image segmentation, markov random fields
  and related problems.
\newblock {\em Journal of the ACM (JACM)}, 48(4):686--701, 2001.

\bibitem{huang2010learning}
S.~Huang, J.~Li, L.~Sun, J.~Ye, A.~Fleisher, T.~Wu, K.~Chen, and E.~Reiman.
\newblock Learning brain connectivity of alzheimer's disease by sparse inverse
  covariance estimation.
\newblock {\em NeuroImage}, 50(3):935--949, 2010.

\bibitem{jewell2017exact}
S.~Jewell and D.~Witten.
\newblock Exact spike train inference via $\ell_0$ optimization.
\newblock {\em arXiv preprint arXiv:1703.08644}, 2017.

\bibitem{kalman1960new}
R.~E. Kalman.
\newblock A new approach to linear filtering and prediction problems.
\newblock 1960.

\bibitem{kershaw1970inequalities}
D.~Kershaw.
\newblock Inequalities on the elements of the inverse of a certain tridiagonal
  matrix.
\newblock {\em Mathematics of computation}, pages 155--158, 1970.

\bibitem{kim2015highly}
J.~Kim and W.~Pan.
\newblock Highly adaptive tests for group differences in brain functional
  connectivity.
\newblock {\em NeuroImage: Clinical}, 9:625--639, 2015.

\bibitem{lam2009sparsistency}
C.~Lam and J.~Fan.
\newblock Sparsistency and rates of convergence in large covariance matrix
  estimation.
\newblock {\em Annals of statistics}, 37(6B):4254, 2009.

\bibitem{lelearning}
B.~Le~Bars, P.~Humbert, A.~Kalogeratos, and N.~Vayatis.
\newblock Learning the piece-wise constant graph structure of a varying ising
  model.
\newblock {\em International Conference on Machine Learning}, 2020.

\bibitem{liu2017learning}
S.~Liu, K.~Fukumizu, and T.~Suzuki.
\newblock Learning sparse structural changes in high-dimensional markov
  networks.
\newblock {\em Behaviormetrika}, 44(1):265--286, 2017.

\bibitem{liu2013time}
X.~Liu and J.~H. Duyn.
\newblock Time-varying functional network information extracted from brief
  instances of spontaneous brain activity.
\newblock {\em Proceedings of the National Academy of Sciences},
  110(11):4392--4397, 2013.

\bibitem{miller2002subset}
A.~Miller.
\newblock {\em {Subset Selection in Regression}}.
\newblock CRC Press, 2002.

\bibitem{narayan2015two}
M.~Narayan, G.~I. Allen, and S.~Tomson.
\newblock Two sample inference for populations of graphical models with
  applications to functional connectivity.
\newblock {\em arXiv preprint arXiv:1502.03853}, 2015.

\bibitem{ravikumar2010high}
P.~Ravikumar, M.~J. Wainwright, and J.~D. Lafferty.
\newblock High-dimensional ising model selection using $\ell_1$-regularized
  logistic regression.
\newblock {\em The Annals of Statistics}, 38(3):1287--1319, 2010.

\bibitem{ravikumar2011high}
P.~Ravikumar, M.~J. Wainwright, G.~Raskutti, and B.~Yu.
\newblock High-dimensional covariance estimation by minimizing
  $\ell_1$-penalized log-determinant divergence.
\newblock {\em Electronic Journal of Statistics}, 5:935--980, 2011.

\bibitem{rinaldo2009properties}
A.~Rinaldo et~al.
\newblock Properties and refinements of the fused lasso.
\newblock {\em The Annals of Statistics}, 37:2922--2952, 2009.

\bibitem{rothman2008sparse}
A.~J. Rothman, P.~J. Bickel, E.~Levina, J.~Zhu, et~al.
\newblock Sparse permutation invariant covariance estimation.
\newblock {\em Electronic Journal of Statistics}, 2:494--515, 2008.

\bibitem{rubinov2010complex}
M.~Rubinov and O.~Sporns.
\newblock Complex network measures of brain connectivity: uses and
  interpretations.
\newblock {\em Neuroimage}, 52(3):1059--1069, 2010.

\bibitem{rudin1992nonlinear}
L.~I. Rudin, S.~Osher, and E.~Fatemi.
\newblock Nonlinear total variation based noise removal algorithms.
\newblock {\em Physica D: Nonlinear Phenomena}, 60:259--268, 1992.

\bibitem{shellenbarger1966estimation}
J.~C. Shellenbarger.
\newblock Estimation of covariance parameters for an adaptive kalman filter.
\newblock 1966.

\bibitem{tibshirani2005sparsity}
R.~Tibshirani, M.~Saunders, S.~Rosset, J.~Zhu, and K.~Knight.
\newblock Sparsity and smoothness via the fused lasso.
\newblock {\em Journal of the Royal Statistical Society: Series B (Statistical
  Methodology)}, 67:91--108, 2005.

\bibitem{vogel1996iterative}
C.~R. Vogel and M.~E. Oman.
\newblock Iterative methods for total variation denoising.
\newblock {\em SIAM Journal on Scientific Computing}, 17:227--238, 1996.

\bibitem{wainwright2003tree}
M.~J. Wainwright, T.~S. Jaakkola, and A.~S. Willsky.
\newblock Tree-reweighted belief propagation algorithms and approximate ml
  estimation by pseudo-moment matching.
\newblock In {\em AISTATS}, volume~3, page~3, 2003.

\bibitem{wainwright2008graphical}
M.~J. Wainwright and M.~I. Jordan.
\newblock {\em Graphical models, exponential families, and variational
  inference}.
\newblock Now Publishers Inc, 2008.

\bibitem{wang2018fast}
B.~Wang, Y.~Qi, et~al.
\newblock Fast and scalable learning of sparse changes in high-dimensional
  gaussian graphical model structure.
\newblock In {\em International Conference on Artificial Intelligence and
  Statistics}, pages 1691--1700. PMLR, 2018.

\bibitem{weiss2000correctness}
Y.~Weiss and W.~T. Freeman.
\newblock Correctness of belief propagation in gaussian graphical models of
  arbitrary topology.
\newblock In {\em Advances in Neural Information Processing Systems}, pages
  673--679, 2000.

\bibitem{yang2014elementary}
E.~Yang, A.~C. Lozano, and P.~K. Ravikumar.
\newblock Elementary estimators for graphical models.
\newblock In {\em Advances in neural information processing systems}, pages
  2159--2167, 2014.

\bibitem{zhao2014direct}
S.~D. Zhao, T.~T. Cai, and H.~Li.
\newblock Direct estimation of differential networks.
\newblock {\em Biometrika}, 101(2):253--268, 2014.

\bibitem{zhou2010time}
S.~Zhou, J.~Lafferty, and L.~Wasserman.
\newblock Time varying undirected graphs.
\newblock {\em Machine Learning}, 80(2-3):295--319, 2010.

\end{thebibliography}

\clearpage

% As a future work, we will study the statistical performance of our proposed estimation method for other classes of sparsely-changing MRFs, such as discrete MRFs. Moreover, 

\appendix
\section{Proofs}
\subsection{Proof of Theorem~\ref{thm_deterministic}} Due to the feasibility of $\{\widehat{\theta}_t\}_{t=0}^T$, one can write $\|\widehat{\theta}_t-\widetilde{F}^*(\widehat{\mu}_t)\|_\infty\leq \lambda_t$. Combined with the first assumption of the theorem, this implies that
\begin{align}\label{eq_error}
\left\|\widehat{\theta}_t-\theta_t^*\right\|_{\infty} &= \left\|\widehat{\theta}_t-\widetilde{F}^*(\widehat{\mu}_t)+\widetilde{F}^*(\widehat{\mu}_t)-\theta_t^*\right\|_{\infty} \nonumber\\
&\leq \left\|\widehat{\theta}_t-\widetilde{F}^*(\widehat{\mu}_t)\right\|_{\infty} + \left\|{\theta}^*_t-\widetilde{F}^*(\widehat{\mu}_t)\right\|_{\infty}\nonumber\\
&< 2\lambda_t,
\end{align}
thereby establishing the element-wise estimation error bound.  We proceed to show the sparsistency of the estimated parameters. First, suppose that $\theta_{t;i}^*\not=0$ for some time $t$ and index $i$. One can write
\begin{align}
\left|\widehat \theta_{t;i}\right| &= \left|\widehat \theta_{t;i}-\theta_{t;i}^*+\theta_{t;i}^*\right|\nonumber\\
&\geq \left|\theta_{t;i}^*\right| - \left|\widehat \theta_{t;i}-\theta_{t;i}^*\right|\nonumber\\
&>0
\end{align}
where the last inequality is due to the second assumption of the theorem and~\eqref{eq_error}. This implies that $\mathrm{supp}(\theta^*_t)\subseteq\mathrm{supp}(\widehat{\theta}_t)$. Similarly, suppose that $\theta^*_{t;i}-\theta_{t-1;i}^*\not=0$ for some time $t>0$ and index $i$. One can write
\begin{align}
\left|\widehat \theta_{t;i}-\widehat \theta_{t-1;i}\right| &= \left|\widehat \theta_{t;i}-\theta_{t;i}^*+\theta_{t;i}^*-\theta_{t-1;i}^*+\theta_{t-1;i}^*-\widehat \theta_{t-1;i}\right|\nonumber\\
&\geq \left|\theta_{t;i}^*-\theta_{t-1;i}^*\right|-\left|\widehat \theta_{t;i}-\theta_{t;i}^*\right|-\left|\widehat \theta_{t-1;i}-\theta_{t-1;i}^*\right|\nonumber\\
&>0
\end{align} 
where the last inequality is due to the third assumption of the theorem and~\eqref{eq_error}. This implies that $\mathrm{supp}(\theta_{t}^*-\theta_{t-1}^*)\subseteq \mathrm{supp}(\widehat\theta_{t}-\widehat\theta_{t-1})$. Finally, due to the optimality of $\{\widehat{\theta}_t\}_{t=0}^T$ and feasibility of $\{{\theta}^*_t\}_{t=0}^T$, one can write
\begin{align}
&(1-\alpha)\sum_{t=0}^T\|\widehat\theta_t\|_0+\alpha\sum_{t=1}^T\|\widehat\theta_t-\widehat\theta_{t-1}\|_0\leq (1-\alpha)\sum_{t=0}^T\|\theta_t^*\|_0+\alpha\sum_{t=1}^T\|\theta_t^*-\theta_{t-1}^*\|_0\nonumber\\
\implies&(1-\alpha)\sum_{t=0}^T\left(\sum_{i\in[p]\backslash\mathcal{S}_t}|\widehat\theta_{t;i}|_0+\sum_{i\in\mathcal{S}_t}|\widehat\theta_{t;i}|_0\right)+\alpha\sum_{t=1}^T\left(\sum_{i\in[p]\backslash\mathcal{D}_t}|\widehat\theta_{t;i}-\widehat\theta_{t-1;i}|_0+\sum_{i\in\mathcal{D}_t}|\widehat\theta_{t;i}-\widehat\theta_{t-1;i}|_0\right)\nonumber\\
& \hspace{6cm}\leq (1-\alpha)\sum_{t=0}^T\sum_{i\in\mathcal{S}_t}|\theta_{t;i}^*|_0+\alpha\sum_{t=1}^T\sum_{i\in\mathcal{D}_t}|\theta_{t;i}^*-\theta_{t-1;i}^*|_0\nonumber\\
\implies & (1-\alpha)\sum_{t=0}^T\sum_{i\in[p]\backslash\mathcal{S}_t}|\widehat\theta_{t;i}|_0+\alpha\sum_{t=1}^T\sum_{i\in[p]\backslash\mathcal{D}_t}|\widehat\theta_{t;i}-\widehat\theta_{t-1;i}|_0\leq 0
\end{align}
where the last inequality follows from $\mathrm{supp}(\theta^*_t)\subseteq\mathrm{supp}(\widehat{\theta}_t)$ and $\mathrm{supp}(\theta_{t}^*-\theta_{t-1}^*)\subseteq \mathrm{supp}(\widehat\theta_{t}-\widehat\theta_{t-1})$, which implies $\sum_{i\in\mathcal{S}_t}|\widehat\theta_{t;i}|_0-|\theta_{t;i}^*|_0\geq 0$ and $\sum_{i\in\mathcal{D}_t}|\widehat\theta_{t;i}^*-\widehat\theta_{t-1;i}|_0 - |\theta_{t;i}^*-\theta_{t-1;i}^*|_0\geq 0$ for every $t$. Due to $0<\alpha<1$, the above inequality implies that $\widehat\theta_{t;i} = 0$ for every $t$ and $i\in[p]\backslash\mathcal{S}_t$, and $\widehat\theta_{t;i}-\widehat\theta_{t-1;i} = 0$ for every $t>0$ and $i\in[p]\backslash\mathcal{D}_t$. This implies that $\mathrm{supp}(\widehat{\theta}_t)\subseteq \mathrm{supp}(\theta^*_t)$ and $\mathrm{supp}(\widehat\theta_{t}-\widehat\theta_{t-1})\subseteq\mathrm{supp}(\theta_{t}^*-\theta_{t-1}^*)$. Finally, since $\mathrm{supp}(\widehat{\theta}_t)\subseteq\mathrm{supp}(\theta^*_t)$, we have $|\mathrm{supp}(\widehat{\theta}_t-\theta_t^*)| = |\mathcal{S}_t|$. This, together with~\eqref{eq_error} implies that $\|\widehat{\theta}_t-\theta_t^*\|_2\leq \sqrt{|\mathcal{S}_t|}\|\widehat{\theta}_t-\theta_t^*\|_\infty\leq 2\sqrt{|\mathcal{S}_t|}\lambda_t$, thereby completing the proof.$\hfill\square$\vspace{2mm}

\subsection{Proof of Theorem~\ref{cor_GMRF}}\label{app_cor_GMRF}
The proof is inspired by Corollary 1 in~\cite{yang2014elementary}. First, we present the following key lemmas.
\begin{lemma}[Lemma 2 of~\cite{yang2014elementary} and Lemma 1 of~\cite{ravikumar2011high}]\label{l_Sigma}
	Suppose that $X^{(i)}\sim\mathcal{N}(0,\Sigma)$ for $i = 1,\dots,N$, and $\widehat{\Sigma} = \frac{1}{N}\sum_{i=1}^NX^{(i)}{X^{(i)}}^\top$. Then, we have
	\begin{align}
	\left\|\widehat{\Sigma}-\Sigma\right\|_{\linf}\leq 8\left(\max_{i}\Sigma_{ii}\right)\sqrt{\frac{\tau\log d}{N}}
	\end{align}
	with probability of at least $1-4d^{-\tau+2}$ for any $\tau>2$, provided that $N\geq 40\left(\max_{i}\Sigma_{ii}\right)$.
\end{lemma}
\begin{lemma}[Lemma 1 of~\cite{yang2014elementary}; modified] \label{l_Thresh}
	Under the conditions of Lemma~\ref{l_Sigma}, we have
	\begin{align}
	\left\|\texttt{ST}_{\nu}(\widehat{\Sigma})-\Sigma\right\|_\infty\leq 5\nu^{1-q}s(q,d)+24\nu^{-q}s(q,d)\left(\max_{i}\Sigma_{ii}\right)\sqrt{\frac{\tau\log d}{N}}
	\end{align}
	with probability of at least $1-4d^{-\tau+2}$ for any $\tau>2$, provided that $N\geq 40\left(\max_{i}\Sigma_{ii}\right)$.
\end{lemma}
Based on the above lemmas, we proceed to present the proof of Corollary~\ref{cor_GMRF}.\vspace{2mm}\\
\noindent{\it Proof of Corollary~\ref{cor_GMRF}.} It suffices to show that the conditions of Theorem~\ref{thm_deterministic} are satisfied with the proposed choices of $\lambda_t$ and $\nu_t$. It is easy to see that
\begin{align}\label{eq_diff}
\left\|\Theta_t-[\texttt{ST}_{\nu_t}(\widehat{\Sigma}_t)]^{-1}\right\|_{\linf}&= \left\|[\texttt{ST}_{\nu_t}(\widehat{\Sigma}_t)]^{-1}(\texttt{ST}_{\nu_t}(\widehat{\Sigma}_t)\Theta_t-I)\right\|_{\linf}\nonumber\\
&\leq \left\|[\texttt{ST}_{\nu_t}(\widehat{\Sigma}_t)]^{-1}\right\|_{\infty}\left\|\Theta_t\right\|_{\infty}\left\|\texttt{ST}_{\nu_t}(\widehat{\Sigma}_t)-\Sigma_t\right\|_{\linf}
\end{align}
We provide separate bounds for different terms of the above inequality. Due to Assumption~\ref{assum1}, one can write $\left\|\Theta_t\right\|_{\infty}\leq \kappa_1$. Moreover, due to Lemma~\ref{l_Sigma}, the following inequality holds with probability of at least $1-4d^{-\tau+2}$ for any $\tau>2$
\begin{align}\label{eq_ST}
\left\|\texttt{ST}_{\nu_t}(\widehat{\Sigma}_t)-\Sigma_t\right\|_{\linf}&\leq \left\|\texttt{ST}_{\nu_t}(\widehat{\Sigma}_t)-\widehat{\Sigma}_t\right\|_{\linf}+\left\|\widehat{\Sigma}_t-\Sigma_t\right\|_{\linf}\nonumber\\
&\leq \nu_t+8\kappa_3\sqrt{\frac{\tau\log d}{N_t}}\nonumber\\
&=16\kappa_3\sqrt{\frac{\tau\log d}{N_t}}
\end{align}
provided that $N_t\geq 40\kappa_3$ and $\nu_t = 8\kappa_3\sqrt{\frac{\tau\log d}{N_t}}$. Finally, for any vector $w$, one can write
\begin{align}\label{eq_ST}
\|\texttt{ST}_{\nu_t}(\widehat{\Sigma}_t)w\|_\infty&\geq \|\Sigma_t w\|_\infty-\left\|(\texttt{ST}_{\nu_t}(\widehat{\Sigma}_t)-\Sigma_t)w\right\|_\infty\nonumber\\
&\geq\left(\kappa_2-\left\|\texttt{ST}_{\nu_t}(\widehat{\Sigma}_t)-\Sigma_t\right\|_\infty\right)\|w\|_\infty
\end{align}
On the other hand, the aforementioned choice of $\nu_t$ and Lemma~\ref{l_Thresh} implies that
\begin{align}
\left\|\texttt{ST}_{\nu_t}(\widehat{\Sigma}_t)-\Sigma_t\right\|_\infty\leq 64\kappa_3^{1-q}s(q,d)\left(\frac{\tau\log d}{N_t}\right)^{\frac{1-q}{2}}
\end{align}
Combining this inequality with~\eqref{eq_ST} leads to
\begin{align}
\left\|\texttt{ST}_{\nu_t}(\widehat{\Sigma}_t)-\Sigma_t\right\|_\infty\leq \frac{\kappa_2}{2}
\end{align}
provided that
\begin{align}
N_t\geq \left(\frac{128 s(q,d)}{\kappa_2}\right)^{\frac{2}{1-q}}\kappa_3^2\tau\log d
\end{align}
This implies that $\|\texttt{ST}_{\nu_t}(\widehat{\Sigma}_t)w\|_\infty\geq \frac{\kappa_2}{2}\|w\|_\infty$, and hence, $\left\|[\texttt{ST}_{\nu_t}(\widehat{\Sigma}_t)]^{-1}\right\|_{\infty}\leq\frac{2}{\kappa_2}$.
Combining these bounds with~\eqref{eq_diff} yields
\begin{align}\label{eq_diff}
\left\|\Theta_t-[\texttt{ST}_{\nu_t}(\widehat{\Sigma}_t)]^{-1}\right\|_{\linf}\leq\frac{32\kappa_1\kappa_3}{\kappa_2}\sqrt{\frac{\tau\log d}{N_t}} = \lambda_t
\end{align}
with probability of at least $1-4d^{-\tau+2}$. Finally, we need to verify that the conditions $\lambda_t\leq \Theta_t^{\min}/2$ and $\lambda_t+\lambda_{t-1}\leq \Delta\Theta_t^{\min}/2$ hold. Based on the above definition of $\lambda_t$, it is easy to see that both of these conditions are satisfied if
\begin{align}
N_t&\geq \left(\frac{128\kappa_1\kappa_3}{\kappa_2}\right)^2\max\left\{{\left(\Theta_t^{\min}\right)}^{-2}, {\left(\Delta\Theta_t^{\min}\right)}^{-2}, {\left(\Delta\Theta_{t-1}^{\min}\right)}^{-2}\right\}\tau\log d\nonumber\\
\implies N_t &\gtrsim \tau\log d
\end{align}
Based on our assumption, we have $T+1 \leq Cd^{\zeta}$ for some universal constant $C>0$. Therefore, a simple union bound over $t=0,\dots, T$ implies that the statements of the corollary holds for every $t=0,\dots, T$ with the probability of at least
\begin{align}
1-4\sum_{t=0}^Td^{-\tau+2}\geq 1-4(T+1)d^{-\tau+2}\geq 1-4d^{\zeta-\tau+2}
\end{align}
Selecting $\tau>\zeta+2$ completes the proof.$\hfill\square$

\subsection{Proof of Theorem~\ref{cor_GMRF_ker}}\label{app_cor_GMRF_ker}
First, we delineate the imposed assumptions on the selected kernel function.
\begin{assumption}[~\cite{greenewald2017time}]
	The kernel $K(x)$ satisfies the following conditions:
	\begin{itemize}
		\item[-] $\int_{-1}^{1} K(x)dx=1$,
		\item[-] $\int_{-1}^{1} x^2K(x)dx\leq\infty$,
		\item[-] $K(x)$ is uniformly bounded on its support,
		\item[-] $\sup_{-1\leq x\leq 1}K''(x/h) = \mathcal{O}(h^{-4})$.
	\end{itemize}
\end{assumption}
The following key lemmas are borrowed from~\cite{greenewald2017time}.
\begin{lemma}[Lemma 5 of~\cite{greenewald2017time}]\label{l_h}
	For any fixed $t$, we have
	\begin{align}
	\|\mathbb{E}[\Sigma_t^w]-\Sigma(t/T)\|_{\linf}\lesssim C\left(h+\frac{1}{T^2h^5}\right)
	\end{align}
	for some constant $C>0$.
\end{lemma}
\begin{lemma}[Lemma 2 of~\cite{greenewald2017time}]\label{l_const}
	There exists a constant $c>0$ such that 
	\begin{align}
	\mathbb{P}\left(\left|[\Sigma_t^w]_{ij}-\mathbb{E}[\Sigma_t^w]_{ij}\right|\geq \epsilon\right)\leq 2\exp(-cTh\epsilon^2)
	\end{align}
	for every $\epsilon>0$ and any fixed $t$.
\end{lemma}
Combining the above lemmas gives rise to the following result.
\begin{lemma}\label{l_T}
	Assume that $h=T^{-1/3}$. Then, the following inequality holds for any $t$ and $\tau>2$
	\begin{align}
	\left\|\widehat \Sigma_t^w-\Sigma(t/T)\right\|_{\linf}\lesssim\frac{\sqrt{\tau\log d}}{T^{1/3}}
	\end{align}
	with probability of at least $1-d^{-(\tau-2)}$.
\end{lemma}
\begin{proof}
	Based on Lemma~\ref{l_const}, one can write
	\begin{align}
	\mathbb{P}\left(\left\|\widehat\Sigma_t^w-\mathbb{E}[\Sigma_t^w]\right\|_{\linf}\geq \epsilon\right)\leq 2\exp(2\log d-cTh\epsilon^2)
	\end{align}
	Upon choosing $\epsilon = \sqrt{\frac{\tau\log d}{cTh}}$ for some $\tau>2$, we have
	\begin{align}
	\left\|\widehat\Sigma_t^w-\mathbb{E}[\Sigma_t^w]\right\|_{\linf}\leq \sqrt{\frac{\tau\log d}{cTh}}
	\end{align}
	with probability of at least $1-d^{-(\tau-2)}$. 
	Combined with Lemma~\ref{l_h}, the following chain of inequalities hold with the same probability
	\begin{align}
	\left\|\widehat \Sigma_t^w-\Sigma(t/T)\right\|_{\linf}&\leq \left\|\widehat \Sigma_t^w-\mathbb{E}[\Sigma_t^w]\right\|_{\linf}+\left\|\mathbb{E}[\Sigma_t^w]-\Sigma(t/T)\right\|_{\linf}\nonumber\\
	&\leq \sqrt{\frac{\tau\log d}{cTh}}+C\left(h+\frac{1}{T^2h^5}\right)
	\end{align}
	Replacing $h = T^{-1/3}$ in the above inequality gives rise to 
	\begin{align}
	\left\|\widehat \Sigma_t^w-\Sigma(t/T)\right\|_{\linf}\lesssim \frac{\sqrt{\tau\log d}}{T^{1/3}}
	\end{align}
	which completes the proof.
\end{proof}

\begin{lemma}\label{l_inf}
	Assume that $h=T^{-1/3}$. Then, the following inequality holds for any $t$ and $\tau>2$
	\begin{align}
	\left\|\texttt{ST}_{\nu}(\widehat{\Sigma}_t^w)-\Sigma(t/T)\right\|_\infty\lesssim \nu^{1-q}s(q,d)+\nu^{-q}s(q,d)\frac{\sqrt{\tau\log d}}{T^{1/3}}
	\end{align}
	with probability of at least $1-d^{-\tau+2}$.
\end{lemma}
\begin{proof}
	The proof is implied by Lemma 1 of~\cite{yang2014elementary} and Lemma~\ref{l_T}.
\end{proof}

\noindent{\it Proof of Corollary~\ref{cor_GMRF_ker}.} We only provide a sketch of the proof, due to to its similarity to the proof of Corollary~\ref{cor_GMRF}.
One can write
\begin{align}\label{eq_upper}
\left\|\Theta(t/T)-[\texttt{ST}_{\nu_t}(\widehat \Sigma_t^w)]^{-1}\right\|_{\linf}\leq \left\|[\texttt{ST}_{\nu_t}(\widehat \Sigma_t^w)]^{-1}\right\|_\infty\|\Theta(t/T)\|_\infty\left\|\texttt{ST}_{\nu_t}(\widehat \Sigma_t^w)-\Sigma(t/T)\right\|_{\linf}
\end{align}
Due to Assumption~\ref{assum12}, we have $\|\Theta(t/T)\|_\infty\leq \kappa_1$. Furthermore, similar to~\eqref{eq_ST}, one can write
\begin{align}
\left\|\texttt{ST}_{\nu_t}(\widehat \Sigma_t^w)-\Sigma(t/T)\right\|_{\linf}&\leq \left\|\texttt{ST}_{\nu_t}(\widehat \Sigma_t^w)-\widehat \Sigma_t^w\right\|_{\linf}+\left\|\widehat \Sigma_t^w-\Sigma(t/T)\right\|_{\linf}\nonumber\\
&\lesssim \frac{\sqrt{\tau\log d}}{T^{1/3}}
\end{align}
with probability of at least $1-d^{-\tau+2}$, where the second inequality follows from Lemma~\ref{l_T} and the choice of $\nu_t\asymp\frac{\sqrt{\tau\log d}}{T^{1/3}}$. Finally, Lemma~\ref{l_inf} combined with an argument similar to the proof of Corollary~\ref{cor_GMRF} leads to
\begin{align}
\left\|[\texttt{ST}_{\nu_t}(\widehat \Sigma_t^w)]^{-1}\right\|_\infty\leq\frac{2}{\kappa_2}
\end{align}
provided that
\begin{align}
T\gtrsim s(q,d)^{\frac{3}{1-q}}(\tau\log d)^{3/2}
\end{align}
Combining these inequalities leads to the desired upper bound on~\eqref{eq_upper}. The rest of the proof is similar to that of Corollary~\ref{cor_GMRF} and omitted for brevity.$\hfill\square$

\subsection{Proof of Proposition~\ref{prop_beta1}}\label{app_prop_beta1}

Let $\delta_1<\delta_2<\ldots<\delta_m=T$ be the elements of the set $\Gamma$ from Algorithm \ref{alg_greedy}, and define $\delta_0=-1$. By construction, $\DI_{\delta_{i-1}+1\to\delta_i+1}=\emptyset$ for all $i=1,\dots,m-1$. It follows that for any $\theta$ satisfying bound constraints \eqref{opt_ij_bounds} and $i=1,\dots,m-1$, we have that
$$\sum_{t=\delta_{i-1}+1}^{\delta_{i}}\mathbbm{1}\{\theta_{t+1}-\theta_{t}\not=0\}\geq 1.$$

Given any $j=1,\dots,T$, let $h$ be the maximum index such that $\delta_h<j$. Therefore, we find that for any feasible $\theta$,
\begin{align*}
f_{0\to j}(\theta)=\sum_{t=0}^{j-1}\mathbbm{1}\{\theta_{t+1}-\theta_{t}\not=0\}\geq \sum_{t=0}^{\delta_{h}}\mathbbm{1}\{\theta_{t+1}-\theta_{t}\not=0\}=\sum_{i=1}^{h}\sum_{t=\delta_{i-1}+1}^{\delta_{i}}\mathbbm{1}\{\theta_{t+1}-\theta_{t}\not=0\}\geq h.
\end{align*}
Since $f_{0\to j}(\theta^{\texttt{Greedy}})=h$ meets this lower bound, it follows that $\{\theta^{\texttt{Greedy}}_t\}_{t=0}^j$ is indeed an optimal solution to $\texttt{OPT}_{0\to j}(1)$. Setting $j=T$ and $h=m-1$, we find that $\theta^{\texttt{Greedy}}$ is optimal for $\texttt{OPT}_{0\to T}(1)$. $\hfill\square$

% \section{Proof of Lemma~\ref{prop:allZero}}
% Let $\theta$ be any feasible solution to \eqref{generalopt} that does not satisfy the conditions of Proposition~\ref{prop:allZero}, i.e., there exists $\tau=i,\ldots,j-1$ such that either $\theta_{\tau}=0$ and $\theta_{\tau+1}\neq 0$, or $\theta_{\tau}\neq 0$ and $\theta_{\tau+1}= 0$. We now show how to construct a solution $\hat \theta$ with improved objective value, i.e., $f_{0\to T}(\hat \theta)< f_{0\to T}(\theta)$. 
		
% 		Consider the case $\theta_{\tau}=0$ and $\theta_{\tau+1}\neq 0$. Define $\hat \theta_{\tau+1}=0$ and $\hat \theta_t=\theta_t$ for all other coordinates $t\neq \tau+1$. Clearly, $\hat \theta$ satisfies all bound constraints \eqref{generalopt}. Moreover, $$f_{0\to T}(\hat \theta)=f_{0\to T}(\theta) -\underbrace{(1-\alpha)}_\text{$\hat \theta_{\tau+1}=0$}-\underbrace{\alpha}_\text{$\hat \theta_{\tau}=\hat \theta_{\tau+1}$}+\underbrace{\alpha \mathbbm{1}\{\hat\theta_{\tau+1}\neq \hat \theta_{\tau+2}\}}_\text{this term is $0$ if $\tau+1=T$}\leq f_{0\to T}(\theta)- (1-\alpha)<f_{0\to T}(\theta).$$
% 		The case $\theta_{\tau}\neq 0$ and $\theta_{\tau+1}= 0$ is handled analogously. $\hfill\square$

\subsection{Proof of Theorem~\ref{thm_spath}}
Before proving this theorem, we need the following intermediate lemma:
\begin{lemma}\label{prop:allZero}
		Given any optimal solution $\widehat\theta$ to \eqref{opt_ij}, exactly one of the following holds for any given zero-feasible sequence $\mathcal{Z}_{i\to j}$:
		\begin{enumerate}
			\item $\widehat\theta_i=\widehat\theta_{i+1}=\ldots=\widehat\theta_{j}=0$
			\item $\widehat\theta_{\tau}\neq 0$ for all $\tau=i,\ldots,j$. 
		\end{enumerate}
	\end{lemma}
	\begin{proof}
	Let $\theta$ be any feasible solution to \eqref{generalopt} that does not satisfy the conditions of Proposition~\ref{prop:allZero}, i.e., there exists $\tau=i,\ldots,j-1$ such that either $\theta_{\tau}=0$ and $\theta_{\tau+1}\neq 0$, or $\theta_{\tau}\neq 0$ and $\theta_{\tau+1}= 0$. We now show how to construct a solution $\hat \theta$ with improved objective value, i.e., $f_{0\to T}(\hat \theta)< f_{0\to T}(\theta)$. 
		
		Consider the case $\theta_{\tau}=0$ and $\theta_{\tau+1}\neq 0$. Define $\hat \theta_{\tau+1}=0$ and $\hat \theta_t=\theta_t$ for all other coordinates $t\neq \tau+1$. Clearly, $\hat \theta$ satisfies all bound constraints \eqref{generalopt}. Moreover, $$f_{0\to T}(\hat \theta)=f_{0\to T}(\theta) -\underbrace{(1-\alpha)}_\text{$\hat \theta_{\tau+1}=0$}-\underbrace{\alpha}_\text{$\hat \theta_{\tau}=\hat \theta_{\tau+1}$}+\underbrace{\alpha \mathbbm{1}\{\hat\theta_{\tau+1}\neq \hat \theta_{\tau+2}\}}_\text{this term is $0$ if $\tau+1=T$}\leq f_{0\to T}(\theta)- (1-\alpha)<f_{0\to T}(\theta).$$
		The case $\theta_{\tau}\neq 0$ and $\theta_{\tau+1}= 0$ is handled analogously.
	\end{proof}
{Since Lemma~\ref{prop:allZero} holds for any zero-feasible sequence, it holds in particular for all maximal zero-feasible sequences. }Based on this lemma, we are ready to present the proof of Theorem~\ref{thm_spath}.

{\it Proof of Theorem~\ref{thm_spath}}. Let $\widehat\theta$ be an optimal solution to \eqref{opt_ij}. Due to the optimality of $\widehat\theta$, the conditions in Lemma~\ref{prop:allZero} are satisfied for all maximal nonzero intervals. We first show that there exists a path in $\mathcal{G}$ with cost $f(\widehat\theta)$, and then we show that this path is indeed a shortest path.
		
		{Let $V_0=\{v_1,v_2,\dots,v_m\}\subseteq \{1,\dots,Z\}$ be the set of indexes of the maximal zero-feasible sequences where $\widehat\theta$ vanishes, i.e., $\widehat\theta_{\mathcal{Z}_{i_{s}\to j_{s}}} = 0$ for every $s\in V_0$. It is easy to verify that $f^*$ is the optimal cost for the following constrained optimization:}
		{\begin{subequations}\label{eq:decomposable1}
				\begin{align}
				f_{0\to T}(\widehat\theta)=\underset{\{\theta_{t}\}_{t=0}^T}{\min}\ \ & (1-\alpha)\left(T+1-\left|\bigcup_{h=1}^m \mathcal{Z}_{i_{v_h}\to j_{v_h}}\right|\right)+\alpha\sum_{t=1}^T\mathbbm{1}\{\theta_{t}-\theta_{t-1}\not=0\}\label{eq:decomposable1_obj}\\
				\mathrm{subject\ to}\ \ \ & l_{t}\leq \theta_{t} \leq u_{t}&&\hspace{-8cm}t = 0,\dots, T\\
				&\theta_t=0 &&\hspace{-8cm}t\in \bigcup_{h=1}^m \mathcal{Z}_{i_{v_h}\to j_{v_h}}.\label{eq:decomposable1_bounds}
				\end{align}
			\end{subequations}
			The constant term in \eqref{eq:decomposable1_obj} reduces to
			\begin{align}
			(1-\alpha)\left(T+1-\left|\bigcup_{h=1}^m \mathcal{Z}_{i_{v_h}\to j_{v_h}}\right|\right)&=(1-\alpha)\left(T+1-\sum_{h=1}^m(j_{v_h}-i_{v_h}+1)\right)\\
			&=(1-\alpha)\left(i_{v_1}+\sum_{h=2}^m\left(i_{v_h}-j_{v_{h-1}}-1\right)+\left(T-j_{v_m}\right)\right).\label{eq_constant}
			\end{align}
			Let the feasible region of~\eqref{eq:decomposable1} be denoted as $\mathcal{X}$. The second term in \eqref{eq:decomposable1_obj}, under constraints \eqref{eq:decomposable1_bounds}, decomposes as 
			\begin{align}
			&\min_{\{\theta_t\}_{t=0}^T\in\mathcal{X}}\left\{\alpha\sum_{t=1}^T\mathbbm{1}\{\theta_{t}\neq\theta_{t-1}\}\right\}\notag\\
			=&\alpha\min_{\{\theta_t\}_{t=0}^T\in\mathcal{X}}\left\{\sum_{t=1}^{i_{v_{1}}}\mathbbm{1}\{\theta_{t}\neq\theta_{t-1}\}\right\}
			+ \alpha\sum_{h=1}^{m-1} \min_{\{\theta_t\}_{t=0}^T\in\mathcal{X}}\left\{\sum_{t=j_{v_h}+1}^{i_{v_{h+1}}}\mathbbm{1}\{\theta_{t}\neq\theta_{t-1}\}\right\}\notag\\
			&+\alpha\min_{\{\theta_t\}_{t=0}^T\in\mathcal{X}}\left\{\sum_{t=j_{v_m}+1}^{T}\mathbbm{1}\{\theta_{t}\neq\theta_{t-1}\}\right\}\label{eq:diff}
			\end{align}
			Note that each intermediate term in \eqref{eq:diff} simplifies as follows:
			\begin{align}
			\min_{\{\theta_t\}_{t=0}^T\in\mathcal{X}}\left\{\sum_{t=j_{v_h}+1}^{i_{v_{h+1}}}\mathbbm{1}\{\theta_{t}\neq\theta_{t-1}\}\right\}
			&=\underbrace{\min_{\{\theta_t\}_{t=0}^T\in\mathcal{X}}\left\{\sum_{t=j_{v_h}+2}^{i_{v_{h+1}-1}}\mathbbm{1}\{\theta_{t}\neq\theta_{t-1}\}\right\}}_{=f_{j_{v_h}+1\to i_{v_{h+1}}-1}^{\texttt{Greedy}}}\notag\\
			&+\underbrace{\mathbbm{1}\{\theta_{j_{v_h}+1}\neq\theta_{j_{v_h}}\}}_{=1}+\underbrace{\mathbbm{1}\{\theta_{i_{v_{h+1}}}\neq\theta_{i_{v_{h+1}}-1}\}}_{=1}.\label{eq_h}
			\end{align}
			Similarly, we find that the first and last term in \eqref{eq:diff} reduces to
			{\small
				\begin{align}
				&\min_{\{\theta_t\}_{t=0}^T\in\mathcal{X}}\left\{\sum_{t=1}^{i_{v_{1}}}\mathbbm{1}\{\theta_{t}\neq\theta_{t-1}\}\right\}=\min_{\{\theta_t\}_{t=0}^T\in\mathcal{X}}\left\{\sum_{t=1}^{i_{v_{1}}-1}\mathbbm{1}\{\theta_{t}\neq\theta_{t-1}\}\right\}+\mathbbm{1}\{i_{v_1}\geq 1\}=f_{0\to i_{v_{1}}-1}^{\texttt{Greedy}}+\mathbbm{1}\{i_{v_1}\neq 0\}\label{eq_1}\\
				&\min_{\{\theta_t\}_{t=0}^T\in\mathcal{X}}\left\{\sum_{t=j_{v_m}+1}^{T}\mathbbm{1}\{\theta_{t}\neq\theta_{t-1}\}\right\}=\min_{\{\theta_t\}_{t=0}^T\in\mathcal{X}}\left\{\sum_{t=j_{v_m}+2}^{T}\mathbbm{1}\{\theta_{t}\neq\theta_{t-1}\}\right\}+\mathbbm{1}\{j_{v_m}+1\leq T\}\notag\\
				&\hspace{5.3cm}=f_{j_{v_m}+1\to T}^{\texttt{Greedy}}+\mathbbm{1}\{j_{v_m}< T\}.\label{eq_T}
				\end{align}
			}
			\begin{sloppypar}
				\noindent Combining~\eqref{eq_T},~\eqref{eq_1} with~\eqref{eq:diff} and~\eqref{eq_constant}, we find that $f_{0\to T}(\widehat\theta)$ is precisely the length of the path $(0,v_1,\dots,v_m,Z+1)$ in the constructed graph $\mathcal{G}$ with weights defined as~\eqref{eq_A}.
			\end{sloppypar}
			
			Now suppose that there exists a path $(0,\bar v_1,\bar v_2,\ldots,\bar v_p,Z+1)$ with length $\bar d<f_{0\to T}(\widehat\theta)$. Consider a solution $\bar{\theta}$ such that: (i) $\bar \theta$ is zero at zero-feasible sequences given by $\bar v_1,\bar v_2,\ldots,\bar v_p$, and (ii) $\bar \theta$ is obtained from \texttt{Greedy}$(l,u,0,i_{v_1}-1)$, \texttt{Greedy}$(l,u,j_{v_p}+1,T)$ and \texttt{Greedy}$(l,u,j_{v_h}+1,i_{v_{h+1}}-1)$, otherwise.
			It is easy to verify that $\bar{\theta}$ is feasible and satisfies $f_{0\to T}(\bar \theta)\leq \bar d$ (the inequality could be strict if any solution reported by a call to the \texttt{Greedy} routine has zero values), which contradicts the optimality of $\widehat\theta$. Thus, we conclude that $f_{0\to T}(\widehat\theta)$ is indeed the length of the shortest $(0,Z+1)$-path in $\mathcal{G}$.$\hfill\square$  }
% }
%%%%%%%%%%%%%%%%%%%%%%%%%%%%%%%%%%%%%%%%%%%%%%%%%%%%%%%%%%%%%%%%%%%%%%%%%%%%%%%
%%%%%%%%%%%%%%%%%%%%%%%%%%%%%%%%%%%%%%%%%%%%%%%%%%%%%%%%%%%%%%%%%%%%%%%%%%%%%%%
\subsection{Proof of Theorem~\ref{thm_runtime_sp}}
{Algorithm~\ref{alg_path} involves three main components: construct graph $\mathcal{G}$ (line \ref{alg_path_construct}), solve a shortest problem on the constructed graph (line \ref{alg_path_solve}), and recover the optimal solution from the obtained shortest path.} Since $\mathcal{G}$ is acyclic, the shortest path problem can be solved in time linear in the number of arcs, which is $\mathcal{O}(Z^2)$, via a simple labeling algorithm; see, e.g., Chapter 4.4. in \cite{ahuja1988network}. Constructing graph $\mathcal{G}$ requires computing the costs of all arcs. A na\"ive implementation, where Algorithm~\ref{alg_greedy} is called for every arc, would require $\mathcal(O)(Z^2T)$ time and memory. However, from the second statement in Proposition~\ref{prop_beta1}, we note that a single call to \texttt{Greedy}$(l,u,i,T)$ allows us to compute $f_{i\to j}^{\texttt{Greedy}}$ for all $i\leq j\leq T$. Therefore, Algorithm~\ref{alg_greedy} needs to be invoked only $\mathcal{O}(Z)$ times, and each call require $\mathcal{O}(T)$ leading to a total complexity of $\mathcal{O}(ZT)$. {Moreover, given the shortest path, the optimal solution can be constructed by concatenating the solutions obtained from the calls of \texttt{Greedy}.} Finally, since $Z\leq T+2$, we find that the overall complexity is dominated by that of constructing the graph. This completes the proof. $\hfill\square$

\end{document}